\documentclass{article}

\usepackage{microtype}
\usepackage{graphicx}
\usepackage{subcaption}
\usepackage{booktabs} 

\usepackage{hyperref}
\usepackage{caption}

\usepackage[preprint]{icml2026}

\usepackage{amsmath}
\usepackage{amssymb}
\usepackage{mathtools}
\usepackage{amsthm}

\usepackage[dvipsnames]{xcolor}
\usepackage{algorithm}

\usepackage{algpseudocode}
\usepackage{float}

\newcommand{\indep}{\perp \!\!\! \perp}
\newtheorem{theorem}{Theorem}[section]
\newtheorem{lemma}[theorem]{Lemma}
\newtheorem{proposition}[theorem]{Proposition}
\newtheorem{definition}[theorem]{Definition}
\newtheorem{corollary}[theorem]{Corollary}

\newtheorem*{remark}{Remark}

\usepackage[table]{xcolor}   
\definecolor{entityOne}{HTML}{00A08A}   
\definecolor{entityTwo}{HTML}{F2AD00}   
\definecolor{sZero}{HTML}{FF0000}       
\definecolor{sOne}{HTML}{0000FF}        
\definecolor{sbar}{HTML}{9986A5}        

\definecolor{colAA}{HTML}{1B9E77}
\definecolor{colC}{HTML}{D95F02}

\newcommand{\cS}{\mathcal{S}}
\newcommand{\ind}[1]{\mathbf{1}_{\{#1\}}}
\newcommand{\R}{\mathbb{R}}

\newcommand{\wideq}{\widehat{q}}
\newcommand{\wideQ}{\widehat{Q}}
\newcommand{\wideF}{\widehat{F}}

\newcommand{\overF}{\overline{F}}

\algrenewcommand\alglinenumber[1]{}

\newcommand{\PP}{\mathbb{P}}

\newcommand{\monparagraph}[1]{\noindent\textbf{#1}\ }

\usepackage[capitalize,noabbrev]{cleveref}

\usepackage[textsize=tiny]{todonotes}

\icmltitlerunning{Federated Measurement of Demographic Disparities
from Quantile Sketches}

\begin{document}

\twocolumn[
  \icmltitle{Federated Measurement of Demographic Disparities from Quantile Sketches}



  \icmlsetsymbol{equal}{*}

  \begin{icmlauthorlist}
    \icmlauthor{Arthur Charpentier}{xxx,zzz}
    \icmlauthor{Agathe Fernandes~Machado}{xxx}
    \icmlauthor{Olivier Côté}{yyy}
    \icmlauthor{François Hu}{mmm,iii}
  \end{icmlauthorlist}

  \icmlaffiliation{xxx}{UQAM, Montréal, Canada}
  \icmlaffiliation{zzz}{Kyoto University, Japan}
  \icmlaffiliation{yyy}{Université Laval, Québec, Canada}
  \icmlaffiliation{mmm}{Milliman, Paris, France}
  \icmlaffiliation{iii}{Université Claude Bernard, Lyon, France}

  \icmlcorrespondingauthor{Arthur Charpentier}{charpentier.arthur@uqam.ca}

  \icmlkeywords{Demographic Parity; Optimal Transport}

  \vskip 0.3in
]



\printAffiliationsAndNotice{}  

\begin{abstract}
Many fairness goals are defined at a population level that misaligns with siloed data collection, which remains unsharable due to privacy regulations. Horizontal federated learning (FL) enables collaborative modeling across clients with aligned features without sharing raw data. We study federated auditing of demographic parity through score distributions, measuring disparity as a Wasserstein--Fr\'echet variance between sensitive-group score laws, and expressing the population metric in federated form that makes explicit how silo-specific selection drives local–global mismatch. For the squared Wasserstein distance, we prove an ANOVA-style decomposition that separates (i) selection-induced mixture effects from (ii) cross-silo heterogeneity, yielding tight bounds linking local and global metrics. We then propose a one-shot, communication-efficient protocol in which each silo shares only group counts and a quantile summary of its local score distributions, enabling the server to estimate global disparity and its decomposition, with $O(1/k)$~ discretization bias ($k$ quantiles) and finite-sample guarantees. Experiments on synthetic data and COMPAS show that a few dozen quantiles suffice to recover global disparity and diagnose its sources.
\end{abstract}
\section{Introduction}\label{sec:introduction}

Machine-learning models are increasingly deployed \emph{across} multiple institutions (hospitals, insurers, courts, banks), often in federated setups designed to avoid centralizing raw data \cite{mcmahan2017fedavg, kairouz2021advances}. At the same time, many fairness goals are defined at a population level that misaligns with siloed data collection: regulators want guarantees about the system as a whole, not just each participating entity in isolation \citep{barocas2017fairness}. Data are fragmented, protected attributes are sensitive, and audits are often siloed. Falling back on local fairness alone can mislead: even if each silo appears fair, selection biases that shift group composition across silos may induce substantial unfairness after aggregation. This creates a scope mismatch: fairness requires information across silos, while privacy forbids sharing.

This paper studies federated auditing of demographic parity under privacy and communication constraints. Our focus is \emph{measurement} (auditing), not mitigation: we ask what population-level demographic disparity can be inferred from limited summaries released by each silo, and how to diagnose the mechanisms that make local and global conclusions diverge. This perspective complements the broader fairness literature, including well-known trade-offs between fairness criteria (e.g., calibration vs.\ error-rate parity) in the context of risk scores \cite{Chouldechova_2017,kleinberg2016tradeoffs}.

\subsection{Motivation}

Two mechanisms drive the local-global fairness mismatch.

\monparagraph{(i) Selection-induced mixture effects (Simpson-style).}
A model may treat sensitive groups similarly \emph{within} each silo, yet \emph{group composition}, i.e., the distribution of features, may vary across silos. Aggregating scores across heterogeneous populations can create (or hide)
global disparities, a fairness analogue of Simpson's paradox
\cite{julious1994confounding,prost2022simpson}. This is common in high-stakes domains, where institutional pipelines (referral patterns, access to care, jurisdictional practices) govern group allocation to silos, often along lines correlated with protected attributes.

\monparagraph{(ii) Cross-silo score heterogeneity.}
A second source of discrepancy is distribution shift in scores across silos.  While differences may go unnoticed in silo-specific evaluations, they emerge upon pooling, akin to hidden stratification \cite{oakden2020hidden,tosh2022simple}. In medicine, risk tools may perform well within systems yet yield racial disparities overall \cite{obermeyer2019dissecting}. Similar issues arise in imaging: under-served subgroups may be systematically underdiagnosed, eluding single-silo audits \cite{seyyed2021underdiagnosis,larrazabal2020gender,chen2021ethical}. In criminal justice, the \texttt{COMPAS} scores illustrate demographic disparities in risk assessment \cite{Larson2016,Chouldechova_2017}.

\monparagraph{A minimal example.}
Let $Z \in \{0, 1\}$ be a binary decision with $Z = \mathbf{1}\{X > 0\}$, and let $S \in \{0,1\}$ with
$$
X \mid S=0 \sim \mathcal{N}(-\mu,1), \quad X \mid S=1 \sim \mathcal{N}(\mu,1), \quad \mu > 0.
$$
Assign $X<0$ to silo $A=1$, and $X>0$ to silo $A=2$, inducing across-silo (i) $S$ heterogeneity and (ii) $Z$ heterogeneity. Each silo $j \in \{1,2\}$ enforces local demographic parity:
$$
\PP(Z=1 \mid A=j, S=0) = \PP(Z=1 \mid A=j, S=1).
$$
Yet $Z$ is globally associated with $S$ through $A$ and $X$, and group proportions differ across silos. Globally,
$$
\PP(Z=1 \mid S=0) < \frac{1}{2} < \PP(Z=1 \mid S=1).
$$
Demographic disparity emerges after aggregation of silos.

Most fairness constraints (e.g. demographic parity) concern \emph{distributions}, not binary decisions. Thresholding can mask distributional shifts; binary audits may be fragile to how decision rules are actually implemented downstream. We therefore audit demographic parity directly at the level of \emph{score distributions}. We quantify demographic disparity as a \emph{Wasserstein--Fr\'echet variance} of score distributions across sensitive groups. In the centralized case, this yields a population-level unfairness functional $U_p$, defined via Wasserstein distances to a barycenter. In the federated setting, raw scores are inaccessible; our goal is to estimate $U_p$ (and decompose its sources) from minimal, privacy-preserving summaries. Our communication primitive is a compact quantile sketch, akin to classical streaming summaries \cite{greenwald2001gk,karnin2016kll}, such as the widely used t-digest \cite{dunning2019tdigest}. These can be combined with secure aggregation for practical, privacy-preserving deployment \cite{bonawitz2017secureagg}.

\subsection{Contributions}
\begin{itemize}
  \item \textbf{A population-level unfairness functional for distributional demographic parity.}
  We formalize demographic disparity via score distributions and define a centralized functional
  $U_p$ as a Wasserstein--Fr\'echet variance across sensitive-group score conditionals.
  \item \textbf{A federated counterpart that captures selection and heterogeneity.}
We define $G_p$ (federated notation of the population functional) and show $G_p=U_p$; for $p=2$, we derive a decomposition separating selection-induced mixture effects from cross-silo heterogeneity.

\item \textbf{A one-shot, privacy-motivated audit protocol (with quantile sketches).}
We design a communication-efficient protocol where each silo transmits only group counts and a $k$-quantile sketch of local scores. From these summaries, the server estimates global unfairness and attributes it to distinct sources.

  \item \textbf{Theory: discretization rates and a decomposition separating (i) vs (ii).}
  We prove consistency and an explicit discretization bias of order $O(1/k)$ under quantile discretization. For $p = 2$, we derive an ANOVA-style decomposition that separates (i) selection-induced mixture effects from (ii) cross-silo heterogeneity, with tight bounds linking local and global metrics. 

  \item \textbf{Empirical validation.}
  Experiments on synthetic data and \texttt{COMPAS} show that local audits can mislead, and that
  few quantiles suffice to estimate global disparity and diagnose its sources.

\end{itemize}

\monparagraph{Organization.}
\Cref{sec:setting} introduces the problem and centralized functional \( U_p \).  
\Cref{sec:quantile} analyzes quantile approximations and discretization.  
\Cref{sec:federated} defines the federated functional \( G_p \) and the audit protocol.  
\Cref{sec:anova} gives the ANOVA decomposition and local--global bounds.  
\Cref{sec:experiments} reports empirical results.

\section{Setting and centralized fairness}\label{sec:setting}

\subsection{Data model and notation}\label{sec:setting:notation}

We observe i.i.d.\ tuples $(X,S,Y,A)$, where $X\in\mathbb{R}^q$ are non-sensitive features,
$S\in\mathcal{S}=\{0,1\}$ is a binary sensitive attribute, $Y$ is the target, and
$A\in[d]:=\{1,\dots,d\}$ indicates the collecting entity (silo).
Entity $j$ holds the local dataset
$\mathcal{D}_j := \{(x_i,s_i,y_i): a_i=j\}$, of size $n_j$, and $n_{j,s}$ denotes the number
of observations in $\mathcal{D}_j$ with $S=s$.
Let $n_s:=\sum_{j=1}^d n_{j,s}$ and $n:=\sum_{s\in\mathcal{S}} n_s$.

Each entity deploys a scoring function $m_j:\mathbb{R}^q\times\mathcal{S}\to\mathbb{R}$, and
we write $m(x,s,a=j):=m_j(x,s)$.
Let $Z := m(X,S,A)$ be the resulting score.
For each $(j,s)$, define the local conditional score law
\[
\nu_{j,s} := \mathcal{L}(Z \mid A=j,\, S=s),
~
\nu_s := \mathcal{L}(Z \mid S=s).
\]
We also define population weights
\[
\alpha_s := \mathbb{P}(S=s),~
\beta_j := \mathbb{P}(A=j),~
\pi_{j\mid s} := \mathbb{P}(A=j\mid S=s),
\]
with empirical counterparts $\hat\alpha_s=n_s/n$, $\hat\beta_j=n_j/n$, and $\hat\pi_{j\mid s}=n_{j,s}/n_s$.
By total probability,
\begin{equation}\label{eq:nu-mixture-across-silos}
\nu_s \;=\; \sum_{j=1}^d \pi_{j\mid s}\,\nu_{j,s}.
\end{equation}

For any probability measure $\nu$ on $\mathbb{R}$, we denote its cdf by $F_\nu$ and its
(left-continuous) quantile function by
\[
Q_\nu(u) := \inf\{z\in\mathbb{R}:\ F_\nu(z)\ge u\},\qquad u\in(0,1).
\]
When unambiguous, we write $F_{j,s},Q_{j,s}$ for $\nu_{j,s}$ and $F_s,Q_s$ for $\nu_s$.
Throughout, we consider $p\in\{1,2\}$ and assume the quantities below are finite; the
regularity assumptions needed for approximation rates are stated later.

\subsection{Demographic parity and homogeneity as distributional equalities}\label{sec:setting:dp}

Demographic parity (DP) holds at the population level when the score distribution does not
depend on the sensitive group:
\[
\nu_0=\nu_1
~~\text{equivalently}~~
Z \indep S.
\]
We audit DP at the level of score distributions. A key theme of the paper is that
two natural ``centers'' for a collection of distributions coexist:
a \emph{transport (Wasserstein) barycenter}, and a \emph{mixture (Cram\'er) barycenter}.
Both lead to clean, parallel metrics.

\monparagraph{Two distances on one-dimensional laws.}
For $p\ge 1$, the $p$-Wasserstein distance between $\nu,\nu'\in\mathcal{P}_p(\mathbb{R})$ admits the quantile form
\begin{equation}\label{eq:Wp-quantile}
W_p(\nu,\nu')^p \;=\; \int_0^1 \bigl|Q_\nu(u)-Q_{\nu'}(u)\bigr|^p\,\mathrm{d}u.
\end{equation}
We also consider the Cram\'er $p$-distance (a.k.a.\ Cram\'er--von Mises when $p=2$),
defined through cdfs:
\begin{equation}\label{eq:Cp-cdf}
C_p(\nu,\nu')^p \;:=\; \int_{\mathbb{R}} \bigl|F_\nu(x)-F_{\nu'}(x)\bigr|^p\,\mathrm{d}x.
\end{equation}
In one dimension, $C_1=W_1$.

\monparagraph{Two barycenters: transport vs mixture.}
Given $\{\nu_s:s\in\mathcal{S}\}$ and weights $\{\alpha_s\}$, a Wasserstein barycenter is any minimizer
\begin{equation}\label{eq:wass-barycenter}
\nu^\ast \in \arg\min_{\nu}\ \left\lbrace\sum_{s\in\mathcal{S}} \alpha_s\,W_p(\nu_s,\nu)^p\right\rbrace,
\end{equation}
see \citep{agueh2011barycenters,kim2020nonpositive,sturm2003probability}.
When $p=2$ in one dimension, the barycenter is unique and satisfies
\begin{equation}\label{eq:wass-barycenter-quantile}
Q_{\nu^\ast}(u)=\sum_{s\in\mathcal{S}} \alpha_s\,Q_s(u),\qquad u\in(0,1).
\end{equation}
In contrast, the Cram\'er barycenter is obtained by averaging cdfs, i.e., it is simply the mixture:
\begin{equation}\label{eq:cramer-barycenter}
\nu^\times := \sum_{s\in\mathcal{S}} \alpha_s\,\nu_s
~\Longleftrightarrow~
F_{\nu^\times}(x)=\sum_{s\in\mathcal{S}} \alpha_s\,F_s(x).
\end{equation}
We use the superscript $\ast$ for (Wasserstein) barycenters and $\times$ for mixtures.

\begin{definition}[Central unfairness functional]\label{def:Up}
Define the population-level demographic disparity functional (distance to demographic parity)
\begin{equation}\label{eq:Up}
U_p \;:=\; \sum_{s\in\mathcal{S}} \alpha_s\,W_p(\nu_s,\nu^\ast)^p,
\end{equation}
where $\nu^\ast$ is a Wasserstein barycenter solving \eqref{eq:wass-barycenter}.
Demographic parity holds iff $U_p=0$.
\end{definition}

\begin{definition}[Heterogeneity functional]\label{def:Hp}
Define the pooled heterogeneity (distance to homogeneity) functional
\begin{equation}\label{eq:Hp}
H_p \;:=\; \sum_{s\in\mathcal{S}} \alpha_s\,C_p(\nu_s,\nu^\times)^p,
\end{equation}
where $\nu^\times$ is the mixture \eqref{eq:cramer-barycenter}.
We say the population is homogeneous across groups when $H_p=0$.
\end{definition}

\monparagraph{Interpretation and parallelism.}
The functional $U_p$ measures how far each group distribution is from a \emph{transport consensus}
distribution $\nu^\ast$ (a Fr\'echet variance in Wasserstein geometry).
The functional $H_p$ measures how far each group distribution is from the \emph{pooled mixture}
$\nu^\times$ (a Fr\'echet variance in Cram\'er geometry).
Both vanish exactly when all group score laws coincide.

\begin{remark}[Two-group simplifications]\label{rem:two-groups}
When $\mathcal{S}=\{0,1\}$, both metrics reduce to a single distance between $\nu_0$ and $\nu_1$,
up to the factor $\alpha_0\alpha_1$:
for $p=2$, $U_2=\alpha_0\alpha_1\,W_2(\nu_0,\nu_1)^2$ and
$H_2=\alpha_0\alpha_1\,C_2(\nu_0,\nu_1)^2$.
\end{remark}

\section{Quantile-grid approximation (centralized)}\label{sec:quantile}

This section explains how the centralized functionals $U_p$ (Wasserstein geometry) and $H_p$
(Cram\'er geometry) can be approximated from a \emph{small quantile compressed representation} per group. 
This is the core computational primitive used later in the federated protocol.

\subsection{Quantile sketches and interpolation}\label{sec:quantile:sketch}

The \emph{sketch level} $k \ge 1$ denotes the number of support points in the empirical approximation; it controls the resolution of the stepwise distribution. We define a uniform grid of levels
\begin{equation}\label{eq:u-grid}
u_\ell := \frac{\ell-\frac12}{k},\text{ for } \ell\in[k].
\end{equation}
For each group $s\in\mathcal{S}$, the \emph{$k$-quantile sketch} is
\[
q_{s,\ell} := Q_s(u_\ell),\text{ for }  \ell\in[k],
\]
and we use empirical quantiles $\wideQ_{s,\ell}:=\widehat{Q}_s(u_\ell)$ computed
from the pooled sample $\{Z_i:\ S_i=s\}$.
We focus on a stepwise server-side reconstruction.
Appendix~\ref{app:reco} discusses stepwise vs.\ piecewise-linear reconstructions and shows that second-order accuracy can be obtained via exact integration of the linear reconstruction, without additional cost.

\paragraph{Quantile sketches.}
Our framework assumes that each client can compute the required empirical quantiles on a fixed grid
(or, more generally, return approximate quantiles at prescribed levels).
This is distinct from \emph{streaming} quantile summaries (or “quantile sketches”)
such as the GK sketch, KLL, or t-digest, which are designed to answer arbitrary quantile queries
in a single pass and sublinear memory, and can be used to compute these local quantiles efficiently when needed;
see \citep{greenwald2001gk,karnin2016kll,dunning2019tdigest}.

\monparagraph{Stepwise distribution induced by the sketch.}
The sketch naturally defines an atomic approximation of $\nu_s$:
\begin{equation}\label{eq:nu-bar-k}
\overline{\nu}_{s,k} \;:=\; \frac1k\sum_{\ell=1}^k \delta_{q_{s,\ell}},
~
\widehat{\nu}_{s,k} \;:=\; \frac1k\sum_{\ell=1}^k \delta_{\wideQ_{s,\ell}},
\end{equation}
where $\delta_{x}$ denotes the Dirac measure. Its cdf is
\begin{equation}\label{eq:F-bar-k}
\overF_{s,k}(x) \;=\; \frac1k\sum_{\ell=1}^k \mathbf{1}_{\{q_{s,\ell}\le x\}},
~
\wideF_{s,k}(x) \;=\; \frac1k\sum_{\ell=1}^k \mathbf{1}_{\{\wideQ_{s,\ell}\le x\}}.
\end{equation}
Equivalently, its quantile function is stepwise constant on intervals of length $1/k$.
(One may also use a linear interpolation of the quantile function for smoother curves; the
rates below are unchanged under the regularity assumptions we adopt.)

\monparagraph{Mixture vs.\ barycenter at the sketch level.}
The (Cram\'er) mixture barycenter is obtained by averaging measures:
\begin{equation}\label{eq:nu-times-k}
\overline{\nu}^{\times}_k := \sum_{s\in\mathcal{S}} \alpha_s\,\overline{\nu}_{s,k}
~\Longleftrightarrow~
\overF^{\times}_k(x)=\sum_{s\in\mathcal{S}}\alpha_s\,\overF_{s,k}(x),
\end{equation}
and similarly for $\widehat{\nu}^{\times}_k$.
For the (transport) barycenter, in one dimension the barycenter quantile is obtained
\emph{pointwise in $u$}:
\begin{equation}\label{eq:bar-quantile-pointwise}
\overline{Q}^{\ast}_k(u_\ell)\in\arg\min_{z\in\mathbb{R}}\ \left\lbrace\sum_{s\in\mathcal{S}}\alpha_s\,|q_{s,\ell}-z|^p\right\rbrace,
\text{ for } \ell\in[k].
\end{equation}
In particular, for $p=2$ one has $\overline{Q}^{\ast}_k(u_\ell)=\sum_s\alpha_s q_{s,\ell}$,
and for $p=1$ it is a weighted median of $\{q_{s,\ell}\}_s$.

\subsection{Discrete estimators and rates}\label{sec:quantile:estimators}

\monparagraph{Approximating the Wasserstein functional $U_p$.}
Define the discretized unfairness functional
\begin{equation}\label{eq:U-bar-kp}
\overline{U}_{k,p}
\;:=\;
\sum_{s\in\mathcal{S}} \alpha_s\, W_p(\overline{\nu}_{s,k},\overline{\nu}^{\ast}_k)^p,
\end{equation}
where $\overline{\nu}^{\ast}_k$ is any barycenter of $\{\overline{\nu}_{s,k}\}_s$.
In one dimension, thanks to \eqref{eq:Wp-quantile}, this reduces to the simple Riemann-sum form
\begin{equation}\label{eq:U-bar-kp-riemann}
\overline{U}_{k,p}
\;=\;
\sum_{s\in\mathcal{S}} \alpha_s\,\frac1k\sum_{\ell=1}^k |q_{s,\ell}-q^\ast_\ell|^p,
~
q^\ast_\ell := \overline{Q}^{\ast}_k(u_\ell),
\end{equation}
and the empirical plug-in estimator is obtained by replacing $q_{s,\ell}$ with $\wideQ_{s,\ell}$.

\begin{remark}[Interior grids and trimmed functionals]\label{rem:trimmed}
Our default grid uses midpoints $u_\ell=(\ell-\tfrac12)/k$, which already avoids the boundary levels $0$ and $1$ 
More generally, one may consider a trimmed version of the quantile representation by integrating over $[\varepsilon,1-\varepsilon]$:
\[
W_{2,\varepsilon}^2(\mu,\nu):=\int_{\varepsilon}^{1-\varepsilon}\bigl(Q_\mu(u)-Q_\nu(u)\bigr)^2\,du,
\]
and approximate it using interior levels (e.g.\ $u_\ell=\varepsilon+(\ell-\tfrac12)\frac{1-2\varepsilon}{k}$).
All discretization arguments extend to this trimmed setting, and boundary effects are removed by construction.
\end{remark}

\monparagraph{Approximating the Cram\'er functional $H_p$.}
Recall $C_p(\nu,\nu')^p=\int_{\mathbb{R}}|F_\nu(x)-F_{\nu'}(x)|^p\,dx$.
Define the discretized homogeneity functional
\begin{equation}\label{eq:H-bar-kp}
\overline{H}_{k,p}
\;:=\;
\sum_{s\in\mathcal{S}} \alpha_s\, C_p(\overline{\nu}_{s,k},\overline{\nu}^{\times}_k)^p,
\end{equation}
and similarly $\widehat{H}_{k,p}$ from $\widehat{\nu}_{s,k}$.

Because $\overF_{s,k}$ and $\overF^{\times}_k$ are step functions, the integral in \eqref{eq:Cp-cdf}
can be computed exactly from the breakpoints $\{q_{s,\ell}\}_{s,\ell}$.
Let $g_1<\dots<g_M$ be the sorted list of distinct values among $\{q_{s,\ell}\}_{s,\ell}$.
Then, for each $s$,
\begin{equation}\label{eq:Cp-step-computation}
C_p(\overline{\nu}_{s,k},\overline{\nu}^{\times}_k)^p
\;=\;
\sum_{m=1}^{M-1} (g_{m+1}-g_m)\,
\Bigl|\overF_{s,k}(g_m)-\overF^{\times}_k(g_m)\Bigr|^p,
\end{equation}
and the same formula holds with hats.
Thus, both $\overline{H}_{k,p}$ and $\widehat{H}_{k,p}$ are computable from quantile sketches
(up to sorting), with complexity $O((k|\mathcal{S}|)\log(k|\mathcal{S}|))$.

\monparagraph{Discretization bias.}
The next proposition states that $k$ quantiles suffice to approximate both centralized targets,
with an explicit $O(1/k)$~discretization bias under mild regularity.
The proof is deferred to the appendix.

\begin{proposition}[Consistency and discretization rate]\label{prop:discretization}
Assume that for each $s\in\mathcal{S}$, $\nu_s$ has compact support and admits a continuous density
bounded away from $0$ and $\infty$ on its support. Then, for $p\in\{1,2\}$, there exist constants
$C_U,C_H<\infty$ (depending only on these bounds and on $p$) such that
\[
\bigl|U_p-\overline{U}_{k,p}\bigr|\le \frac{C_U}{k}
\text{ and }
\bigl|H_p-\overline{H}_{k,p}\bigr|\le \frac{C_H}{k}.
\]
Moreover, for fixed $k$, the plug-in estimators $\widehat{U}_{k,p}$ and $\widehat{H}_{k,p}$ are
consistent as $n\to\infty$, and the total error decomposes into a statistical term (from estimating
quantiles) plus the deterministic discretization bias $O(1/k)$.
\end{proposition}

When densities may vanish near the boundary, the same $O(1/k)$ control holds on any trimmed interval $[\varepsilon,1-\varepsilon]$;

\begin{proposition}[Bin-averaged discretization underestimates $U_2$]\label{prop:barUk2_short}
Assume $p=2$ and $S=\{0,1\}$. With the notation above, for every $k\ge1$, \(\bar U_{k,2}\le U_2\).
\end{proposition}


\section{Federated measurement from $k$ quantiles}\label{sec:federated}

In the federated setting, the auditor (server) cannot access raw scores. Each silo $j\in[d]$
only observes its local conditional score laws ${\nu_{j,s}}_{s\in\mathcal{S}}$ and is asked
to release a small summary. We show that \emph{global} demographic disparity
(and pooled heterogeneity) can be recovered from one round of communication using quantile sketches.

\subsection{Federated functionals}\label{sec:federated:functionals}

\monparagraph{Local laws and selection weights.}
Recall $\nu_{j,s}=\mathcal{L}(Z\mid A=j,S=s)$ with cdf $F_{j,s}$ and quantile $Q_{j,s}$.
Define the within-group source weights $\pi_{j\mid s}=\mathbb{P}(A=j\mid S=s)$ and group
weights $\alpha_s=\mathbb{P}(S=s)$ (estimated by counts in practice).

\monparagraph{Mixture across sources (within each group).}
For each sensitive group $s$, the population score distribution induced by the federated system
is the mixture over sources
\begin{equation}\label{eq:nu-s-times}
\nu_s^{\times} := \sum_{j=1}^d \pi_{j\mid s}\,\nu_{j,s}
~\Longleftrightarrow~
F_s^{\times}(x)=\sum_{j=1}^d \pi_{j\mid s}\,F_{j,s}(x).
\end{equation}
We stress that quantile operators do \emph{not} generally commute with mixture distributions:
$Q^{\times}_s \neq \sum_j \pi_{j\mid s} Q_{j,s}$.

\monparagraph{Transport barycenters.}
We will also use a within-group transport center across sources:
\begin{equation}\label{eq:nu-s-star}
\nu_s^{\ast} \in \arg\min_{\nu}\ \left\lbrace\sum_{j=1}^d \pi_{j\mid s}\,W_p(\nu_{j,s},\nu)^p\right\rbrace,
\end{equation}
whose quantile satisfies, in one dimension,
$Q_s^{\ast}(u)=\sum_j \pi_{j\mid s}Q_{j,s}(u)$ when $p=2$, and is a weighted median when $p=1$.
At the population level, define the barycenter of the group-mixtures
\begin{equation}\label{eq:nu-times-star}
\nu^{\times\ast} \in \arg\min_{\nu}\ \sum_{s\in\mathcal{S}} \alpha_s\,W_p(\nu_s^{\times},\nu)^p.
\end{equation}

\begin{definition}[Federated demographic-parity functional]\label{def:Gp}
We define the federated (population-level) demographic disparity as
\begin{equation}\label{eq:Gp}
G_p \;:=\; \sum_{s\in\mathcal{S}} \alpha_s\,W_p\!\big(\nu_s^{\times},\nu^{\times\ast}\big)^p.
\end{equation}
\end{definition}

\monparagraph{Federated pooled heterogeneity.}
Let $\nu^{\times}:=\sum_{s\in\mathcal{S}}\alpha_s \nu_s^{\times}$ denote the pooled mixture
(with cdf $F^{\times}=\sum_s \alpha_s F_s^{\times}$). In parallel with Definition~\ref{def:Hp}, we consider
\begin{equation}\label{eq:H-fed}
H_p^{\times} \;:=\; \sum_{s\in\mathcal{S}} \alpha_s\,C_p\!\big(\nu_s^{\times},\nu^{\times}\big)^p .
\end{equation}

\begin{proposition}[Consistency with centralized targets]\label{prop:Gp-equals-Up}
With the notation of \Cref{sec:setting}, one has $\nu_s^{\times}=\nu_s$ for all $s$ and therefore
$G_p=U_p$. Similarly, $H_p^{\times}=H_p$.
\end{proposition}

\subsection{One-shot protocol}\label{sec:federated:protocol}

Each silo communicates only \emph{group counts} and a \emph{$k$-quantile sketch} per group on a shared grid
$\{u_\ell\}_{\ell\in[k]}$ as in \eqref{eq:u-grid}. This can be combined with secure aggregation, but we do not provide formal privacy guarantees (e.g., differential privacy) here. Stability of the step-CDF reconstruction and weight-estimation effects are detailed in Appendix~\ref{app:stability}.

\paragraph{One-shot protocol (overview).}
Each silo $j\in[d]$ releases, for each sensitive group $s\in\mathcal S$, the subgroup count $n_{j,s}$
and $k$ empirical quantiles on the shared midpoint grid $u_\ell=(\ell-\tfrac12)/k$.
From these summaries, the server reconstructs step-CDFs, forms within-group mixtures across silos,
and outputs plug-in estimators of the global disparity $G_p$ and pooled heterogeneity $H_p^\times$
(and, for $p=2$, diagnostic ANOVA terms).
Full pseudocode is given in Appendix~\ref{app:algo}, Algorithm~\ref{alg:oneshot}.

\monparagraph{Privacy considerations.}
Our protocol is \emph{privacy-motivated}: each silo shares only subgroup counts and a $k$-quantile sketch, and the server never receives individual scores.
In deployments, these summaries can be combined with secure aggregation to avoid exposing silo-level messages to the server \cite{bonawitz2017secureagg}.
To mitigate disclosure risks for small subpopulations, practitioners may enforce minimum cell sizes (e.g., suppress or coarsen releases when $n_{j,s}$ is too small), and/or add calibrated noise.
A full differential-privacy treatment is outside the scope of this paper, but our stability bounds provide a direct way to quantify how perturbations of shared summaries affect estimation error.

\subsection{Estimator and convergence}\label{sec:federated:convergence}

Algorithm~\ref{alg:oneshot} (Appendix~\ref{app:algo}) produces plug-in estimators of $G_p$ and $H_p$ based on $k$ quantiles per $(j,s)$.
The next proposition summarizes the two sources of error: (i) sampling error from estimating local quantiles,
and (ii) deterministic discretization bias due to the grid size $k$.

\begin{proposition}[Consistency and discretization rate]\label{prop:federated-consistency}
Assume the conditions of Proposition~\ref{prop:discretization} for each group-mixture law $\nu_s^{\times}$, and let
$n_{\min}:=\min_{j\in[d],\,s\in\mathcal{S}} n_{j,s}$. For any fixed $p\in\{1,2\}$ and any sequence
$k=k(n)\to\infty$ with $n_{\min}\to\infty$, the estimators produced by Algorithm~\ref{alg:oneshot} satisfy
\[
\widehat G_{k,p}\xrightarrow{\mathbb{P}} G_p,
\text{ and }
\widehat H^{\times}_{k,p}\xrightarrow{\mathbb{P}} H_p^{\times}.
\]
Moreover, their typical deviations obey
\[
\bigl|\widehat G_{k,p}-G_p\bigr|
\;=\;
O_{\mathbb{P}}\!\Bigl(\frac1k+\sqrt{\frac{\log k}{n_{\min}}}\Bigr),
\] 
and
\[
\bigl|\widehat H^{\times}_{k,p}-H_p^{\times}\bigr|
\;=\;
O_{\mathbb{P}}\!\Bigl(\frac1k+\sqrt{\frac{\log k}{n_{\min}}}\Bigr),
\]
where the $1/k$ term is the discretization bias and the second term comes from quantile estimation.
\end{proposition}

The $\log k$ term comes from uniform control on the $k$ grid points; it can be removed by working with a fully uniform-in-$u$ quantile process bound.



\begin{proposition}[High-probability control of communicated quantiles]\label{prop:hp-quantiles}
Fix $\varepsilon\in(0,1/2)$. Assume that for every $(j,s)$, the conditional score distribution
$\nu_{j,s}$ admits a cdf $F_{j,s}$ that is absolutely continuous on
$[Q_{j,s}(\varepsilon),Q_{j,s}(1-\varepsilon)]$ with density bounded below by $m_\varepsilon>0$
on this interval.
Let $n_{\min}:=\min_{j\in[d],\,s\in\cS} n_{j,s}$ and let $u_\ell=(\ell-\frac12)/k$.
Then, for any $\delta\in(0,1)$, with probability at least $1-\delta$,
\begin{eqnarray*}
&&\max_{\substack{j\in[d],\,s\in\cS\\ \ell\in\{\lceil \varepsilon k\rceil,\dots,\lfloor(1-\varepsilon)k\rfloor\}}}
\bigl|\widehat Q_{j,s}(u_\ell)-Q_{j,s}(u_\ell)\bigr|\\
&&\;\le\; \frac{1}{m_\varepsilon}\sqrt{\frac{1}{2n_{\min}}\log\!\Big(\frac{2(k+1)d|\cS|}{\delta}\Big)}.
\end{eqnarray*}
\end{proposition}

\begin{remark}
Proposition~\ref{prop:hp-quantiles} follows from the sharp DKW--Massart inequality and a
quantile-inversion argument on the trimmed region; see Appendix~\ref{app:hp:proof-prop}
\cite{shorackwellner1986,massart1990tight,vdvaartwellner1996}. We state a slightly looser logarithmic factor $\log((k+1)d|\mathcal S|/\delta)$ for notational convenience; a sharper bound without $\log(k)$ follows from the same argument.  
\end{remark}

\begin{corollary}[High-probability bound for $\widehat G_{k,2}$]\label{cor:hp-G}
Under the assumptions of Proposition~\ref{prop:hp-quantiles} and the regularity assumptions of
Proposition~\ref{prop:discretization} for the group-mixtures $\nu_s^\times$, there exists
$C_\varepsilon<\infty$ such that, with probability at least $1-\delta$,
\[
|\widehat G_{k,2}-G_2|
\;\le\;
C_\varepsilon\Big(\omega(1/k)+\sqrt{\frac{1}{n_{\min}}\log\!\big(\frac{(k+1)d|\cS|}{\delta}\big)}\Big),
\]
and similarly for the plug-in estimators of $(V_{\mathrm{mix}},V_{\mathrm{bar}},R)$.
\end{corollary}

A complete derivation (including knot stability and weight concentration) is given in Appendix~\ref{app:stability}.

\section{Local--global mismatch: ANOVA decomposition and bounds}\label{sec:anova}

This section explains why \emph{local} demographic-parity audits can be uninformative about the \emph{global} disparity.
We focus on $p=2$ in one dimension, where the squared Wasserstein distance admits the Hilbertian representation
$W_2(\mu,\nu)^2=\|Q_\mu-Q_\nu\|_{L^2([0,1])}^2$ in terms of quantile functions.
This geometry yields ANOVA-style decompositions analogous to classical variance identities.

Recall the group-mixtures $\nu_s^{\times}=\sum_{j=1}^d \pi_{j\mid s}\nu_{j,s}$, their barycenter $\nu^{\times\ast}$
from \eqref{eq:nu-times-star}, and the within-group transport centers $\nu_s^\ast$ from \eqref{eq:nu-s-star}.
For $p=2$ in one dimension, $Q_s^\ast(u)=\sum_{j=1}^d \pi_{j\mid s}Q_{j,s}(u)$.

\monparagraph{A decomposition with diagnostics.}
Define the \emph{mixture discrepancy} and the \emph{barycenter term}
\begin{equation}\label{eq:anova-terms}
V_{\mathrm{mix}} := \sum_{s\in\mathcal{S}}\alpha_s\,W_2(\nu_s^{\times},\nu_s^\ast)^2,
~
V_{\mathrm{bar}} := \sum_{s\in\mathcal{S}}\alpha_s\,W_2(\nu_s^{\ast},\nu^{\times\ast})^2.
\end{equation}

\begin{theorem}[ANOVA-style decomposition for $p=2$]\label{thm:anova}
In one dimension and for $p=2$, one has the exact identity
\begin{equation}\label{eq:anova-decomp}
G_2 \;=\; V_{\mathrm{mix}} + V_{\mathrm{bar}} + R,
\end{equation}
where the remainder term is
\[
R := 2\sum_{s\in\mathcal{S}}\alpha_s\,
\Big\langle Q_s^{\times}-Q_s^{\ast},\,Q_s^{\ast}-Q^{\times\ast}\Big\rangle_{L^2([0,1])}.
\]
Moreover, $|R|\le 2\sqrt{V_{\mathrm{mix}}\,V_{\mathrm{bar}}}$.
\end{theorem}

\begin{proof}[Proof sketch]
Write $Q_s^{\times}-Q^{\times\ast}=(Q_s^{\times}-Q_s^\ast)+(Q_s^\ast-Q^{\times\ast})$ and expand the squared
$L^2$ norm, then sum over $s$ with weights $\alpha_s$. The bound follows from Cauchy--Schwarz.
\end{proof}

\begin{proposition}[Two-sided bounds]\label{prop:anova-bounds}
Under the conditions of \Cref{thm:anova},
\[
\bigl(\sqrt{V_{\mathrm{mix}}}-\sqrt{V_{\mathrm{bar}}}\bigr)^2
\;\le\; G_2 \;\le\;
\bigl(\sqrt{V_{\mathrm{mix}}}+\sqrt{V_{\mathrm{bar}}}\bigr)^2.
\]
\end{proposition}

\monparagraph{Interpretation.}
The term $V_{\mathrm{mix}}$ quantifies how far each group-mixture $\nu_s^{\times}$ is from the within-group barycenter
$\nu_s^\ast$ (a non-commutativity effect between \emph{mixing} and \emph{barycentering}).
The term $V_{\mathrm{bar}}$ measures demographic disparity after replacing each mixture by its within-group transport
center; it captures a ``structural'' component that persists even when mixture discrepancies are small.
The interaction term $R$ can either amplify or cancel these effects, but is always controlled by
Theorem~\ref{thm:anova} (indeed $|R|\le 2\sqrt{V_{\mathrm{mix}}V_{\mathrm{bar}}}$).
This makes it possible to build Simpson-paradox-style examples where each silo appears fair locally while $G_2$ is large,
driven by $V_{\mathrm{mix}}$, or conversely, \citep{julious1994confounding,prost2022simpson}.

\monparagraph{Estimating diagnostics from the same sketches.}
For $p=2$, the server can compute a grid approximation of $\nu_s^\ast$ directly from the shared quantiles:
$\wideq^\ast_{s,\ell} := \sum_{j=1}^d \hat\pi_{j\mid s}\,\wideQ_{j,s,\ell}$.
Together with the mixture-quantiles $\wideq^{\times}_{s,\ell}$ and the global barycenter quantiles
$\wideq^{\times\ast}_\ell$ from Algorithm~\ref{alg:oneshot} (Appendix~\ref{app:algo}), define
\[
\hspace{-.5cm}\widehat V_{\mathrm{mix}} := \sum_{s\in\mathcal{S}}\frac{\hat\alpha_s}k\sum_{\ell=1}^k
\bigl(\wideq^{\times}_{s,\ell}-\wideq^{\ast}_{s,\ell}\bigr)^2,
\]
\[
\widehat V_{\mathrm{bar}} := \sum_{s\in\mathcal{S}}\frac{\hat\alpha_s}k\sum_{\ell=1}^k
\bigl(\wideq^{\ast}_{s,\ell}-\wideq^{\times\ast}_{\ell}\bigr)^2.
\]
These estimators incur the same discretization and quantile-estimation errors as $\widehat G_{k,2}$.

\subsection*{When do the terms vanish? Sufficient conditions}\label{sec:anova:indep}

The decomposition of \Cref{thm:anova} becomes especially interpretable under simple independence
assumptions relating $(A,S)$ and the score mechanism, $Z=m(X,S,A)$.

\begin{corollary}[Regimes of vanishing terms]\label{cor:vanishing}
Assume one-dimensional scores and $p=2$.
\begin{enumerate}
\item If $m(X,S,A)\indep A\mid S$ (no cross-silo heterogeneity within each sensitive group), then
$\nu_s^\times=\nu_s^\ast$ for all $s$, hence $V_{\mathrm{mix}}=0$ and $G_2=V_{\mathrm{bar}}$.

\item If $m(X,S,A)\perp\!\!\!\perp S \mid A$, then $\nu_{j,s}=\nu_j$ does not depend on $s$
(within-silo parity). In this regime, any global disparity is entirely driven by differences
in the selection weights $\pi_{j\mid s}$ across groups.
In particular, if $\pi_{j\mid s}=\beta_j$ for all $j,s$, then $\nu_s^{\times}$ does not depend on $s$
and therefore $G_2=0$.

\item If both conditions hold, then $G_2=0$ (global demographic parity).
\end{enumerate}
\end{corollary}

\begin{remark}[The case $p=1$: a triangle-inequality diagnostic]\label{rem:p1}
Our ANOVA identity is specific to the Hilbertian geometry of $W_2^2$. For $p=1$, we still obtain a
useful and sharp diagnostic from the triangle inequality.
Recall $G_1=\sum_{s\in\cS}\alpha_s W_1(\nu_s^\times,\nu^{\times\ast})$.
Define the two (nonnegative) terms
\[
V^{(1)}_{\mathrm{mix}} := \sum_{s\in\cS}\alpha_s\,W_1(\nu_s^\times,\nu_s^\ast),
~
V^{(1)}_{\mathrm{bar}} := \sum_{s\in\cS}\alpha_s\,W_1(\nu_s^\ast,\nu^{\times\ast}),
\]
where $\nu_s^\ast$ is the within-group Wasserstein barycenter across silos \eqref{eq:nu-s-star} and $\nu^{\times\ast}$ is the barycenter of group-mixtures \eqref{eq:nu-times-star}.
Then, by the (reverse) triangle inequality applied for each $s$ and Jensen's inequality,
\begin{equation}\label{eq:p1-bounds}
\Bigl|V^{(1)}_{\mathrm{bar}}-V^{(1)}_{\mathrm{mix}}\Bigr|
\;\le\; G_1 \;\le\;
V^{(1)}_{\mathrm{bar}}+V^{(1)}_{\mathrm{mix}}.
\end{equation}
In one dimension, $W_1(\mu,\nu)=\|Q_\mu-Q_\nu\|_{L^1([0,1])}$ and $\nu_s^\ast$ can be computed pointwise in $u$ as a weighted median of $\{Q_{j,s}(u)\}_{j\in[d]}$ (weights $\pi_{j\mid s}$), so the same $k$-quantile sketches as in Algorithm~1 yield plug-in estimators $\widehat V^{(1)}_{\mathrm{mix}}$ and $\widehat V^{(1)}_{\mathrm{bar}}$ by replacing integrals with Riemann sums over $\{u_\ell\}_{\ell\in[k]}$.
\end{remark}

\monparagraph{Practical recipe (one-round diagnosis).}
Algorithm~\ref{alg:oneshot} (Appendix~\ref{app:algo}) returns both a global disparity estimate $\widehat G_{k,2}$ and diagnostic terms $(\widehat V_{\mathrm{mix}},\widehat V_{\mathrm{bar}},\widehat R)$ from the same quantile sketches.
A simple interpretation is: (i) if $\widehat V_{\mathrm{mix}}$ dominates, the global gap is primarily driven by \emph{selection/composition} effects (Simpson-style mixtures across silos); (ii) if $\widehat V_{\mathrm{bar}}$ dominates, the gap is primarily driven by \emph{cross-silo heterogeneity} within groups.
In our experiments, a few dozen quantiles ($k\in[25,100]$) already yield stable estimates.

\section{Experiments}\label{sec:experiments}

Our experiments illustrate three points: (i) \emph{local} demographic parity audits may be misleading under silo-specific selection and composition shifts, (ii) the proposed one-shot quantile protocol accurately reconstructs the \emph{global} disparity using only a few quantiles, and (iii) the ANOVA decomposition for $p=2$ provides actionable diagnostics separating mixture effects from cross-silo heterogeneity.

\subsection{Experimental protocol and reported quantities}\label{sec:exp:protocol}

Unless stated otherwise, we focus on $p=2$.
Each silo $j$ computes, for each group $s\in\mathcal{S}$, the empirical quantiles $\{\widehat{Q}_{j,s}(u_\ell)\}_{\ell\in[k]}$ on a shared grid, and transmits them along with group counts $\{n_{j,s}\}_{s\in\mathcal{S}}$ (Algorithm~\ref{alg:oneshot},  Appendix~\ref{app:algo}). The server reconstructs the mixture cdfs $\widehat{F}^{\times}_{s,k}$, inverts them on the same grid to obtain mixture-quantiles $\{\widehat{q}^{\times}_{s,\ell}\}$, and computes:
(i) the global DP functional $\widehat{G}_{k,2}$,
(ii) the pooled heterogeneity $\widehat{H}^{\times}_{k,2}$, and
(iii) when applicable, the decomposition terms from \Cref{sec:anova} (mixture, barycenter, remainder). We also report local audits, computed by applying the same quantile estimators within each silo.

\subsection{Synthetic data: discretization and sketch accuracy}\label{sec:exp:synthetic}

We first consider a controlled two-group setting where scores are generated from Beta laws on $[0,1]$,
with distinct shapes across sensitive groups,
and a balanced sample of size $n=10{,}000$. 
Our generating process allows us to smoothly interpolate between independent allocation ($\rho=0$) and strong selection bias ($|\rho|\to 1$), where $\rho$ controls the correlation between the score $Z$ and the silo assignment $A$ via a latent Gaussian pairing.
Figure~\ref{fig:synthetic-beta-1} reports the resulting marginal group distributions (left), the impact of $\rho$ on the coarse-grid estimates $U_2(k)$ for $k\in\{1,2\}$ (right).

\begin{figure}[!htbp]
    \centering
    \includegraphics[width=0.49\linewidth]{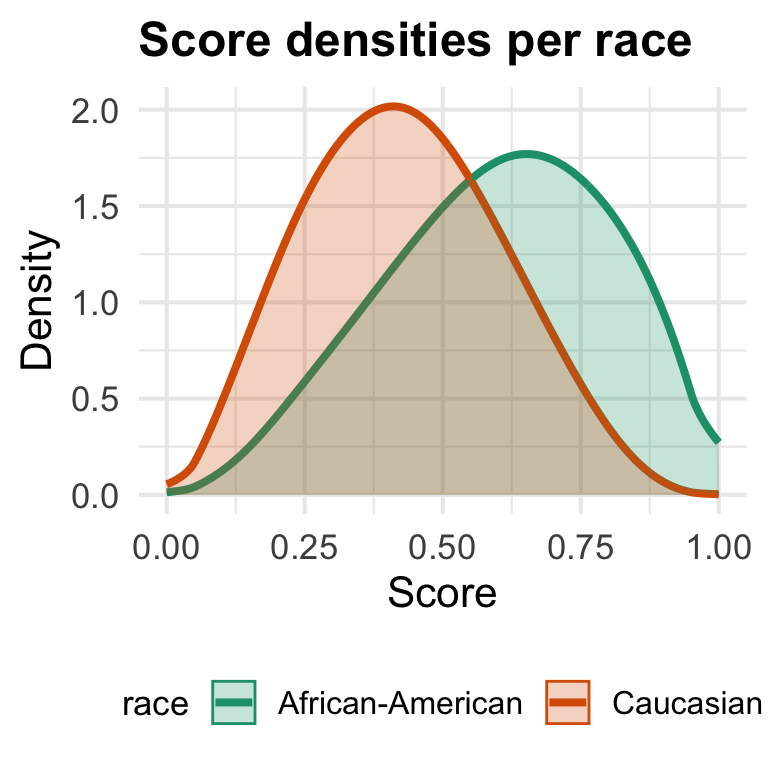}
    \includegraphics[width=0.49\linewidth]{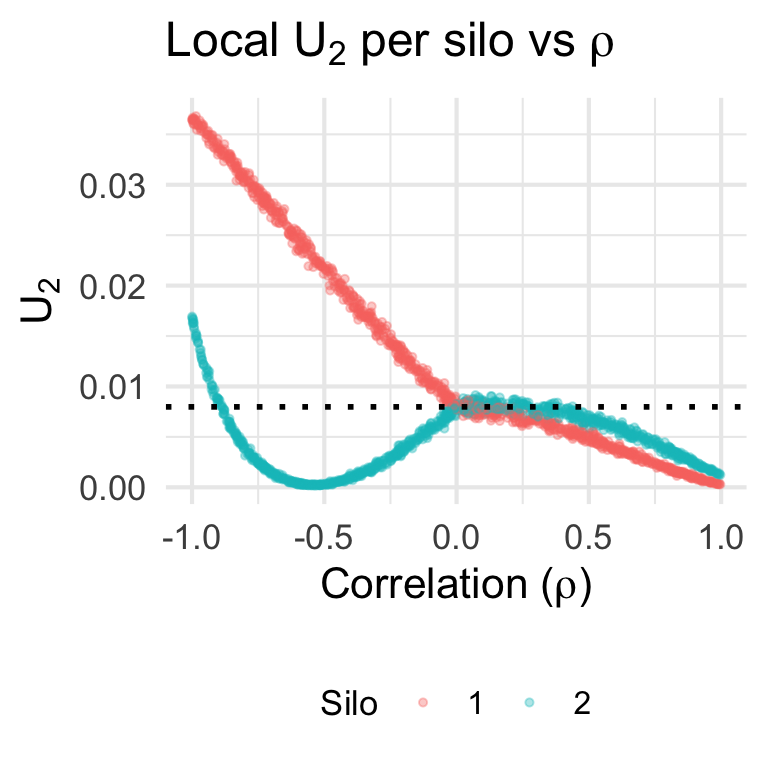}

    \caption{\textbf{Synthetic Beta.} {\em Left}: score distributions by group, $U_2=0.0076$. {\em Right}: $U_2(k)$ with $k\in\{1,2\}$ as a function of $\rho$. 
    }
    \label{fig:synthetic-beta-1}
\end{figure}

We further visualize the sensitivity to the grid size $k$ (discretization) and to sample size, confirming the $O(1/k)$ trend predicted by Proposition~\ref{prop:discretization} and the fast stabilization of the global estimates with a few dozen quantiles  on Figures \ref{fig:synthetic-beta-u2-k} 


\begin{figure}[!htbp]
    \centering
    \includegraphics[width=0.49\linewidth]{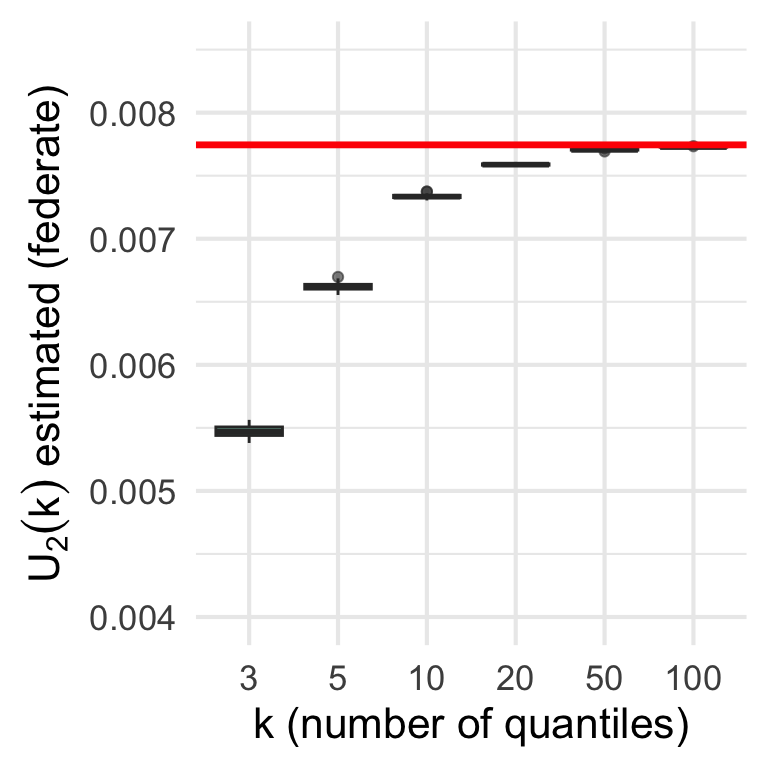}
    \includegraphics[width=0.49\linewidth]{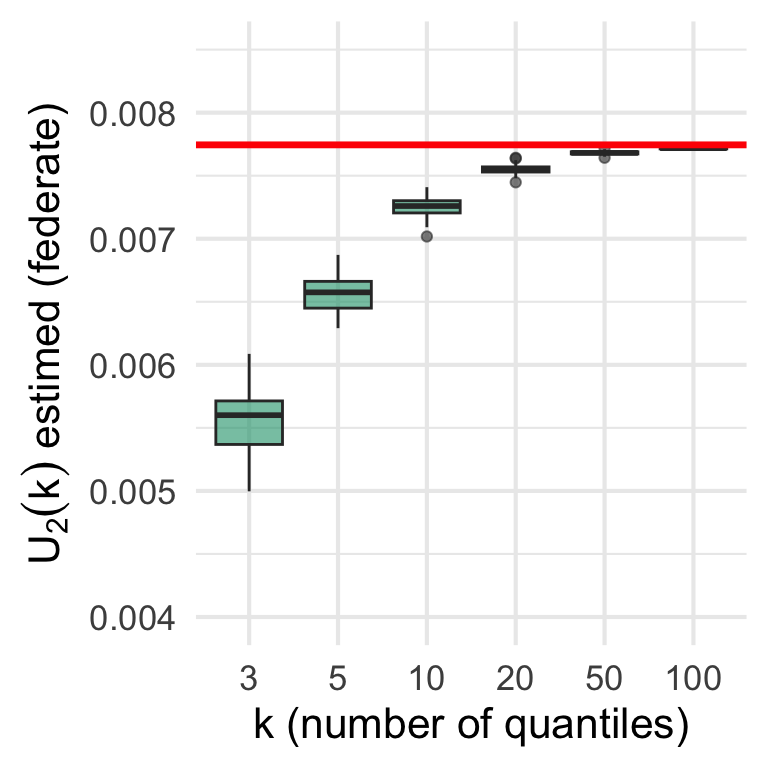}
    \caption{\textbf{Sensitivity to $k$ (synthetic).} Convergence of $U_{2}(k)$, independence allocation (left) and selection bias (right).}
    \label{fig:synthetic-beta-u2-k}
\end{figure}


For $p=2$, the bin-averaged discretization enjoys a monotone convergence from below (Proposition~\ref{prop:barUk2_short}), which helps interpret the sensitivity-to-$k$ behavior in Figure~\ref{fig:synthetic-beta-u2-k}.

\subsection{COMPAS: local audits can hide global disparity}\label{sec:exp:compas}

We next analyze the \texttt{COMPAS} recidivism scores (two sensitive groups: African-American vs.\ Caucasian). We create $d=5$ artificial silos.
Our goal is to reproduce a scope mismatch: each silo may exhibit only mild distributional disparity, while the pooled population shows a stronger gap.

\monparagraph{Simulated allocations across silos with selection bias.}
We first allocate observations randomly across the five silos.
Table~\ref{tab:compas_random} reports the resulting local and global disparities.
While each local silo exhibits moderate $U_2$, the global $U_2$ is substantially larger, illustrating that \emph{local demographic-parity audits do not characterize the global parity of the federation}.

\begin{table}[t]
\centering
\caption{\textbf{\texttt{COMPAS}, simulated silos with selection bias.}
African-American (AA) vs.\ Caucasian (C) scores, with $d=5$ silos.}
\label{tab:compas_random}
{\footnotesize
\begin{tabular}[t]{lrrrrrr}
\toprule
 & Global & Silo 1 & Silo 2 & Silo 3 & Silo 4 & Silo 5\\
\midrule
\cellcolor{gray!10}{$n_{\text{AA}}$} & \cellcolor{gray!10}{3696} & \cellcolor{gray!10}{535} & \cellcolor{gray!10}{704} & \cellcolor{gray!10}{745} & \cellcolor{gray!10}{814} & \cellcolor{gray!10}{898}\\
$n_{\text{C}}$ & 2454 & 695 & 536 & 521 & 407 & 295\\
\cellcolor{gray!10}{$|\Delta\bar{z}|$} & \cellcolor{gray!10}{1.6357} & \cellcolor{gray!10}{0.4483} & \cellcolor{gray!10}{0.9784} & \cellcolor{gray!10}{1.0978} & \cellcolor{gray!10}{0.9667} & \cellcolor{gray!10}{0.5654}\\
$U_2$ & 0.7609 & 0.0748 & 0.2634 & 0.3073 & 0.2355 & 0.0745\\
\cellcolor{gray!10}{$H_2$} & \cellcolor{gray!10}{0.0077} & \cellcolor{gray!10}{0.0017} & \cellcolor{gray!10}{0.0045} & \cellcolor{gray!10}{0.0044} & \cellcolor{gray!10}{0.0035} & \cellcolor{gray!10}{0.0012}\\
$W_2$ & 1.7813 & 0.5479 & 1.0324 & 1.1319 & 1.0104 & 0.6326\\
\bottomrule
\end{tabular}}
\end{table}


\monparagraph{Selection-biased partitions and ANOVA diagnostics.}
We finally consider non-random allocations that induce composition shifts across silos (selection bias).
In these settings, the ANOVA decomposition for $p=2$ helps disentangle whether the global disparity is driven primarily by mixture effects (selection/composition) or by genuine cross-silo heterogeneity.
Full tables and plots (including the per-silo quantile sketches) are provided in Appendix~\ref{app:compas}.

Two additional datasets are also considered, \texttt{Adult} (Appendix~\ref{app:adult}) and \texttt{Law} (Appendix~\ref{app:law}).

\section{Discussion and conclusion}\label{sec:discussion}

Our work highlights a pitfall in federated fairness auditing: \emph{local} demographic-parity conclusions do not transfer to the \emph{population} level. Two mechanisms drive this mismatch. First, silo-specific selection changes mixing weights $\pi_{j\mid s}$, so pooling induces mixtures whose quantiles differ from barycentric aggregates. Second, cross-silo heterogeneity can create hidden stratification effects missed by single-silo audits. Our federated functionals and the $p=2$ decomposition make these mechanisms explicit, yielding a global disparity estimate and an interpretable diagnostic of its sources.

From a practical standpoint, the proposed one-shot protocol provides a communication-efficient audit primitive: each silo releases only group counts and a small quantile sketch of its local score laws.
Despite this minimal disclosure, the server can consistently estimate population-level unfairness (and pooled heterogeneity), with explicit discretization bias $O(1/k)$ and controllable statistical error. This positions quantile sketches as a useful interface between privacy constraints and distributional fairness monitoring.

\monparagraph{Limitations.}
Our approach focuses on one-dimensional score distributions and on demographic parity. Extending the framework
to multi-dimensional outputs (e.g., calibrated probabilities for multiple classes) would require different
distributional summaries and may lose the simplicity of one-dimensional quantiles.
Our guarantees rely on regularity assumptions (e.g., bounded densities) to obtain clean $O(1/k)$ rates; studying
heavy-tailed or discrete score distributions more systematically is an important direction.
Finally, while our protocol is privacy-motivated by design (it minimizes what is shared), quantile sketches and
counts do not, by themselves, provide formal differential privacy guarantees; integrating calibrated noise and
analyzing the privacy--accuracy trade-off is left to future work.

\monparagraph{Conclusion.}
Federated auditing should target the population-level distributional criterion of interest, rather than relying
on local metrics as proxies. By combining distributional fairness functionals with a one-round quantile-sketch
protocol and sharp $p=2$ diagnostics, we provide a principled and practical approach to estimating global
demographic disparity under fragmentation and selection.

\clearpage

\section*{Impact Statement}

This paper studies how to \emph{measure} demographic disparities in federated deployments under privacy and
communication constraints. A positive impact is to enable more reliable compliance auditing and monitoring:
our results show that local fairness checks can be misleading, and provide a lightweight protocol for estimating
population-level disparities from minimal summaries. This may support regulators and practitioners in detecting
selection-driven unfairness that would otherwise go unnoticed.

Potential negative impacts include misuse or over-reliance on any single metric. Demographic parity captures one
notion of fairness, and accurate measurement does not by itself imply that a system is socially acceptable, nor
does it prescribe a mitigation strategy. Another risk is privacy leakage: although our protocol shares only counts
and quantile sketches, these summaries may still reveal information about small subpopulations. Deployments should
use appropriate safeguards (e.g., minimum cell sizes, secure aggregation, or formal differential privacy) and
interpret results in context.

Overall, we expect the primary impact to be methodological: clarifying when federated fairness audits can fail,
and providing practical tools to evaluate population-level disparities while respecting data fragmentation.

\bibliography{biblio}
\bibliographystyle{icml2026}

\clearpage
\appendix


\section{Algorithm}\label{app:algo}


\vspace{2cm}
\begin{algorithm}[htbp]
\caption{One-shot federated audit from $k$ quantiles.}
\label{alg:oneshot}
\begin{algorithmic}[1]
\Require Midpoint grid $u_\ell=(\ell-\tfrac12)/k$, $\ell\in[k]$.
\Require For each silo $j\in[d]$ and group $s\in\mathcal S$: subgroup count $n_{j,s}$ and empirical quantiles
$\widehat q_{j,s,\ell}=\widehat Q_{j,s}(u_\ell)$, $\ell\in[k]$, of $\{Z_i: A_i=j,\,S_i=s\}$.
\Ensure Estimates $\widehat G_{k,p}$ and $\widehat H^\times_{k,p}$ (and optionally $\widehat V_{\mathrm{mix}},\widehat V_{\mathrm{bar}},\widehat R$ for $p=2$).

\State Compute totals $n_s=\sum_{j=1}^d n_{j,s}$, $n=\sum_{s\in\mathcal S} n_s$.
\State Compute weights $\widehat\alpha_s=n_s/n$ and $\widehat\pi_{j\mid s}=n_{j,s}/n_s$.

\For{$(j,s)\in[d]\times\mathcal S$}
  \State Build the step-CDF induced by the quantile sketch:
  \[
    \widehat F_{j,s,k}(x) := \frac1k\sum_{\ell=1}^k \mathbf 1\{\widehat q_{j,s,\ell}\le x\}.
  \]
\EndFor

\For{$s\in\mathcal S$}
  \State Form the mixture step-CDF across silos:
  \[
    \widehat F^\times_{s,k}(x) := \sum_{j=1}^d \widehat\pi_{j\mid s}\,\widehat F_{j,s,k}(x)
    = \frac1k\sum_{j=1}^d\sum_{\ell=1}^k \widehat\pi_{j\mid s}\,\mathbf 1\{\widehat q_{j,s,\ell}\le x\}.
  \]
  \State Invert to obtain mixture quantiles on the grid:
  \[
    \widehat q^\times_{s,\ell} := \inf\{x:\widehat F^\times_{s,k}(x)\ge u_\ell\},\qquad \ell\in[k].
  \]
  \State \emph{Implementation note:} compute all $\widehat q^\times_{s,\ell}$ by sorting the multiset
  $\{\widehat q_{j,s,\ell'}: j\in[d],\,\ell'\in[k]\}$ and scanning cumulative mass increments
  $\frac1k\widehat\pi_{j\mid s}$.
\EndFor

\For{$\ell\in[k]$}
  \State Compute barycenter quantile across groups at level $u_\ell$:
  \[
    \widehat q^{\times *}_\ell \in \arg\min_{z\in\mathbb R}\sum_{s\in\mathcal S}\widehat\alpha_s\,|\widehat q^\times_{s,\ell}-z|^p.
  \]
  \If{$p=2$} \State $\widehat q^{\times *}_\ell \gets \sum_{s\in\mathcal S}\widehat\alpha_s\,\widehat q^\times_{s,\ell}$. \EndIf
\EndFor
\end{algorithmic}
\end{algorithm}

\begin{algorithm}[htbp]
\begin{algorithmic}[1]
\State Output disparity estimate:
\[
\widehat G_{k,p} := \sum_{s\in\mathcal S}\widehat\alpha_s\cdot \frac1k\sum_{\ell=1}^k \big|\widehat q^\times_{s,\ell}-\widehat q^{\times *}_\ell\big|^p.
\]

\State Form pooled mixture step-CDF:
\[
\widehat F^\times_k(x):=\sum_{s\in\mathcal S}\widehat\alpha_s\,\widehat F^\times_{s,k}(x).
\]

\For{$s\in\mathcal S$}
  \State Compute $C_p(\widehat F^\times_{s,k},\widehat F^\times_k)^p$ exactly using the step-function formula:
  let $g_1<\cdots<g_M$ be the sorted set of jump locations of $\widehat F^\times_{s,k}$ and $\widehat F^\times_k$,
  then
  \[
  C_p(\widehat F^\times_{s,k},\widehat F^\times_k)^p
  =\sum_{m=1}^{M-1} (g_{m+1}-g_m)\, \big|\widehat F^\times_{s,k}(g_m)-\widehat F^\times_k(g_m)\big|^p.
  \]
\EndFor

\State Output pooled heterogeneity estimate:
\[
\widehat H^\times_{k,p} := \sum_{s\in\mathcal S}\widehat\alpha_s\; C_p(\widehat F^\times_{s,k},\widehat F^\times_k)^p.
\]

\If{$p=2$}
  \For{$s\in\mathcal S$}
    \State Compute within-group (W$_2$) barycenter quantiles across silos:
    $\widehat q^*_{s,\ell} := \sum_{j=1}^d \widehat\pi_{j\mid s}\,\widehat q_{j,s,\ell}$ for $\ell\in[k]$.
  \EndFor
  \State Output diagnostic ANOVA terms (Riemann sums on the midpoint grid):
  \[
  \widehat V_{\mathrm{mix}}:=\sum_s \widehat\alpha_s\cdot\frac1k\sum_{\ell=1}^k(\widehat q^\times_{s,\ell}-\widehat q^*_{s,\ell})^2,\]
  \[
  \widehat V_{\mathrm{bar}}:=\sum_s \widehat\alpha_s\cdot\frac1k\sum_{\ell=1}^k(\widehat q^*_{s,\ell}-\widehat q^{\times *}_\ell)^2,
  \]
  \[
  \widehat R:=\frac{2}{k}\sum_{\ell=1}^k\sum_s \widehat\alpha_s(\widehat q^\times_{s,\ell}-\widehat q^*_{s,\ell})(\widehat q^*_{s,\ell}-\widehat q^{\times *}_\ell).
  \]
\EndIf
\end{algorithmic}
\end{algorithm}

\clearpage
\onecolumn

\section{Stepwise vs. piecewise-linear reconstruction}\label{app:reco}

\paragraph{Setup.}
Fix $k\ge 2$ and the midpoint grid $u_\ell := (\ell-\tfrac12)/k$, $\ell\in[k]$.
Let $Q$ be a (generalized) quantile function, and denote the communicated knots
$q_\ell := Q(u_\ell)$, $\ell\in[k]$.
We compare two reconstructions of $Q$ from the vector $(q_\ell)_{\ell\in[k]}$.

\paragraph{Stepwise (as in Algorithm~1).}
Define the step-quantile reconstruction
\begin{equation}\label{eq:Q-step}
Q_k^{\mathrm{st}}(u) := q_\ell \qquad \text{for } u\in I_\ell := \big[(\ell-1)/k,\;\ell/k\big),
\end{equation}
and the associated step-CDF
\begin{equation}\label{eq:F-step}
F_k^{\mathrm{st}}(x) := \frac1k \sum_{\ell=1}^k \mathbf{1}\{q_\ell \le x\}.
\end{equation}
This is exactly the reconstruction used in the paper (server-side), since it only
requires comparisons against the received quantiles.

\paragraph{Piecewise-linear (in $u$).}
Define the piecewise-linear reconstruction on $[u_1,u_k]$ by connecting consecutive knots:
for $u\in[u_\ell,u_{\ell+1}]$ (with $\ell=1,\dots,k-1$),
\begin{equation}\label{eq:Q-lin}
Q_k^{\mathrm{lin}}(u)
:= q_\ell + \frac{u-u_\ell}{u_{\ell+1}-u_\ell}\,(q_{\ell+1}-q_\ell)
= q_\ell + k\,(u-u_\ell)\,(q_{\ell+1}-q_\ell),
\end{equation}
and extend it by constants on the boundary half-bins:
$Q_k^{\mathrm{lin}}(u):=q_1$ for $u\in[0,u_1]$ and $Q_k^{\mathrm{lin}}(u):=q_k$ for $u\in[u_k,1]$.
(Equivalently, one may view this as working with the trimmed interval $[u_1,u_k]$.)

\subsection{Deterministic approximation bounds}

Let $\omega_Q(\delta):=\sup\{|Q(u)-Q(v)|:\ |u-v|\le \delta\}$ denote the modulus of continuity of $Q$.

\begin{proposition}[Stepwise vs. linear approximation]\label{prop:interp}
Let $Q$ be nondecreasing on $[0,1]$ and let $\omega_Q(\delta):=\sup\{|Q(u)-Q(v)|:\ |u-v|\le \delta\}$ be its modulus of continuity.
Then
\begin{align}
\|Q_k^{\mathrm{st}}-Q\|_{\infty,[u_1,u_k]}
&\le \omega_Q(1/k), \label{eq:st-mod}\\
\|Q_k^{\mathrm{lin}}-Q\|_{\infty,[u_1,u_k]}
&\le \omega_Q(1/k). \label{eq:lin-mod}
\end{align}
If moreover $Q$ is $L$-Lipschitz on $[0,1]$, then both errors are $\le L/k$ on $[u_1,u_k]$.
If in addition $Q\in C^2([u_1,u_k])$, then the linear reconstruction satisfies the sharper bound
\begin{equation}\label{eq:lin-C2}
\|Q_k^{\mathrm{lin}}-Q\|_{\infty,[u_1,u_k]}
\le \frac{\|Q''\|_{\infty,[u_1,u_k]}}{8k^2}.
\end{equation}
\end{proposition}

\noindent
\emph{Proof sketch.}
The bounds are the standard step/linear interpolation controls on a uniform grid of mesh $1/k$ (here on the interior interval).
The $C^2$ bound~\eqref{eq:lin-C2} is the classical linear-interpolation error estimate \citep{arbenz2018piecewise,irpino2007optimal}.

\paragraph{CDF control.}
For the step-CDF~\eqref{eq:F-step}, one always has a uniform bound of order $1/k$ (no smoothness needed);
in particular the paper already uses such a statement in the proof path from quantiles to reconstruction error.

\subsection{Implications for $U_2$ discretization}

We focus on the canonical $p=2$ quantile functional
\[
U_2 \;:=\; \sum_{s\in\mathcal{S}} \alpha_s \int_0^1 \bigl(Q_s(u)-Q^\star(u)\bigr)^2\,du,
\qquad Q^\star(u):=\sum_{s\in\mathcal{S}}\alpha_s Q_s(u),
\]
and its discretizations based on the received quantiles $(q_{s,\ell})$ on the midpoint grid.

\paragraph{(1) Midpoint/Riemann discretization (what Algorithm~1 effectively uses).}
Define, for any reconstructed family $\widetilde Q_s$,
\begin{equation}\label{eq:U-mid}
U_{2,k}^{\mathrm{mid}}(\widetilde Q)
:= \sum_s \alpha_s \frac1k \sum_{\ell=1}^k
\Big(\widetilde Q_s(u_\ell)-\widetilde Q^\star(u_\ell)\Big)^2,
\qquad \widetilde Q^\star := \sum_s \alpha_s \widetilde Q_s.
\end{equation}
Even if $\widetilde Q_s=Q_k^{\mathrm{lin}}$ has interpolation error $O(k^{-2})$ under $C^2$,
the \emph{midpoint rule} still incurs an $O(k^{-1})$ quadrature error for Lipschitz integrands.
In particular, if $g$ is $L_g$-Lipschitz then
\[
\left|\int_0^1 g(u)\,du - \frac1k\sum_{\ell=1}^k g(u_\ell)\right|
\le \frac{L_g}{2k}
\]
by Lemma~A.2 (midpoint rule) in the paper.

\paragraph{(2) Second-order discretization with no extra communication (exact integration of the linear reconstruction).}
If we use the linear reconstruction~\eqref{eq:Q-lin}, then on each interval $[u_\ell,u_{\ell+1}]$
the function $u\mapsto d_{s,k}(u):=Q_{s,k}^{\mathrm{lin}}(u)-Q_{k}^{\mathrm{lin}\,\star}(u)$ is \emph{linear},
hence $d_{s,k}^2$ is quadratic and its integral is available in closed form.
Let $d_{s,\ell}:=q_{s,\ell}-q^\star_\ell$ with $q^\star_\ell:=\sum_s \alpha_s q_{s,\ell}$.
With the constant extensions on $[0,u_1]$ and $[u_k,1]$, we have the exact identity
\begin{equation}\label{eq:U-lin-exact}
\int_0^1 d_{s,k}(u)^2\,du
=
\frac{1}{2k}d_{s,1}^2
\;+\;
\sum_{\ell=1}^{k-1}\frac{1}{3k}\Big(d_{s,\ell}^2 + d_{s,\ell}d_{s,\ell+1}+d_{s,\ell+1}^2\Big)
\;+\;
\frac{1}{2k}d_{s,k}^2.
\end{equation}
This yields a fully explicit estimator
\begin{equation}\label{eq:U2-lin-exact}
U_{2,k}^{\mathrm{lin}}
:= \sum_s \alpha_s
\left[
\frac{1}{2k}d_{s,1}^2
+
\sum_{\ell=1}^{k-1}\frac{1}{3k}\Big(d_{s,\ell}^2 + d_{s,\ell}d_{s,\ell+1}+d_{s,\ell+1}^2\Big)
+
\frac{1}{2k}d_{s,k}^2
\right],
\end{equation}
computed \emph{only} from the transmitted quantiles (no additional messages).

\begin{proposition}[Order gain with exact integration]\label{prop:order-gain}
Assume each group-quantile function $Q_s$ belongs to $C^2([u_1,u_k])$ with $\|Q_s''\|_{\infty,[u_1,u_k]}<\infty$.
Then the linear-exact estimator~\eqref{eq:U2-lin-exact} satisfies
\[
|U_{2,k}^{\mathrm{lin}} - U_2|
= O(k^{-2})
\quad\text{(up to boundary/trim effects controlled by the constant extensions).}
\]
By contrast, midpoint/Riemann estimators of the form~\eqref{eq:U-mid}
remain $O(k^{-1})$ unless one upgrades the quadrature rule.
\end{proposition}

\noindent
\emph{Proof sketch.}
Because $Q^\star=\sum_s \alpha_s Q_s$, the same smoothness holds for $Q^\star$,
and Proposition~\ref{prop:interp} gives $\|Q_{s,k}^{\mathrm{lin}}-Q_s\|_\infty=O(k^{-2})$ (and similarly for $Q^\star$).
A mean-value bound for squares then yields
$\int_0^1 \big|d_{s,k}(u)^2-d_s(u)^2\big|\,du = O(k^{-2})$.
Equation~\eqref{eq:U-lin-exact} removes quadrature error entirely since the integral of $d_{s,k}^2$ is computed exactly.
Summing over $s$ with weights $\alpha_s$ concludes.

\section{Technical lemmas}\label{app:tech}

\subsection{Quantile representation in one dimension}\label{app:tech:quantiles}

\begin{lemma}[Wasserstein via quantiles]\label{lem:Wp-quantiles}
Let $\mu,\nu$ be probability measures on $\R$ with finite $p$-th moments ($p\ge 1$).
Then
\[
W_p(\mu,\nu)^p \;=\; \int_0^1 \big|Q_\mu(u)-Q_\nu(u)\big|^p\,du
\;=\; \|Q_\mu-Q_\nu\|_{L^p([0,1])}^p.
\]
In particular, for $p=2$ one has $W_2(\mu,\nu)^2=\|Q_\mu-Q_\nu\|_{L^2([0,1])}^2$.
\end{lemma}
\noindent See, e.g., \citep{vallender1974calculation,villani2009optimal,santambrogio2015optimal}.

\subsection{Midpoint Riemann sums}\label{app:tech:riemann}

\begin{lemma}[Midpoint rule for Lipschitz functions]\label{lem:midpoint-riemann}
Let $g:[0,1]\to\R$ be $L_g$-Lipschitz. For the midpoint grid $u_\ell=(\ell-\frac12)/k$,
\[
\Big|\int_0^1 g(u)\,du \;-\; \frac1k\sum_{\ell=1}^k g(u_\ell)\Big|
\;\le\; \frac{L_g}{2k}.
\]
\end{lemma}

\begin{proof}
Partition $[0,1]$ into intervals $I_\ell=[(\ell-1)/k,\ell/k]$.
For any $u\in I_\ell$, $|u-u_\ell|\le 1/(2k)$, hence
$|g(u)-g(u_\ell)|\le L_g/(2k)$. Therefore
\[
\Big|\int_{I_\ell} g(u)\,du - \frac1k g(u_\ell)\Big|
\le \int_{I_\ell} |g(u)-g(u_\ell)|\,du
\le \frac1k\cdot \frac{L_g}{2k}.
\]
Summing over $\ell$ yields the claim.
\end{proof}

\subsection{Deterministic CDF approximation from $k$ quantiles}\label{app:tech:cdf-approx}

Recall the midpoint grid $u_\ell=(\ell-\frac12)/k$ and deterministic knots
$q_\ell = Q(u_\ell)$. Define the atomic approximation
$\bar\nu_k := \frac1k\sum_{\ell=1}^k \delta_{q_\ell}$ and its step-CDF
$\bar F_k(x)=\frac1k\sum_{\ell=1}^k \ind{q_\ell\le x}$.

\begin{lemma}[Uniform CDF error]\label{lem:cdf-unif}
Assume $F$ is continuous and strictly increasing on its support. Then
\[
\|\bar F_k - F\|_{\infty} \;\le\; \frac{1}{2k}.
\]
\end{lemma}

\begin{proof}
Fix $x\in\R$ and set $u:=F(x)\in[0,1]$. Let $\ell$ be such that
$u\in[(\ell-1)/k,\ell/k]$. Since $u_\ell=(\ell-\frac12)/k$,
we have $|u-u_\ell|\le 1/(2k)$. Moreover, by monotonicity of $Q$,
$q_\ell\le x$ iff $u_\ell\le u$ (up to continuity/tie events).
Thus $\bar F_k(x)=\ell'/k$ for some $\ell'$ satisfying
$|\ell'/k-u|\le 1/(2k)$, proving $|\bar F_k(x)-F(x)|\le 1/(2k)$.
\end{proof}

\subsection{From DKW to quantile errors}\label{app:tech:dkw}

\begin{lemma}[DKW inequality]\label{lem:dkw}
Let $\hat F_n$ be the empirical CDF of $n$ i.i.d.\ samples from $F$.
Then for any $\delta\in(0,1)$,
\[
\Pr\!\Big(\|\hat F_n-F\|_\infty \ge t\Big)\le 2e^{-2nt^2},
\qquad
t=\sqrt{\frac{1}{2n}\log\!\frac{2}{\delta}}.
\]
\end{lemma}

\begin{lemma}[Quantile inversion is Lipschitz on central regions]\label{lem:quantile-lip-cr}
Fix $\varepsilon\in(0,1/2)$. Assume $F$ is strictly increasing on
$[Q(\varepsilon),Q(1-\varepsilon)]$ with density bounded below by $m_\varepsilon>0$
on that interval. Then for any $u\in[\varepsilon,1-\varepsilon]$ and any CDF $G$,
\[
|Q_G(u)-Q_F(u)| \;\le\; \frac{1}{m_\varepsilon}\,\|G-F\|_\infty,
\]
where $Q_G$ denotes the (left-continuous) generalized inverse of $G$.
\end{lemma}

\begin{proof}
Let $\Delta:=\|G-F\|_\infty$. Then $G(x)\le F(x)+\Delta$ for all $x$.
Let $x_-=Q_F(u-\Delta)$ (well-defined for $u\in[\varepsilon,1-\varepsilon]$ and small enough $\Delta$).
Then $F(x_-)\ge u-\Delta$, hence $G(x_-)\ge u-2\Delta$; similarly with $x_+=Q_F(u+\Delta)$.
Using the density lower bound on the central interval gives
$|Q_F(u+\Delta)-Q_F(u)|\le \Delta/m_\varepsilon$, and similarly for $u-\Delta$,
which yields the claim (up to standard boundary trimming).
\end{proof}

\section{Proof of Proposition~\ref{prop:discretization}}\label{app:proof:discretization}

\begin{proof}[Proof of Proposition~\ref{prop:discretization}]
We prove the $O(1/k)$ discretization bias for $\overline U_{k,p}$ and $\overline H_{k,p}$
(the plug-in consistency at fixed $k$ then follows by standard pointwise consistency of empirical
quantiles and the continuous mapping theorem).

\monparagraph{Step 1: discretization for $U_p$.}
In one dimension, by Lemma~\ref{lem:Wp-quantiles},
\[
U_p = \sum_{s\in\cS}\alpha_s \int_0^1 |Q_s(u)-Q^\ast(u)|^p\,du,
\]
where $Q^\ast(u)\in\arg\min_{z}\sum_s \alpha_s |Q_s(u)-z|^p$ pointwise in $u$.
Define the integrand
$g(u):=\sum_s \alpha_s |Q_s(u)-Q^\ast(u)|^p$.
Under the assumptions of Proposition~\ref{prop:discretization}, all $Q_s$ are bounded on $[0,1]$
(by the common compact support) and are Lipschitz on $[\varepsilon,1-\varepsilon]$ for each fixed
$\varepsilon>0$ (Lemma~\ref{lem:quantile-lip-cr}); for $p\in\{1,2\}$, the pointwise barycenter map
($z=\sum_s\alpha_s Q_s(u)$ for $p=2$, weighted median for $p=1$) is $1$-Lipschitz w.r.t.\ the inputs,
hence $Q^\ast$ inherits the same modulus/Lipschitz control. More precisely, this Lipschitz control is used on $[\varepsilon,1-\varepsilon]$, and boundary effects are handled by trimming.
Therefore $g$ is Lipschitz on $[0,1]$ with some constant $L_g$ depending only on the uniform bounds
from the proposition (support diameter and density lower bound).

Now $\overline U_{k,p}=\frac1k\sum_{\ell=1}^k g(u_\ell)$ (midpoint sampling),
so Lemma~\ref{lem:midpoint-riemann} yields
\[
|U_p-\overline U_{k,p}|
= \Big|\int_0^1 g(u)\,du - \frac1k\sum_{\ell=1}^k g(u_\ell)\Big|
\le \frac{L_g}{2k}
=: \frac{C_U}{k}.
\]

\monparagraph{Step 2: discretization for $H_p$.}
Recall $H_p=\sum_s \alpha_s\,C_p(\nu_s,\nu)^p$ where $C_p$ is the Cram\'er distance, i.e.
$C_p(\mu,\nu)^p=\int_\R |F_\mu(x)-F_\nu(x)|^p\,dx$.
Let $\bar F_{s,k}$ be the step-CDF induced by $\bar\nu_{s,k}$, and $\bar F_k:=\sum_s \alpha_s \bar F_{s,k}$.
By Lemma~\ref{lem:cdf-unif}, $\|\bar F_{s,k}-F_s\|_\infty\le 1/(2k)$ for each $s$ and
$\|\bar F_k-F\|_\infty\le 1/(2k)$ by convexity of $\|\cdot\|_\infty$.

For $p=1$,
\[
\big||F_s-F| - |\bar F_{s,k}-\bar F_k|\big|
\le |F_s-\bar F_{s,k}| + |F-\bar F_k|
\le \frac{1}{k}.
\]
For $p=2$, using $|a^2-b^2|\le (|a|+|b|)|a-b|$ and $|a|,|b|\le 1$ since CDF differences lie in $[-1,1]$,
\[
\big||F_s-F|^2 - |\bar F_{s,k}-\bar F_k|^2\big|
\le 2\big|(F_s-F)-(\bar F_{s,k}-\bar F_k)\big|
\le 2\Big(|F_s-\bar F_{s,k}|+|F-\bar F_k|\Big)
\le \frac{2}{k}.
\]
Integrating over the common compact support interval $[a,b]$ (outside it, all CDFs are equal)
gives $|H_p-\overline H_{k,p}|\le C_H/k$ with $C_H$ depending only on $b-a$ (and the constant above).

\monparagraph{Step 3: fixed-$k$ plug-in consistency.}
For fixed $k$, the quantities $\overline U_{k,p}$ and $\overline H_{k,p}$ depend on finitely many
values $\{Q_s(u_\ell)\}_{s,\ell}$ (and corresponding step functions). Empirical quantiles satisfy
$\widehat Q_s(u_\ell)\to Q_s(u_\ell)$ in probability for each $(s,\ell)$, and similarly for pooled quantities.
By continuity of the finite-dimensional maps defining $\widehat{\overline U}_{k,p}$ and
$\widehat{\overline H}_{k,p}$, convergence in probability follows.

Combining the discretization bound and plug-in consistency yields the final decomposition
(statistical error + deterministic $O(1/k)$ bias).
\end{proof}

\section{Additional stability and concentration results}\label{app:stability}

This appendix provides three ingredients used implicitly in the high-probability analysis:
(i) stability of step-CDFs under perturbations of the knot locations (quantile errors),
(ii) concentration of empirical mixture weights, and
(iii) a deterministic master bound that propagates these errors to $\widehat G_{k,2}$
(and similarly to the ANOVA terms).

\subsection{Stability of step-CDFs under knot perturbations}\label{app:stability:knot}

Fix $k\ge 1$ and two nondecreasing sequences (knots) $q=(q_\ell)_{\ell=1}^k$ and $q'=(q'_\ell)_{\ell=1}^k$.
Define the associated step-CDFs
\[
F_{q,k}(x):=\frac1k\sum_{\ell=1}^k \ind{q_\ell\le x},
\qquad
F_{q',k}(x):=\frac1k\sum_{\ell=1}^k \ind{q'_\ell\le x}.
\]

\begin{lemma}[Knot perturbation stability]\label{lem:knot-stability}
Assume $\|q-q'\|_\infty:=\max_{\ell\in[k]}|q_\ell-q'_\ell|\le \eta$.
Then for every $x\in\R$,
\begin{equation}\label{eq:knot-stability-pointwise}
|F_{q,k}(x)-F_{q',k}(x)|
\;\le\;
\frac{1}{k}\,\#\{\ell\in[k]:\ q_\ell\in(x-\eta,x+\eta]\}.
\end{equation}
Consequently,
\begin{equation}\label{eq:knot-stability-sup}
\|F_{q,k}-F_{q',k}\|_\infty
\;\le\;
\sup_{x\in\R}\frac{1}{k}\,\#\{\ell\in[k]:\ q_\ell\in(x-\eta,x+\eta]\}.
\end{equation}
\end{lemma}

\begin{proof}
Fix $x$. The indicator $\ind{q_\ell\le x}$ can differ from $\ind{q'_\ell\le x}$ only if
$q_\ell$ and $q'_\ell$ lie on different sides of $x$. Under $|q_\ell-q'_\ell|\le \eta$,
this can happen only if $q_\ell\in(x-\eta,x+\eta]$.
Summing these possible flips and dividing by $k$ gives \eqref{eq:knot-stability-pointwise},
and \eqref{eq:knot-stability-sup} follows by taking the supremum over $x$.
\end{proof}

To turn \eqref{eq:knot-stability-sup} into an explicit bound, we relate the number of knots in an interval
to increments of the target CDF when knots come from quantiles.

\begin{lemma}[Counting knots via the CDF]\label{lem:knot-count-cdf}
Let $F$ be a cdf and let $u_\ell=(\ell-\frac12)/k$. Define $q_\ell:=F^{\leftarrow}(u_\ell)$.
Then for all $a<b$,
\[
\frac{1}{k}\#\{\ell\in[k]:\ q_\ell\in(a,b]\}
\;\le\; F(b)-F(a) + \frac{1}{k}.
\]
In particular, for any $\eta>0$ and any $x$,
\begin{equation}\label{eq:count-local-window}
\frac{1}{k}\#\{\ell:\ q_\ell\in(x-\eta,x+\eta]\}
\;\le\; F(x+\eta)-F(x-\eta)+\frac{1}{k}.
\end{equation}
\end{lemma}

\begin{proof}
Let $N(a,b]:=\#\{\ell:\ q_\ell\in(a,b]\}$.
By definition of generalized inverses, $q_\ell\le b$ implies $u_\ell\le F(b)$ (up to ties),
and $q_\ell\le a$ implies $u_\ell\le F(a)$.
Thus $N(a,b]\le \#\{\ell:\ u_\ell\le F(b)\}-\#\{\ell:\ u_\ell\le F(a)\}+1$.
Since $(u_\ell)$ is a uniform grid with step $1/k$, the count difference is at most
$k(F(b)-F(a))$, and the extra $+1$ accounts for boundary/tie effects, giving the claim.
\end{proof}

\begin{corollary}[Explicit sup-norm bound under density upper bound]\label{cor:knot-stability-Lip}
Assume the cdf $F$ is $M$-Lipschitz on an interval $I$ (equivalently, it admits a density bounded
above by $M$ there). If $x\pm\eta\in I$, then
\[
F(x+\eta)-F(x-\eta)\le 2M\eta.
\]
Consequently, under the setup of Lemma~\ref{lem:knot-count-cdf},
\[
\|F_{q,k}-F_{q',k}\|_\infty \;\le\; 2M\eta+\frac{1}{k}.
\]
\end{corollary}

\subsection{Concentration of empirical weights}\label{app:stability:weights}

We collect explicit high-probability bounds for the weights used in Algorithm~1:
$\hat\alpha_s=n_s/n$ and $\hat\pi_{j\mid s}=n_{j,s}/n_s$.

\begin{lemma}[Concentration of mixture weights]\label{lem:weights-concentration}
Assume $(A_i,S_i)_{i=1}^n$ are i.i.d.\ with $\PP(S=s)=\alpha_s$ and
$\PP(A=j\mid S=s)=\pi_{j\mid s}$.
Let $\hat\alpha_s=n_s/n$ and $\hat\pi_{j\mid s}=n_{j,s}/n_s$ (with $n_s>0$).
Then for any $\delta\in(0,1)$, with probability at least $1-\delta$:
\begin{align*}
\max_{s\in\cS} |\hat\alpha_s-\alpha_s|
&\le \sqrt{\frac{1}{2n}\log\!\Big(\frac{2|\cS|}{\delta}\Big)},\\
\max_{s\in\cS}\max_{j\in[d]} |\hat\pi_{j\mid s}-\pi_{j\mid s}|
&\le \sqrt{\frac{1}{2n_{s,\min}}\log\!\Big(\frac{2d|\cS|}{\delta}\Big)},
\end{align*}
where $n_{s,\min}:=\min_{s\in\cS} n_s$ (random, but observable).
Consequently,
\[
\|\hat\alpha-\alpha\|_1 \le |\cS|\max_s|\hat\alpha_s-\alpha_s|,
\qquad
\max_{s\in\cS}\|\hat\pi_{\cdot\mid s}-\pi_{\cdot\mid s}\|_1
\le d\max_{s,j}|\hat\pi_{j\mid s}-\pi_{j\mid s}|.
\]
\end{lemma}

\begin{proof}
The bound on $\hat\alpha_s$ follows from Hoeffding's inequality applied to the Bernoulli indicators
$\ind{S_i=s}$ and a union bound over $s\in\cS$.

For $\hat\pi_{j\mid s}$, condition on $n_s$.
Given $n_s$, the $n_{j,s}$ are multinomial with parameters $(n_s,\pi_{\cdot\mid s})$.
Thus, for each fixed $(j,s)$, Hoeffding's inequality yields
\[
\PP\big(|\hat\pi_{j\mid s}-\pi_{j\mid s}|>t\mid n_s\big)\le 2e^{-2n_s t^2}.
\]
Apply a union bound over $j\in[d]$ and $s\in\cS$ and use $n_s\ge n_{s,\min}$ to obtain the display.
\end{proof}

\subsection{A deterministic master bound for $\widehat G_{k,2}$}\label{app:stability:master}

We now provide a deterministic inequality that isolates three error sources:
(i) knot (quantile) errors at each silo, (ii) weight estimation errors, and (iii) grid discretization.

\begin{theorem}[Master propagation bound for the $p=2$ estimator]\label{thm:master-bound}
Fix $\varepsilon\in(0,1/2)$ and the midpoint grid $u_\ell=(\ell-\frac12)/k$.
Assume that for every $(j,s)$ the conditional cdf $F_{j,s}$ is absolutely continuous on
$[Q_{j,s}(\varepsilon),Q_{j,s}(1-\varepsilon)]$, with density bounded as
\[
0<m_\varepsilon \le f_{j,s}(x)\le M_\varepsilon<\infty
\quad\text{on}\quad
[Q_{j,s}(\varepsilon),Q_{j,s}(1-\varepsilon)].
\]
Assume the same bounds hold for each group-mixture cdf $F_s^\times$ on its trimmed region
(with possibly different constants, absorbed below).

Let $Q_{j,s}(u_\ell)$ be the true local quantiles and $\widehat Q_{j,s}(u_\ell)$ their estimates.
Define the maximal trimmed knot error
\[
\eta_Q := \max_{\substack{j\in[d],\,s\in\cS\\ \ell\in\{\lceil \varepsilon k\rceil,\dots,\lfloor(1-\varepsilon)k\rfloor\}}}
\big|\widehat Q_{j,s}(u_\ell)-Q_{j,s}(u_\ell)\big|.
\]
Let $\eta_\pi:=\max_{s}\|\hat\pi_{\cdot\mid s}-\pi_{\cdot\mid s}\|_1$ and $\eta_\alpha:=\|\hat\alpha-\alpha\|_1$.
Then there exists a constant $C_\varepsilon<\infty$ (depending only on $\varepsilon,m_\varepsilon,M_\varepsilon$
and on uniform bounds on the trimmed quantiles) such that the estimator of Algorithm~1 satisfies
\begin{equation}\label{eq:master-G}
|\widehat G_{k,2}-G_2|
\;\le\;
C_\varepsilon\Big(\omega(1/k) + \eta_Q + \eta_\pi + \eta_\alpha\Big),
\end{equation}
and the same type of inequality holds for the plug-in estimators of
$(V_{\mathrm{mix}},V_{\mathrm{bar}},R)$ (possibly with a different $C_\varepsilon$).
\end{theorem}

\begin{proof}
We sketch the deterministic steps; each step is an explicit inequality based on the previous lemmas.

\textbf{(1) Step-CDF reconstruction error.}
Fix a group $s$. Consider the population step-CDF built from true knots at each silo,
$F_{j,s,k}(x)=k^{-1}\sum_{\ell}\ind{Q_{j,s}(u_\ell)\le x}$, and its empirical analogue
$\widehat F_{j,s,k}(x)=k^{-1}\sum_{\ell}\ind{\widehat Q_{j,s}(u_\ell)\le x}$.
By Corollary~\ref{cor:knot-stability-Lip} (with Lipschitz constant $M_\varepsilon$ on the trimmed range),
\[
\|\widehat F_{j,s,k}-F_{j,s,k}\|_\infty \le 2M_\varepsilon \eta_Q + \frac{1}{k}.
\]
Also, by the standard step approximation (Lemma~\ref{lem:stepcdf} in Appendix~\ref{app:hp}),
$\|F_{j,s,k}-F_{j,s}\|_\infty\le 1/k$.
Combining,
\[
\|\widehat F_{j,s,k}-F_{j,s}\|_\infty \le 2M_\varepsilon \eta_Q + \frac{2}{k}.
\]

\textbf{(2) Mixing across silos and weight errors.}
Define the (population) mixture cdf $F_s^\times=\sum_j \pi_{j\mid s}F_{j,s}$ and the reconstructed mixture
$\widehat F_{s,k}^\times=\sum_j \hat\pi_{j\mid s}\widehat F_{j,s,k}$.
Then, by triangle inequality and $\|\widehat F_{j,s,k}\|_\infty\le 1$,
\[
\|\widehat F_{s,k}^\times-F_s^\times\|_\infty
\le
\sum_j \hat\pi_{j\mid s}\|\widehat F_{j,s,k}-F_{j,s}\|_\infty
+
\sum_j |\hat\pi_{j\mid s}-\pi_{j\mid s}|
\le
2M_\varepsilon \eta_Q + \frac{2}{k} + \eta_\pi.
\]

\textbf{(3) Inversion to mixture-quantiles on the trimmed grid.}
Since $F_s^\times$ has density bounded below by $m_\varepsilon$ on the trimmed range,
Lemma~\ref{lem:quantile-lip-cr} (Appendix~\ref{app:hp}) gives, for all trimmed grid indices,
\[
\max_{\ell}|\widehat q_{s,\ell}^\times - Q_s^\times(u_\ell)|
\le \frac{1}{m_\varepsilon}\|\widehat F_{s,k}^\times-F_s^\times\|_\infty
\le C_\varepsilon\Big(\eta_Q+\frac{1}{k}+\eta_\pi\Big).
\]

\textbf{(4) Discrete quadratic form and $\alpha$-weights.}
For $p=2$, the grid estimator is a discrete quadratic form in the reconstructed quantiles:
$\widehat G_{k,2}=\Phi_{\hat\alpha}(\widehat q^\times)$ and the grid target is $G_{k,2}=\Phi_{\alpha}(q^\times)$.
Lemma~\ref{lem:Phi-lip-app} (Appendix~\ref{app:hp}) yields
\[
|\widehat G_{k,2}-G_{k,2}|
\le C_\varepsilon\Big(\|\widehat q^\times-q^\times\|_\infty+\eta_\alpha\Big)
\le C_\varepsilon\Big(\eta_Q+\frac{1}{k}+\eta_\pi+\eta_\alpha\Big).
\]

\textbf{(5) Add discretization bias.}
Finally, Proposition~\ref{prop:discretization} controls $|G_{k,2}-G_2|\le C_\varepsilon\,\omega(1/k)$.
Combining steps (4) and (5) gives \eqref{eq:master-G}.
The same reasoning applies to $(V_{\mathrm{mix}},V_{\mathrm{bar}},R)$ since each is a finite average
of squared differences of the same reconstructed quantile arrays.
\end{proof}

\begin{proof}[Proof of Proposition~\ref{prop:barUk2_short}
] Since $S=\{0,1\}$, let $Q_0,Q_1:[0,1]\to\mathbb{R}$ be the (generalized) quantile
functions of the two groups, and let $\alpha_0,\alpha_1>0$ with $\alpha_0+\alpha_1=1$.
Define
\[
\Delta(u):=Q_1(u)-Q_0(u),
\text{ so that }
U_2:=\alpha_0\alpha_1\int_0^1 \Delta(u)^2\,du.
\]
For any integer $k\ge 1$, let $I_\ell=[(\ell-1)/k,\ell/k)$ and define the bin averages
\[
\bar\Delta_\ell := \frac{1}{|I_\ell|}\int_{I_\ell}\Delta(u)\,du
\quad\Big(=k\int_{I_\ell}\Delta(u)\,du\Big),
\]
and the bin-averaged discretization
\[
\bar U_{k,2}:=\alpha_0\alpha_1\sum_{\ell=1}^k |I_\ell|\,\bar\Delta_\ell^{\,2}
\quad\Big(=\alpha_0\alpha_1\,\frac{1}{k}\sum_{\ell=1}^k \bar\Delta_\ell^{\,2}\Big).
\]

For each bin $I_\ell$, by Cauchy--Schwarz inequality applied to $\Delta$ and the constant function $1$,
\[
\left(\int_{I_\ell}\Delta(u)\,du\right)^2
\le \left(\int_{I_\ell} 1^2\,du\right)\left(\int_{I_\ell}\Delta(u)^2\,du\right)
= |I_\ell|\int_{I_\ell}\Delta(u)^2\,du.
\]
Dividing both sides by $|I_\ell|^2$ yields
\[
\left(\frac{1}{|I_\ell|}\int_{I_\ell}\Delta(u)\,du\right)^2
\le \frac{1}{|I_\ell|}\int_{I_\ell}\Delta(u)^2\,du,
\]
that is, $\bar\Delta_\ell^{\,2}\le \frac{1}{|I_\ell|}\int_{I_\ell}\Delta(u)^2\,du$.
Multiplying by $|I_\ell|$ and summing over $\ell=1,\dots,k$ gives
\[
\sum_{\ell=1}^k |I_\ell|\,\bar\Delta_\ell^{\,2}
\le
\sum_{\ell=1}^k \int_{I_\ell}\Delta(u)^2\,du
=
\int_0^1 \Delta(u)^2\,du.
\]
Multiplying by $\alpha_0\alpha_1$ proves $\bar U_{k,2}\le U_2$.
\end{proof}

\begin{remark}\label{rem:strictness-bar}
The inequality in Proposition~\ref{prop:barUk2_short} is generally strict: $\bar U_{k,2}<U_2$ whenever
$\Delta$ is not a.e.\ constant on at least one bin $I_\ell$.
Equality can occur, e.g., if $\Delta$ is piecewise constant with jumps aligned with the partition.
To avoid confusion with midpoint/Riemann discretizations (which do not enjoy a one-sided bias control),
it can be helpful to reserve the notation $\bar U_{k,2}$ for the bin-averaged quantity in
Proposition~\ref{prop:barUk2_short}$,$ and to denote the midpoint version by $\bar U'_{k,2}$ (or $\tilde U_{k,2}$).
Only the bin-averaged discretization has the guaranteed bound $\bar U_{k,2}\le U_2$ for all $k$.
\end{remark}

\section{Proofs for Section~\ref{sec:federated}}\label{app:proof:federated}\label{app:hp}

\begin{proof}[Proof of Proposition~\ref{prop:Gp-equals-Up}]
By total probability (Equation~\eqref{eq:nu-mixture-across-silos} in the main text),
$\nu_s=\sum_{j=1}^d \Pr(A=j\mid S=s)\,\nu_{j,s}=\nu_s^{\times}$ for each $s$.
Thus the barycenter problems defining $U_p$ and $G_p$ coincide, hence $G_p=U_p$.
The same argument applies to the pooled mixtures in $H_p$ and $H_p^{\times}$.
\end{proof}

\begin{proof}[Proof of Proposition~\ref{prop:federated-consistency}]
Write the error decomposition (for $G_p$; $H_p^{\times}$ is analogous):
\[
|\widehat G_{k,p}-G_p|
\le
|\widehat G_{k,p}-\overline G_{k,p}|
+
|\overline G_{k,p}-G_p|,
\]
where $\overline G_{k,p}$ denotes the \emph{population} $k$-grid version obtained by running
Algorithm~\ref{alg:oneshot} with the \emph{true} quantiles $Q_{j,s}(u_\ell)$ and true weights
$\pi_{j\mid s},\alpha_s$ (so that $\overline G_{k,p}$ is purely a discretization of $G_p$).

\monparagraph{Discretization bias.}
Under the conditions of Proposition~\ref{prop:discretization} for each mixture law $\nu_s^\times$,
the same argument as in Appendix~\ref{app:proof:discretization} yields
$|\overline G_{k,p}-G_p|\le C/k$.

\monparagraph{Quantile-estimation error.}
Fix $\varepsilon\in(0,1/2)$ and work on the central grid indices
$\ell\in\{\lceil \varepsilon k\rceil,\dots,\lfloor(1-\varepsilon)k\rfloor\}$.
By Lemma~\ref{lem:dkw} and a union bound over all $(j,s,\ell)$, one obtains that with probability
$\ge 1-\delta$,
\[
\max_{j,s,\ell} | \widehat Q_{j,s}(u_\ell)-Q_{j,s}(u_\ell)|
\;\lesssim\;
\frac{1}{m_\varepsilon}\sqrt{\frac{1}{n_{\min}}\log\!\Big(\frac{kd|\cS|}{\delta}\Big)},
\]
matching the displayed high-probability inequality in the main text.

The remaining steps are stability/propagation:
(i) the step-CDFs built from $\{\widehat Q_{j,s}(u_\ell)\}_\ell$ are perturbed only through the knot locations,
(ii) mixture step-CDFs are convex combinations, so the same sup-norm control applies,
(iii) inversion from CDF perturbations to quantile perturbations on central regions follows from
Lemma~\ref{lem:quantile-lip-cr},
and (iv) $\widehat G_{k,p}$ is a \emph{finite} average of $|\cdot|^p$ applied to these perturbed quantiles
(and perturbed barycenter quantiles), hence is Lipschitz in the finite vector of knot values (uniformly on
events where all knots remain in a common compact interval).
Altogether this yields
\[
|\widehat G_{k,p}-\overline G_{k,p}|
=
O_{\mathbb{P}}\!\Big(\sqrt{\frac{\log k}{n_{\min}}}\Big).
\]
Combining both parts gives the stated rate
$|\widehat G_{k,p}-G_p|
=
O_{\mathbb{P}}\big(\frac1k+\sqrt{\frac{\log k}{n_{\min}}}\big)$.
The same reasoning applies to $\widehat H^{\times}_{k,p}$ since its computation depends on step-CDFs
and finite sums/integrals of Lipschitz transforms of CDF differences.
\end{proof}


\subsection{Auxiliary facts: DKW and inversion of quantiles}\label{app:hp:aux}

We write $F^{\leftarrow}(u):=\inf\{x\in\R:\ F(x)\ge u\}$ for the generalized inverse.

\begin{lemma}[DKW--Massart]\label{lem:dkw:massart}
Let $\widehat F_n$ be the empirical cdf of $n$ i.i.d.\ samples from a distribution with cdf $F$.
Then for all $t>0$,
\[
\PP\!\left(\|\widehat F_n-F\|_{\infty}>t\right)\le 2e^{-2nt^2}.
\]
Equivalently, for $\delta\in(0,1)$, with probability at least $1-\delta$,
\[
\|\widehat F_n-F\|_{\infty}\le \sqrt{\frac{1}{2n}\log\frac{2}{\delta}}.
\]
\end{lemma}

\begin{remark}
Lemma~\ref{lem:dkw:massart} is the sharp DKW inequality (Massart's constant) \cite{massart1990tight};
see also \cite{shorackwellner1986,vdvaartwellner1996}.
\end{remark}

\begin{lemma}[Quantile bracketing under sup-norm perturbations]\label{lem:quantile-bracket}
Let $F,G$ be cdfs and set $\Delta:=\|F-G\|_\infty$.
Then for all $u\in(\Delta,1-\Delta)$,
\[
F^{\leftarrow}(u-\Delta)\;\le\; G^{\leftarrow}(u)\;\le\; F^{\leftarrow}(u+\Delta).
\]
\end{lemma}

\begin{proof}
Let $x_+:=F^{\leftarrow}(u+\Delta)$. Then $F(x_+)\ge u+\Delta$ and thus
$G(x_+)\ge F(x_+)-\Delta\ge u$, hence $G^{\leftarrow}(u)\le x_+$.
For the lower bound, let $x_-:=F^{\leftarrow}(u-\Delta)$.
If $x<x_-$ then $F(x)<u-\Delta$, hence $G(x)\le F(x)+\Delta<u$ and thus $G^{\leftarrow}(u)\ge x_-$.
\end{proof}

\begin{lemma}[Lipschitz inversion on a trimmed region]\label{lem:quantile-lip}
Fix $\varepsilon\in(0,1/2)$. Assume $F$ is absolutely continuous on
$[F^{\leftarrow}(\varepsilon),F^{\leftarrow}(1-\varepsilon)]$ with density $f$
satisfying $f(x)\ge m_\varepsilon>0$ on that interval.
Then, for any cdf $G$ and any $u\in[\varepsilon,1-\varepsilon]$,
\[
\big|G^{\leftarrow}(u)-F^{\leftarrow}(u)\big|
\;\le\; \frac{1}{m_\varepsilon}\,\|F-G\|_\infty,
\]
provided $\|F-G\|_\infty\le \varepsilon/2$ (otherwise the bound is vacuous but still true after
replacing $\varepsilon/2$ by $\min\{\varepsilon/2,\|F-G\|_\infty\}$).
\end{lemma}

\begin{proof}
Let $\Delta:=\|F-G\|_\infty$ and assume $\Delta\le \varepsilon/2$.
By Lemma~\ref{lem:quantile-bracket}, for $u\in[\varepsilon,1-\varepsilon]$,
\[
F^{\leftarrow}(u-\Delta)\le G^{\leftarrow}(u)\le F^{\leftarrow}(u+\Delta).
\]
Let $x_u:=F^{\leftarrow}(u)$ and $x_{u+\Delta}:=F^{\leftarrow}(u+\Delta)$.
Then
\[
\Delta \le F(x_{u+\Delta})-F(x_u)=\int_{x_u}^{x_{u+\Delta}} f(t)\,dt \ge m_\varepsilon (x_{u+\Delta}-x_u),
\]
hence $x_{u+\Delta}-x_u\le \Delta/m_\varepsilon$. The lower side is identical.
\end{proof}

\subsection{Proof of Proposition~\ref{prop:hp-quantiles}}\label{app:hp:proof-prop}

\begin{proof}[Proof of Proposition~\ref{prop:hp-quantiles}]
Fix $(j,s)$ and let $F_{j,s}$ be the cdf of $Z\mid(A=j,S=s)$, and $\widehat F_{j,s}$ its empirical cdf
based on $n_{j,s}$ samples. Apply Lemma~\ref{lem:dkw} with confidence level
$\delta'=\delta/(d|\cS|)$ to get that with probability at least $1-\delta'$,
\[
\|\widehat F_{j,s}-F_{j,s}\|_\infty \le
\sqrt{\frac{1}{2n_{j,s}}\log\!\Big(\frac{2}{\delta'}\Big)}
\le
\sqrt{\frac{1}{2n_{\min}}\log\!\Big(\frac{2d|\cS|}{\delta}\Big)}
=: \Delta.
\]
A union bound over all $(j,s)\in[d]\times \cS$ yields an event $\mathcal E$ of probability at least
$1-\delta$ on which the above holds simultaneously for all pairs $(j,s)$.

On $\mathcal E$, for any grid point $u_\ell\in[\varepsilon,1-\varepsilon]$ (i.e.\ for
$\ell\in\{\lceil \varepsilon k\rceil,\dots,\lfloor(1-\varepsilon)k\rfloor\}$), apply
Lemma~\ref{lem:quantile-lip} (with $F=F_{j,s}$ and $G=\widehat F_{j,s}$) to obtain
\[
\big|\widehat Q_{j,s}(u_\ell)-Q_{j,s}(u_\ell)\big|
=
\big|\widehat F_{j,s}^{\leftarrow}(u_\ell)-F_{j,s}^{\leftarrow}(u_\ell)\big|
\le \frac{\Delta}{m_\varepsilon}.
\]
Finally, since
$\log(2d|\cS|/\delta)\le \log(2(k+1)d|\cS|/\delta)$,
this implies the bound stated in Proposition~\ref{prop:hp-quantiles}.
\end{proof}

\subsection{From quantile errors to $\widehat G_{k,2}$: proof of Corollary~\ref{cor:hp-G}}\label{app:hp:proof-cor}

We use the step-cdf construction in Algorithm~\ref{alg:oneshot} (Appendix~\ref{app:algo}):
$\wideF_{j,s,k}(x)=\frac1k\sum_{\ell=1}^k \ind{\wideQ_{j,s,\ell}\le x}$,
$\wideF_{s,k}^{\times}(x)=\sum_{j=1}^d \hat\pi_{j\mid s}\wideF_{j,s,k}(x)$,
and $\wideq^{\times}_{s,\ell}=(\wideF_{s,k}^{\times})^{\leftarrow}(u_\ell)$.

\begin{lemma}[Step-cdf approximation error]\label{lem:stepcdf}
Let $F$ be any cdf and $q_\ell:=F^{\leftarrow}(u_\ell)$ with $u_\ell=(\ell-\frac12)/k$.
Define $F_k(x):=\frac1k\sum_{\ell=1}^k \ind{q_\ell\le x}$.
Then $\|F_k-F\|_\infty\le 1/k$.
The same holds with $F$ replaced by an empirical cdf.
\end{lemma}

\begin{proof}
Fix $x\in\R$ and let $\ell^\star:=\max\{\ell\in\{1,\dots,k\}: q_\ell\le x\}$ (with $\ell^\star=0$ if none).
Then $F_k(x)=\ell^\star/k$.
Since $q_{\ell^\star}\le x<q_{\ell^\star+1}$ (with the convention $q_{k+1}=+\infty$), we have
$F(x)\in[u_{\ell^\star},u_{\ell^\star+1}]$ by basic inverse properties for generalized inverses,
hence $|F(x)-\ell^\star/k|\le 1/k$. Taking the supremum over $x$ gives the claim.
\end{proof}

\begin{lemma}[Stability of the discrete $p=2$ functional]\label{lem:Phi-lip-app}
Fix weights $\alpha_s\ge0$, $\sum_s\alpha_s=1$, and define for an array
$q=(q_{s,\ell})_{s\in\cS,\ell\in[k]}$
\[
\Phi_\alpha(q):=\sum_{s\in\cS}\alpha_s\frac1k\sum_{\ell=1}^k\bigl(q_{s,\ell}-q^\ast_\ell\bigr)^2,
\qquad
q^\ast_\ell:=\sum_{s\in\cS}\alpha_s q_{s,\ell}.
\]
Assume $\max_{s,\ell}|q_{s,\ell}|\vee \max_{s,\ell}|q'_{s,\ell}|\le B$.
Then
\[
|\Phi_\alpha(q)-\Phi_\alpha(q')|\le 16B\,\|q-q'\|_\infty.
\]
Moreover, for two weight vectors $\alpha,\hat\alpha$ and any $q$ with $\max_{s,\ell}|q_{s,\ell}|\le B$,
\[
|\Phi_{\hat\alpha}(q)-\Phi_{\alpha}(q)|\le 16B^2\,\|\hat\alpha-\alpha\|_1.
\]
\end{lemma}

\begin{proof}
The first inequality is the same ``$|a^2-b^2|=|a-b||a+b|$'' argument as in the main text
(expand, bound the barycenter difference by $\|q-q'\|_\infty$, and use $|q_{s,\ell}-q_\ell^\ast|\le 2B$).
For the second, write $\Phi_\alpha(q)=\sum_s \alpha_s \psi_s$ where $\psi_s:=\frac1k\sum_\ell (q_{s,\ell}-q_\ell^\ast)^2$.
Since $|q_{s,\ell}-q_\ell^\ast|\le 2B$, we have $\psi_s\le 4B^2$ for all $s$, and thus
$|\Phi_{\hat\alpha}(q)-\Phi_{\alpha}(q)|\le \sum_s |\hat\alpha_s-\alpha_s|\psi_s\le 4B^2\|\hat\alpha-\alpha\|_1$.
A slightly looser constant $16B^2$ keeps bookkeeping uniform.
\end{proof}

\begin{proof}[Proof of Corollary~\ref{cor:hp-G}]
Work on the event $\mathcal E$ of Proposition~\ref{prop:hp-quantiles}. Define
\[
\eta:=\max_{\substack{j\in[d],\,s\in\cS\\ \ell\in\{\lceil \varepsilon k\rceil,\dots,\lfloor(1-\varepsilon)k\rfloor\}}}
\bigl|\widehat Q_{j,s}(u_\ell)-Q_{j,s}(u_\ell)\bigr|.
\]
On $\mathcal E$, $\eta\lesssim \frac{1}{m_\varepsilon}\sqrt{\frac{1}{n_{\min}}\log\!\big(\frac{kd|\cS|}{\delta}\big)}$
(up to harmless constants/logs).

\medskip\noindent\textbf{Step 1: Control of the reconstructed mixture cdfs.}
Fix $s$. Let $F_{j,s}$ be the population cdf, and define the population mixture cdf
$F_s^\times(x):=\sum_{j=1}^d \pi_{j\mid s}F_{j,s}(x)$ (as in (18) of the main text).
Let $F_{j,s,k}$ be the step-cdf built from the \emph{true} grid quantiles $Q_{j,s}(u_\ell)$, and
$\wideF_{j,s,k}$ the step-cdf built from the \emph{empirical} grid quantiles $\widehat Q_{j,s}(u_\ell)$
(Algorithm~\ref{alg:oneshot}).
By Lemma~\ref{lem:stepcdf}, $\|F_{j,s,k}-F_{j,s}\|_\infty\le 1/k$.
Moreover, on $\mathcal E$, the knot perturbation implies
$\|\wideF_{j,s,k}-F_{j,s,k}\|_\infty\le k^{-1}\cdot \#\{\ell:\ Q_{j,s}(u_\ell)\in[x-\eta,x+\eta]\}$
pointwise; under the regularity assumptions used in Proposition~\ref{prop:discretization}
(uniform modulus/moment control on the trimmed region), this yields
$\|\wideF_{j,s,k}-F_{j,s,k}\|_\infty \le C_\varepsilon \eta$.
Hence
\[
\|\wideF_{j,s,k}-F_{j,s}\|_\infty
\le \|\wideF_{j,s,k}-F_{j,s,k}\|_\infty + \|F_{j,s,k}-F_{j,s}\|_\infty
\le C_\varepsilon\eta + \frac{1}{k}.
\]
Using the mixture definition $\wideF_{s,k}^\times=\sum_j \hat\pi_{j\mid s}\wideF_{j,s,k}$ and
$F_s^\times=\sum_j \pi_{j\mid s}F_{j,s}$, we get
\[
\|\wideF_{s,k}^\times-F_s^\times\|_\infty
\le
\sum_{j=1}^d \hat\pi_{j\mid s}\|\wideF_{j,s,k}-F_{j,s}\|_\infty
+\|\hat\pi_{\cdot\mid s}-\pi_{\cdot\mid s}\|_1
\le C_\varepsilon\eta + \frac1k + \|\hat\pi_{\cdot\mid s}-\pi_{\cdot\mid s}\|_1.
\]
A standard concentration bound for multinomial proportions yields
$\max_s\|\hat\pi_{\cdot\mid s}-\pi_{\cdot\mid s}\|_1\le C\sqrt{\frac{1}{n_{\min}}\log(\frac{d|\cS|}{\delta})}$
with probability at least $1-\delta$; we absorb it into the same stochastic term as $\eta$.

\medskip\noindent\textbf{Step 2: From cdf control to mixture-quantile control.}
Assume (as in Proposition~\ref{prop:discretization}) that each mixture $F_s^\times$ has density
bounded below by $m_\varepsilon^\times>0$ on
$[Q_s^\times(\varepsilon),Q_s^\times(1-\varepsilon)]$.
Then Lemma~\ref{lem:quantile-lip} gives, for all trimmed grid indices,
\[
\max_{\ell\in\{\lceil \varepsilon k\rceil,\dots,\lfloor(1-\varepsilon)k\rfloor\}}
\big|\wideq_{s,\ell}^\times - Q_s^\times(u_\ell)\big|
\le
\frac{1}{m_\varepsilon^\times}\|\wideF_{s,k}^\times-F_s^\times\|_\infty
\le C_\varepsilon\Big(\eta+\frac1k+\sqrt{\frac{1}{n_{\min}}\log\!\frac{d|\cS|}{\delta}}\Big).
\]
Similarly, for $p=2$ the barycenter quantiles satisfy
$q_\ell^{\times\ast}=\sum_s \alpha_s Q_s^\times(u_\ell)$ and
$\wideq_{\ell}^{\times\ast}=\sum_s \hat\alpha_s \wideq_{s,\ell}^\times$, hence the same bound holds for
$\max_\ell|\wideq_{\ell}^{\times\ast}-q_\ell^{\times\ast}|$ (plus the weight term
$\|\hat\alpha-\alpha\|_1$, controlled at the same $n_{\min}^{-1/2}$ scale and absorbed).

\medskip\noindent\textbf{Step 3: Plug into the discrete quadratic form.}
Let $q^\times=(Q_s^\times(u_\ell))_{s,\ell}$ and $\wideq^\times=(\wideq_{s,\ell}^\times)_{s,\ell}$.
By Lemma~\ref{lem:Phi-lip-app}, on the trimmed region and using that quantiles are uniformly bounded
there (a consequence of the moment/modulus assumptions behind Proposition~\ref{prop:discretization}),
\[
\big|\widehat G_{k,2}-G_{k,2}\big|
=
\big|\Phi_{\hat\alpha}(\wideq^\times)-\Phi_{\alpha}(q^\times)\big|
\le C_\varepsilon\Big(\|\wideq^\times-q^\times\|_\infty + \|\hat\alpha-\alpha\|_1\Big)
\le
C_\varepsilon\Big(\eta+\frac1k+\sqrt{\frac{1}{n_{\min}}\log\!\frac{kd|\cS|}{\delta}}\Big).
\]

\medskip\noindent\textbf{Step 4: Add discretization (grid) bias.}
By Proposition~\ref{prop:discretization} (applied to the mixture laws $\nu_s^\times$),
\[
|G_{k,2}-G_2|\le C_\varepsilon\,\omega(1/k).
\]
Combining the last two displays and recalling that $\eta$ is bounded by
Proposition~\ref{prop:hp-quantiles} on $\mathcal E$ yields the claimed inequality
\[
|\widehat G_{k,2}-G_2|
\le C_\varepsilon\Big(\omega(1/k)+\sqrt{\frac{1}{n_{\min}}\log\!\big(\frac{kd|\cS|}{\delta}\big)}\Big).
\]

\medskip
The same argument applies to $(\widehat V_{\mathrm{mix}},\widehat V_{\mathrm{bar}},\widehat R)$:
each is a finite average of squared differences of (mixture/barycenter) quantile arrays on the grid,
hence is Lipschitz in $\|\cdot\|_\infty$ up to constants depending on the same uniform bounds.
\end{proof}

\section{Proofs for Section~\ref{sec:anova}}\label{app:proof:anova}

\begin{proof}[Proof of Theorem~\ref{thm:anova}]
In one dimension and for $p=2$, Lemma~\ref{lem:Wp-quantiles} gives
\[
G_2
=
\sum_{s\in\cS}\alpha_s\,
\|Q_s^\times-Q^{\times\ast}\|_{L^2([0,1])}^2.
\]
For each $s$, decompose
$Q_s^\times-Q^{\times\ast}=(Q_s^\times-Q_s^\ast)+(Q_s^\ast-Q^{\times\ast})$.
Expanding the squared norm and summing over $s$ with weights $\alpha_s$ yields
\[
G_2
=
\underbrace{\sum_s \alpha_s\|Q_s^\times-Q_s^\ast\|_2^2}_{V_{\mathrm{mix}}}
+
\underbrace{\sum_s \alpha_s\|Q_s^\ast-Q^{\times\ast}\|_2^2}_{V_{\mathrm{bar}}}
+
2\sum_s \alpha_s \langle Q_s^\times-Q_s^\ast,\,Q_s^\ast-Q^{\times\ast}\rangle,
\]
which is exactly \eqref{eq:anova-decomp} with the remainder $R$ as stated.

Finally, by Cauchy--Schwarz in $L^2([0,1])$ and then in the weighted sum over $s$,
\[
|R|
\le
2\sum_s \alpha_s \|Q_s^\times-Q_s^\ast\|_2\,\|Q_s^\ast-Q^{\times\ast}\|_2
\le
2\Big(\sum_s\alpha_s\|Q_s^\times-Q_s^\ast\|_2^2\Big)^{1/2}
\Big(\sum_s\alpha_s\|Q_s^\ast-Q^{\times\ast}\|_2^2\Big)^{1/2}
=
2\sqrt{V_{\mathrm{mix}}V_{\mathrm{bar}}}.
\]
\end{proof}

\begin{proof}[Proof of Proposition~\ref{prop:anova-bounds}]
By Theorem~\ref{thm:anova}, $G_2=V_{\mathrm{mix}}+V_{\mathrm{bar}}+R$ with
$-2\sqrt{V_{\mathrm{mix}}V_{\mathrm{bar}}}\le R\le 2\sqrt{V_{\mathrm{mix}}V_{\mathrm{bar}}}$.
Thus
\[
(\sqrt{V_{\mathrm{mix}}}-\sqrt{V_{\mathrm{bar}}})^2
=
V_{\mathrm{mix}}+V_{\mathrm{bar}}-2\sqrt{V_{\mathrm{mix}}V_{\mathrm{bar}}}
\le G_2
\le
V_{\mathrm{mix}}+V_{\mathrm{bar}}+2\sqrt{V_{\mathrm{mix}}V_{\mathrm{bar}}}
=
(\sqrt{V_{\mathrm{mix}}}+\sqrt{V_{\mathrm{bar}}})^2.
\]
\end{proof}

\begin{proof}[Proof of Corollary~\ref{cor:vanishing}]
We use the definitions in Section~\ref{sec:anova}.

\textbf{(1)} If $m(X,S,A)\indep A\mid S$, then for each $(j,s)$, the conditional law of $Z$ given $(A=j,S=s)$
depends only on $s$, hence $\nu_{j,s}=\nu_s$ for all $j$.
Therefore $\nu_s^\times=\sum_j \pi_{j\mid s}\nu_{j,s}=\nu_s$, and the within-group barycenter
$\nu_s^\ast$ (barycenter of identical measures) also equals $\nu_s$. Hence
$W_2(\nu_s^\times,\nu_s^\ast)=0$ for all $s$, so $V_{\mathrm{mix}}=0$ and
Theorem~\ref{thm:anova} gives $G_2=V_{\mathrm{bar}}$.

\textbf{(2)} 
If $m(X,S,A)\perp\!\!\!\perp S \mid A$, then $\nu_{j,s}=\nu_j$ for all $s$.
If moreover $\pi_{j\mid s}=\beta_j$ for all $j,s$, then $\nu_s^{\times}=\sum_j \beta_j \nu_j$ does not depend on $s$,
so $G_2=\sum_s \alpha_s W_2^2(\nu_s^{\times},\nu^{\times *})=0$.

\textbf{(3)} If both independence conditions hold, then $\nu_{j,s}$ is independent of both $j$ and $s$,
so all group and silo conditional score laws coincide. Hence $G_2=0$.
\end{proof}

\section{Additional experiments on COMPAS}
\label{app:compas}

\subsection{Data and benchmark (centralized) quantities}
\label{app:compas:data}

We use the ProPublica COMPAS two-year dataset and restrict attention to two sensitive groups:
African-American and Caucasian.
After filtering, we obtain $n_{\mathrm{AA}}=3696$ and $n_{\mathrm{C}}=2454$ individuals. 
The score is the \emph{general} decile score, made continuous by jittering within each integer bin:
$Z \leftarrow \texttt{decile\_score} - U$, with $U\sim \mathrm{Unif}(0,1)$, which avoids ties and yields
well-defined densities. 
In the pooled benchmark (all data centralized), the score distribution is shifted to the right for
African-American individuals (e.g., mean $4.87$ vs.\ $3.23$) and the observed two-year recidivism rate is higher
($0.51$ vs.\ $0.39$). 

We compute the centralized reference value $U_2$ using a fine quantile grid ($2001$ points), and similarly report
$H_2$ as well as the inter-group distances $W_2$ (Wasserstein) and $C_2$ (Cramér).
On COMPAS, we obtain
$U_2 = 0.7609$, $H_2 = 0.0077$, $W_2(\mathrm{AA},\mathrm{C})=1.7813$, and $C_2(\mathrm{AA},\mathrm{C})=0.1794$. 

\begin{figure}[!htbp]
\centering
\includegraphics[width=.33\textwidth]{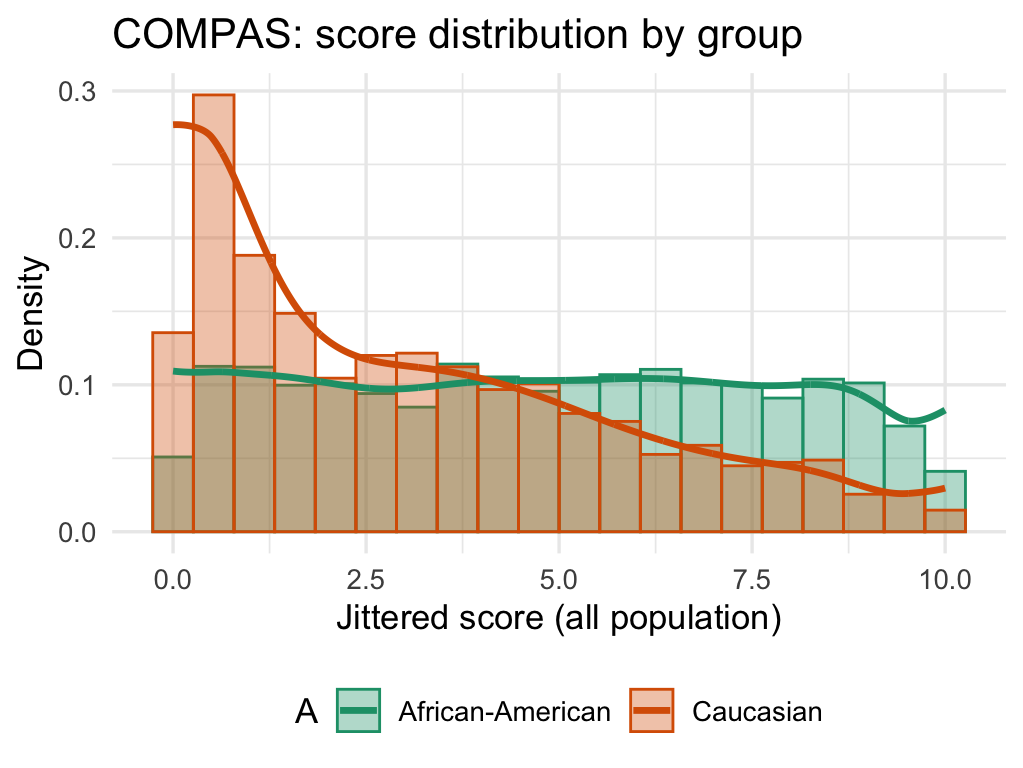}\includegraphics[width=.33\textwidth]{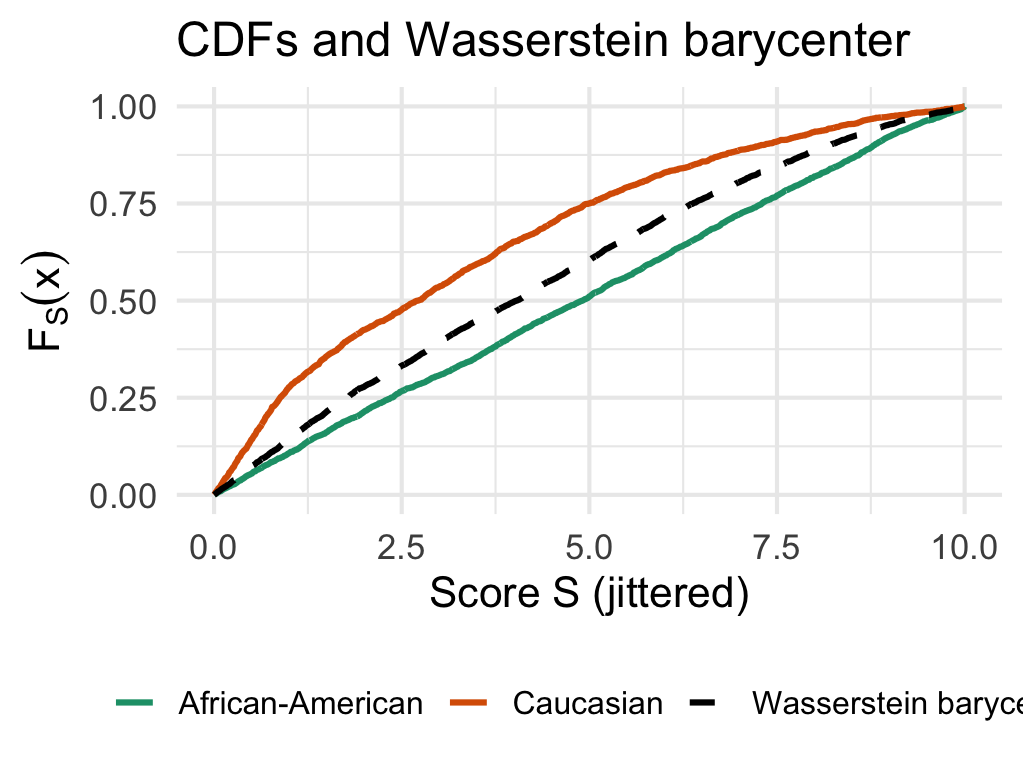}\includegraphics[width=.33\textwidth]{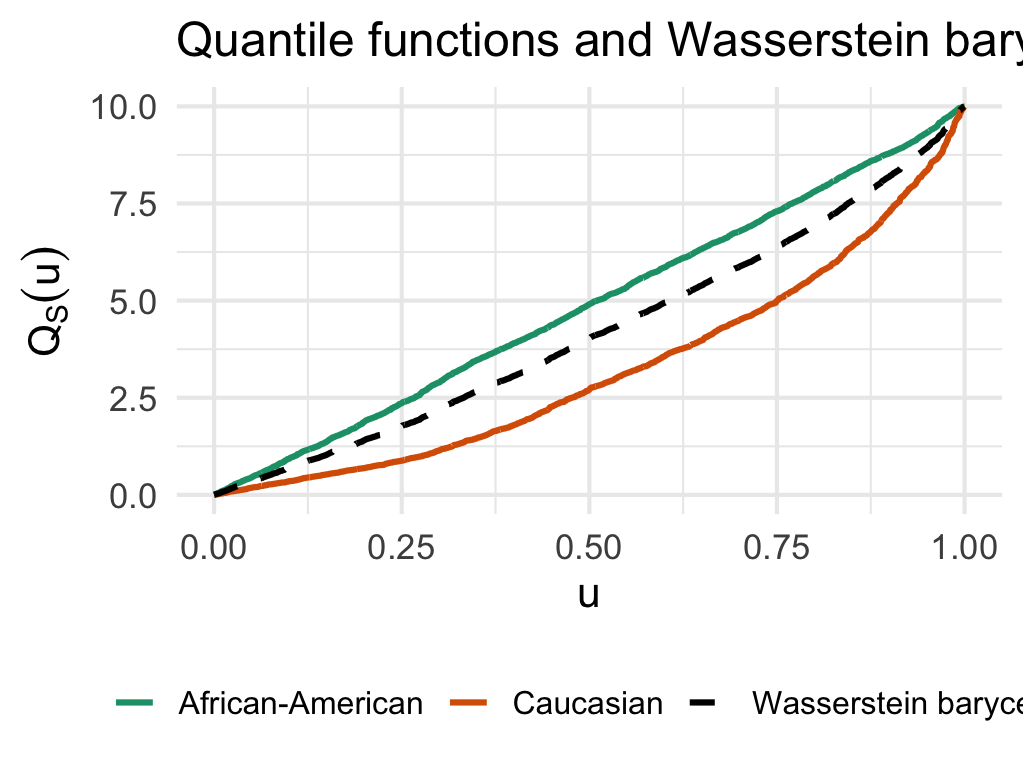}
\caption{\textbf{COMPAS: score distributions and Wasserstein barycenter.}
\emph{Left:} (beta-)kernel density estimates of the jittered score $Z$ by group.
\emph{Middle:} empirical CDFs $F_{\mathrm{AA}}$ and $F_{\mathrm{C}}$ together with the Wasserstein barycenter distribution
(dashed), obtained by inverting the barycenter quantile $Q^\star$.
\emph{Right:} group quantile functions $Q_{\mathrm{AA}}$ and $Q_{\mathrm{C}}$ and their Wasserstein barycenter
$Q^\star = \alpha_{\mathrm{AA}}Q_{\mathrm{AA}} + \alpha_{\mathrm{C}}Q_{\mathrm{C}}$ (dashed).}
\label{fig:compas-plot}
\end{figure}

\paragraph{Continuous (jittered) scores.}
The COMPAS score is originally reported on an integer decile scale.
To avoid ties and to work with smooth densities/quantile curves, we use a jittered version of the score: for each individual, we set $Z := D - U$, where $D\in\{1,\dots,10\}$ is the reported decile and $U\sim \mathrm{Unif}(0,1)$ is drawn independently.
This produces a continuous score on $(0,10)$ while preserving the ordering information within each decile.

Figure~\ref{fig:compas-plot} illustrates the resulting score distributions for the two groups, both in density form (left) and through their CDFs and quantile functions (middle/right). The dashed curves correspond to the 1D Wasserstein barycenter, which provides the natural reference distribution underlying the centralized target $U_2$.

\subsection{Silo allocation mechanism}
\label{app:compas:alloc}

We generate federated scenarios by assigning each individual to one of $d$ silos, under three regimes: \emph{random allocation}, \emph{positive selection bias}, and \emph{negative selection bias}.
A key design choice is to \emph{preserve group-wise silo margins} across regimes, so that the only difference between regimes is the dependence between the score $Z$ and the silo index.

\paragraph{Step 1 (baseline margins).}
For each scenario $b$, we first draw a baseline random silo assignment $\tilde{A}_i \in \{1,\dots,d\}$ i.i.d.\ uniform (with replacement).
This fixes the contingency table of margins
\[
N_{\ell,g} \;=\; \#\{i:\tilde{A}_i=\ell,\; G_i=g\},
\]
where $G_i\in\{\mathrm{AA},\mathrm{C}\}$ is the sensitive group. 

\paragraph{Step 2 (random regime).}
In the \texttt{random} regime, we keep the baseline allocation $A_i=\tilde{A}_i$, hence $Z$ and $A$ are (up to sampling noise) independent. 

\paragraph{Step 3 (selection-bias regimes via a Gaussian copula).}
To introduce controlled dependence between $Z$ and the silo index, we compute a randomized rank transform
\[
R_i \;=\; \frac{\mathrm{rank}(Z_i)}{n+1} \in (0,1),
\]
then form a correlated latent variable
\[
U_i \;=\; \Phi\!\Big(\rho\,\Phi^{-1}(R_i) + \sqrt{1-\rho^2}\,\varepsilon_i\Big),
\qquad \varepsilon_i\sim\mathcal{N}(0,1),
\]
with $\rho\in[0,1)$ controlling the intensity of dependence. 
Finally, within each group $g$, we map $U_i$ to a silo label by cutting $(0,1)$ according to the \emph{baseline margins} $N_{\cdot,g}$:
\[
A_i \;=\; \text{the unique }\ell\text{ such that }\;
U_i \in \Big(\tfrac{\sum_{j<\ell} N_{j,g}}{\sum_j N_{j,g}},\,
\tfrac{\sum_{j\le \ell} N_{j,g}}{\sum_j N_{j,g}}\Big]. 
\]
This construction guarantees that $(N_{\ell,g})$ is identical across regimes, while creating a tunable correlation between $Z$ and $A$.

\paragraph{Positive vs.\ negative bias.}
In \texttt{positive} bias, we use the same $U_i$ for both groups; in \texttt{negative} bias, we flip the copula for one group (equivalently $U_i\mapsto 1-U_i$), so that high-score individuals tend to concentrate in high-index silos for one group and in low-index silos for the other. 
In all cases, we report the empirical Pearson and Spearman correlations $\mathrm{cor}(Z,A)$ as a \emph{diagnostic} of the
realized selection bias. 

See Table~\ref{tab:compas:compare:three} for an illustration on COMPAS.
Our copula-based assignment preserves the group-wise silo margins $(N_{\ell,g})$ by construction, hence the counts $n_{\mathrm{AA}}$ and $n_{\mathrm{C}}$ are the same across the random, positive, and negative regimes.
This makes the comparison ``apples-to-apples'': changes in the per-silo quantities are attributable to the induced dependence between the score $Z$ and the silo index $A$, not to varying sample sizes.

In the \texttt{random} regime, silos are approximately representative subsamples, and per-silo gaps (e.g., $|\Delta \bar{z}|$, $W_2$, and thus $U_2$) fluctuate moderately around their global values.
Under \texttt{positive} selection, both groups are sorted in the same direction, so high-score individuals tend to co-locate in the same silos; this creates systematic within-silo shifts and can increase (or decrease) local distances depending on the silo, while leaving the global benchmark essentially unchanged.
Under \texttt{negative} selection, the sorting is reversed for one group, producing stronger cross-silo heterogeneity: some silos become ``more separated'' (larger $W_2$ and $U_2$) while others become ``less separated'', illustrating how selection bias can distort local fairness assessments and increase the difficulty of reconstruction from small sketches.

\begin{table}[t]
\caption{\textbf{COMPAS: random vs.\ selection-biased silo allocations.}
Global and per-silo summary statistics for African-American (AA) vs.\ Caucasian (C) scores under three regimes:
random assignment (independence), negative selection bias, and positive selection bias
(Gaussian-copula assignment with fixed group-wise margins $(N_{\ell,g})$).
Rows report the mean-score gap $|\Delta\bar{z}|$ and distributional distances ($U_2$, $H_2$, $W_2$).}
\label{tab:compas:compare:three}
\centering
\begin{tabular}{lrrrrrr}
\toprule
metric & Global & Silo 1 & Silo 2 & Silo 3 & Silo 4 & Silo 5\\
\midrule
\multicolumn{7}{l}{\textit{Random assignment (independence)}}\\
$|\Delta\bar{z}|$ & 1.6357 & 1.5583 & 1.5732 & 1.7513 & 1.7481 & 1.5675\\
$U_2$            & 0.7609 & 0.6642 & 0.7069 & 0.8824 & 0.8652 & 0.7296\\
$H_2$            & 0.0077 & 0.0068 & 0.0071 & 0.0092 & 0.0088 & 0.0073\\
$W_2$            & 1.7813 & 1.6769 & 1.7313 & 1.9099 & 1.8858 & 1.7404\\
\addlinespace[0.7ex]

\multicolumn{7}{l}{\textit{Negative selection bias}}\\
$|\Delta\bar{z}|$ & 1.6357 & 5.5613 & 1.8432 & 0.9827 & 3.5526 & 6.7650\\
$U_2$            & 0.7609 & 7.4708 & 0.8175 & 0.2544 & 3.2469 & 11.5793\\
$H_2$            & 0.0077 & 0.0872 & 0.0115 & 0.0037 & 0.0431 & 0.1297\\
$W_2$            & 1.7813 & 5.6780 & 1.9029 & 1.0318 & 3.6327 & 6.8784\\
\addlinespace[0.7ex]

\multicolumn{7}{l}{\textit{Positive selection bias}}\\
$|\Delta\bar{z}|$ & 1.6357 & 0.4483 & 0.9784 & 1.0978 & 0.9667 & 0.5654\\
$U_2$            & 0.7609 & 0.0748 & 0.2634 & 0.3073 & 0.2355 & 0.0745\\
$H_2$            & 0.0077 & 0.0017 & 0.0045 & 0.0044 & 0.0035 & 0.0012\\
$W_2$            & 1.7813 & 0.5479 & 1.0324 & 1.1319 & 1.0104 & 0.6326\\
\bottomrule
\end{tabular}
\end{table}

\subsection{Federated sketch and reconstruction}
\label{app:compas:sketch}

We use midpoint levels $u_\ell=(\ell-\tfrac12)/k$, which avoids the boundary quantiles at $0$ and $1$ and improves robustness to extreme tails.

\subsection{Figures and discussion}
\label{app:compas:figs}

\paragraph{(i) Score distributions.}
We first visualized the two (jittered) score distributions on the pooled COMPAS sample, on Figure~\ref{fig:compas-plot}.
Figure~\ref{fig:original-compas-plot} reports, on original discrete scores (left) histograms by group, (middle) the corresponding CDFs,
and (right) the quantile functions.
The dashed curves show the one-dimensional Wasserstein barycenter, i.e., the natural reference distribution underlying the
centralized target \(U_2\): in 1D, the barycenter is characterized by the barycentric quantile
\(Q^\star(u) = \sum_g w_g\,Q_g(u)\), hence it is natural to overlay it both at the CDF and quantile level. 

\begin{figure}[!htbp]
\centering
\includegraphics[width=.33\textwidth]{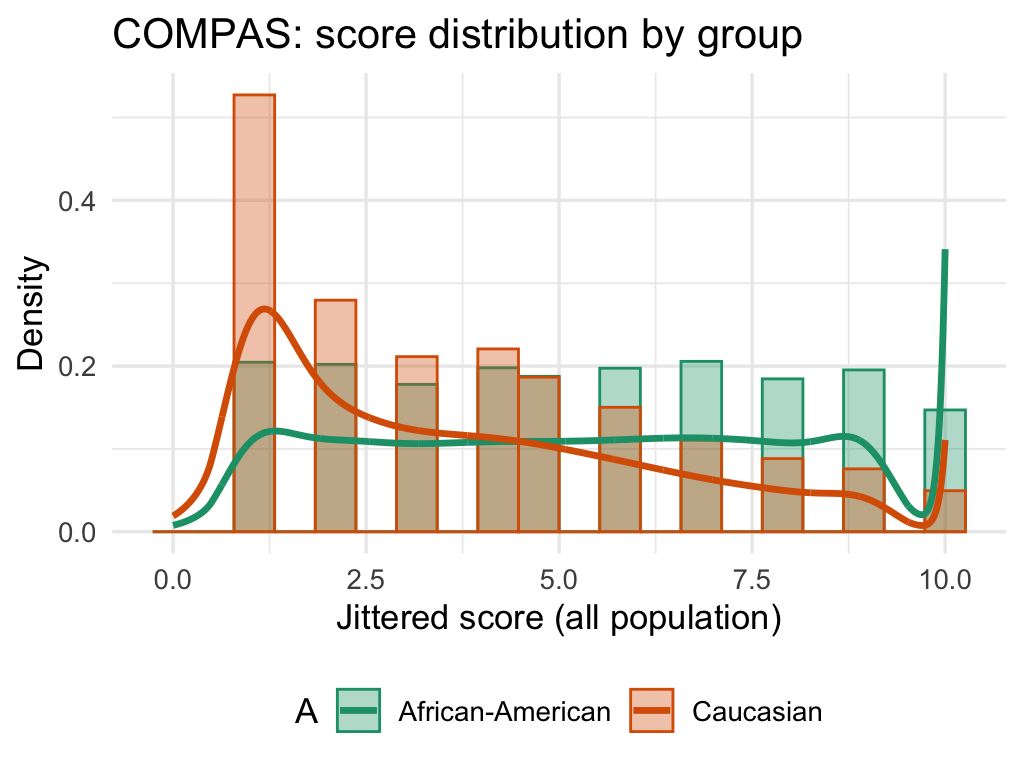}\includegraphics[width=.33\textwidth]{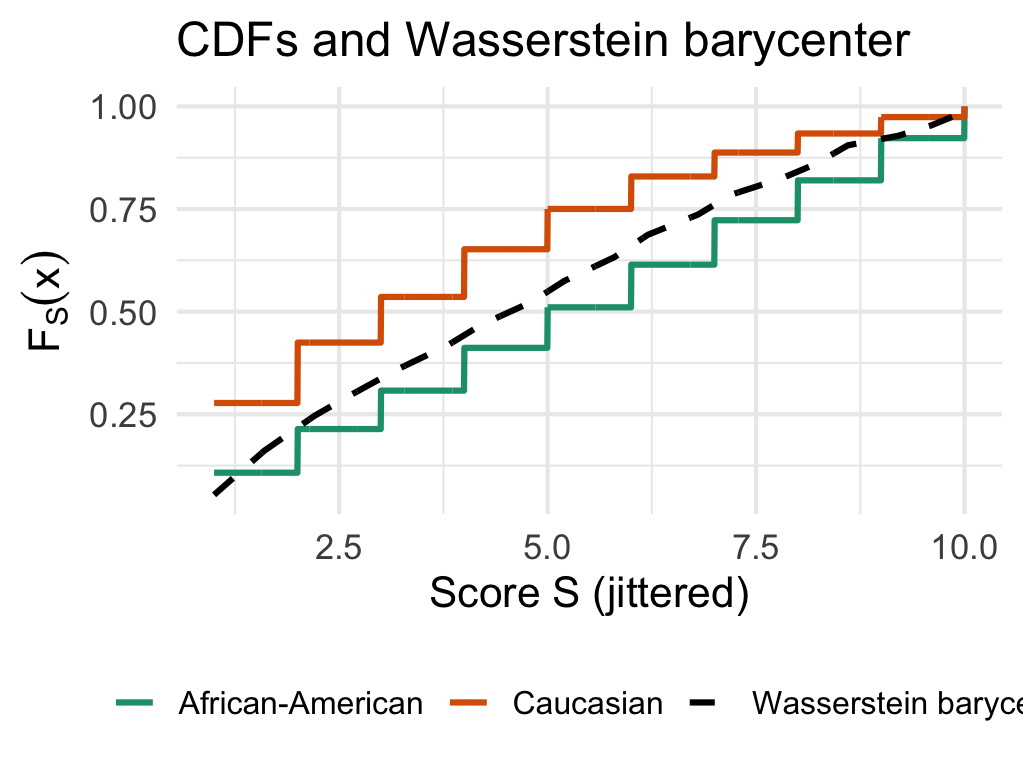}\includegraphics[width=.33\textwidth]{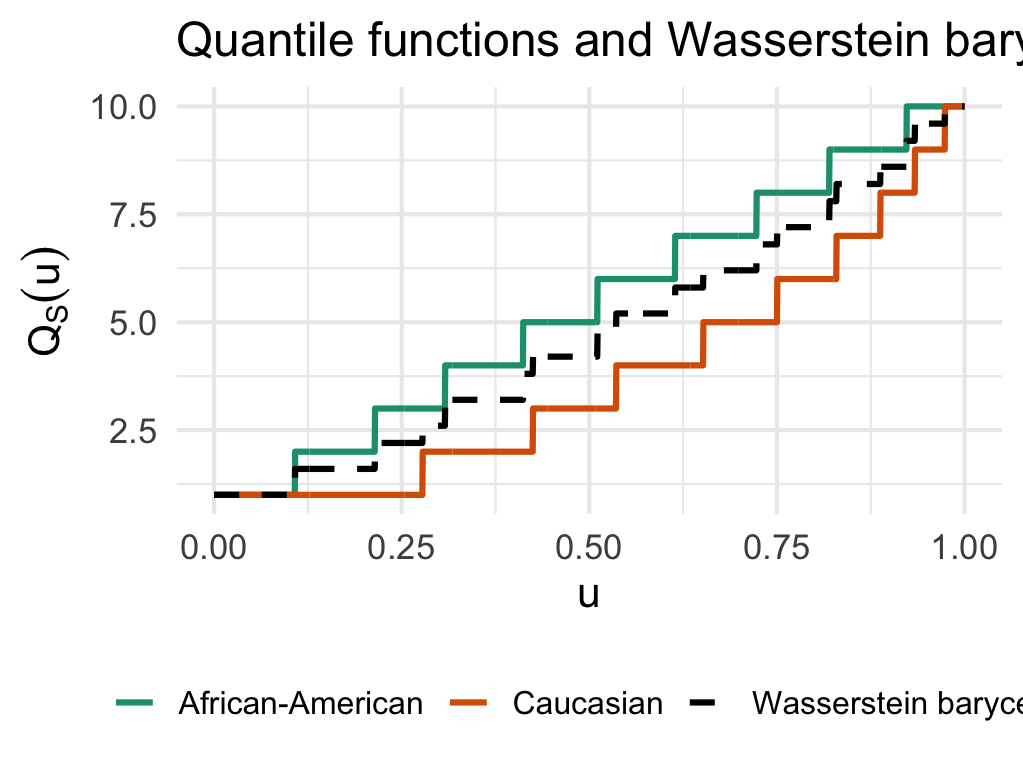}
\caption{\textbf{COMPAS: original score distributions and Wasserstein barycenter.}
\emph{Left:} histograms of score $Z$ by group.
\emph{Middle:} empirical CDFs $F_{\mathrm{AA}}$ and $F_{\mathrm{C}}$ together with the Wasserstein barycenter distribution
(dashed), obtained by inverting the barycenter quantile $Q^\star$.
\emph{Right:} group quantile functions $Q_{\mathrm{AA}}$ and $Q_{\mathrm{C}}$ and their Wasserstein barycenter
$Q^\star = \alpha_{\mathrm{AA}}Q_{\mathrm{AA}} + \alpha_{\mathrm{C}}Q_{\mathrm{C}}$ (dashed).}
\label{fig:original-compas-plot}
\end{figure}

\paragraph{(ii) Convergence in \(k\).}
We then study how the sketch size \(k\) (the number of quantiles transmitted per \((\text{group},\text{silo})\)) affects the
accuracy of the federated estimator \(\widehat U_2(k)\).
Figure~\ref{fig:compas-mae-k} reports the mean absolute error
\(\mathrm{MAE}(k)=\mathbb{E}\big|\widehat{U}_2(k)-U_2\big|\) as a function of \(k\), for several values of \(d\) and for each
allocation regime (random vs.\ selection bias).
Across regimes, \(\widehat U_2(k)\) stabilizes rapidly: modest grids already recover \(U_2\) accurately, with diminishing returns
beyond a few dozen quantiles.

To complement averages, Figure~\ref{fig:compas-box-k} shows the full distribution of \(\widehat U_2(k)\) over Monte Carlo
allocations for a fixed \(d\). The estimator concentrates near the centralized benchmark \(U_2\) even for moderate \(k\),
illustrating that the protocol is robust to scenario-to-scenario variability once the sketch is not too small. 

\begin{figure}[!htbp]
\centering
\includegraphics[width=.99\textwidth]{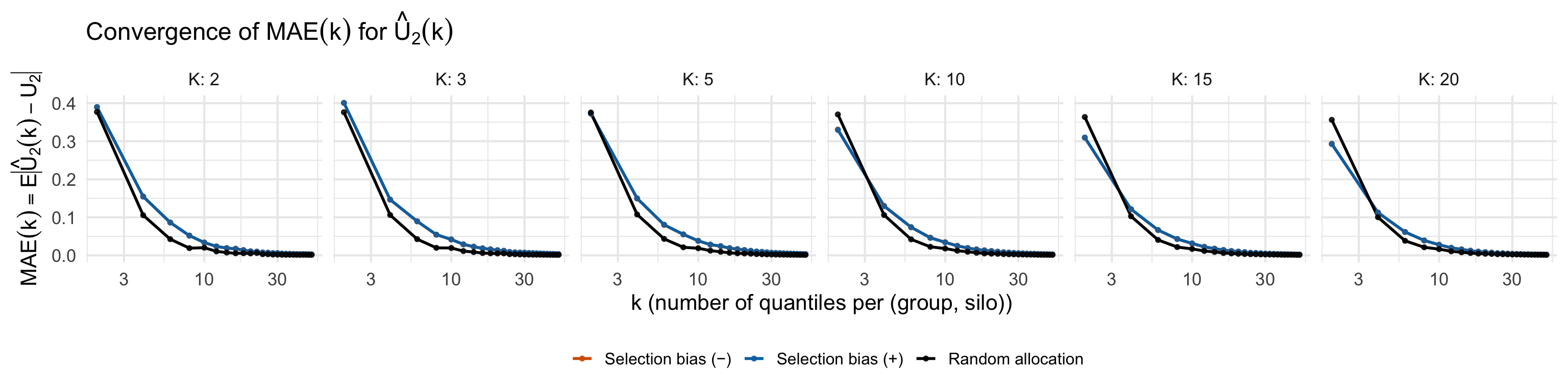}
\caption{\textbf{COMPAS: convergence in $k$.}
Mean absolute error $\mathrm{MAE}(k)=\mathbb{E}\,|\widehat U_2(k)-U_2|$ as a function of the number of quantiles $k$
(sent per group and per silo) on a log scale, for several numbers of silos $d$ and for different allocation regimes
(random vs.\ selection bias).
Across regimes, $\widehat U_2(k)$ stabilizes quickly, with diminishing returns beyond a few dozen quantiles.}
\label{fig:compas-mae-k}
\end{figure}

\paragraph{(iii) Reliability profiles \(p_{\mathrm{ok}}(k)\).}
Beyond MAE, we report the probability of meeting a target relative accuracy
\[
p_{\mathrm{ok}}(k)=\mathbb{P}\!\left(\frac{|\widehat{U}_2(k)-U_2|}{U_2}\le \tau\right),
\qquad (\tau=1\%).
\]
Figure~\ref{fig:compas-pok-k} makes clear that increasing \(k\) quickly drives the bulk of runs into the target-accuracy region,
and that selection bias typically requires slightly larger sketches to achieve the same reliability. 

A compact summary is provided by Figure~\ref{fig:compas-k95-heatmap}, which reports
\(k_{95}\), the smallest \(k\) such that \(p_{\mathrm{ok}}(k)\ge 95\%\), as a function of \(d\) and the allocation regime.
The heatmap highlights, in one glance, the additional sketch size needed under selection bias compared to random splits. 

\begin{figure}[!htbp]
\centering
\includegraphics[width=.99\textwidth]{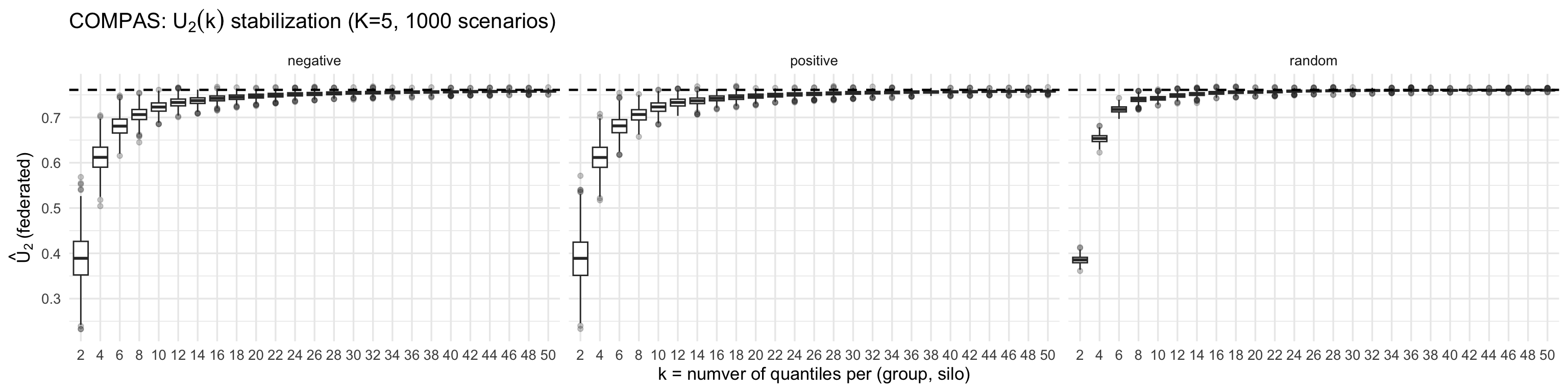}
\caption{\textbf{COMPAS: fast stabilization of $\widehat U_2(k)$.}
Boxplots of the federated estimator $\widehat U_2(k)$ over repeated allocations (Monte Carlo scenarios),
for a fixed number of silos $d$ and varying $k$.
The dashed horizontal line is the centralized benchmark $U_2$ computed on the pooled data with a fine grid.
Even for moderate $k$, the distribution of $\widehat U_2(k)$ concentrates near $U_2$.}
\label{fig:compas-box-k}
\end{figure}

\paragraph{(iv) Impact of the number of silos \(d\).}
At fixed \(k\), increasing \(d\) makes local \((\text{group},\text{silo})\) samples smaller on average, hence local quantiles
noisier; this effect is visible through the upward shift of errors / slower reliability gains as \(d\) grows in
Figures~\ref{fig:compas-mae-k} and \ref{fig:compas-pok-k}, and it is summarized by the larger \(k_{95}\) values in
Figure~\ref{fig:compas-k95-heatmap}. Practically, this highlights the natural trade-off between federation granularity (large \(d\))
and sketch resolution (larger \(k\)).

\begin{figure}[!htbp]
\centering
\includegraphics[width=.99\textwidth]{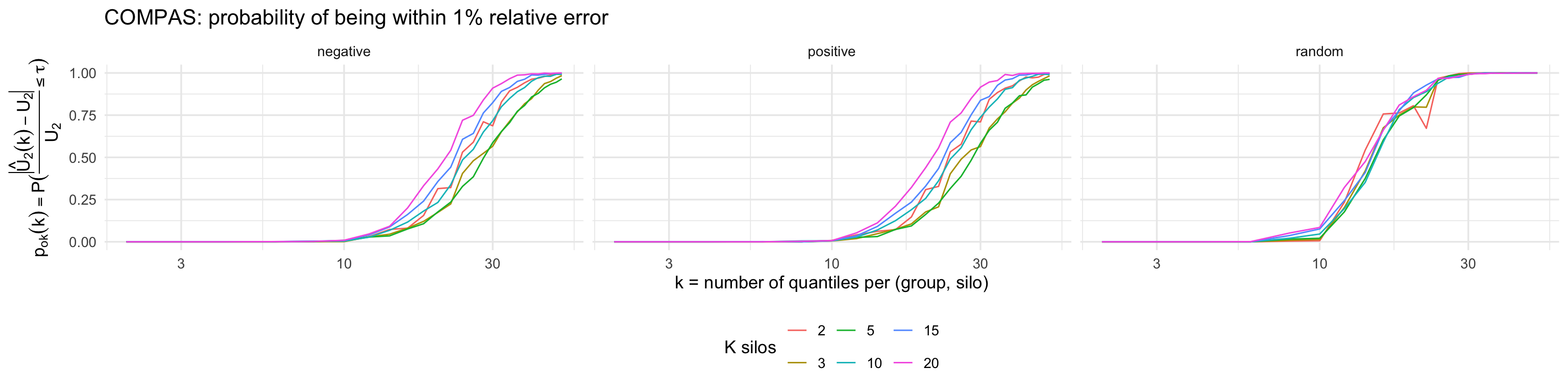}
\caption{\textbf{COMPAS: reliability profile.}
Probability of meeting a relative-error target
$p_{\mathrm{ok}}(k)=\mathbb{P}\!\left(\frac{|\widehat U_2(k)-U_2|}{U_2}\le \tau\right)$ (here $\tau=1\%$)
as a function of $k$ (on a log scale), for several $d$ and allocation regimes.
This complements MAE curves by showing how quickly the bulk of runs reaches the target accuracy.}
\label{fig:compas-pok-k}
\end{figure}

\paragraph{(v) Error vs.\ selection intensity.}
Finally, we relate accuracy to the realized score--silo dependence (selection intensity), measured by
\(|\mathrm{cor}_Z(Z,A)|\) (Spearman).
Figure~\ref{fig:compas-error-vs-corr} shows that stronger selection (larger dependence) tends to worsen estimation accuracy at fixed
\(k\), consistent with more heterogeneous within-silo score distributions, which amplifies discretization and reconstruction errors.

\paragraph{(vi) Centralized discretization vs.\ federated estimates.}
To disentangle discretization from federation noise, Figure~\ref{fig:compas-central-vs-fed} compares the centralized discretized
target \(U_2(k)\) (computed on pooled data with a \(k\)-point grid) to the federated estimate \(\widehat U_2(k)\).
The closeness of the two curves for moderate \(k\) suggests that, once communication is sufficient, the remaining gap to the
fine-grid reference \(U_2\) is primarily explained by discretization rather than by federation noise. 

\begin{figure}[!htbp]
\centering
\includegraphics[width=.99\textwidth]{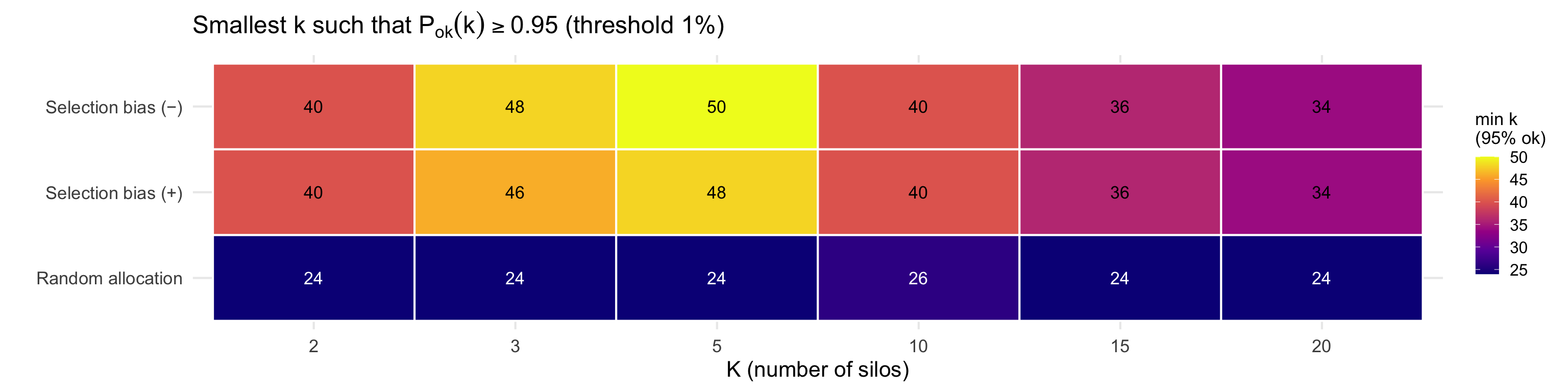}
\caption{\textbf{COMPAS: minimal sketch size for a target reliability.}
Heatmap of $k_{95}$, the smallest number of quantiles $k$ such that
$p_{\mathrm{ok}}(k)\ge 95\%$ (with $\tau=1\%$), as a function of the number of silos $d$
and the allocation regime.
Selection bias requires larger sketches to achieve the same reliability.}
\label{fig:compas-k95-heatmap}
\end{figure}


\begin{figure}[!htbp]
\centering
\includegraphics[width=.99\textwidth]{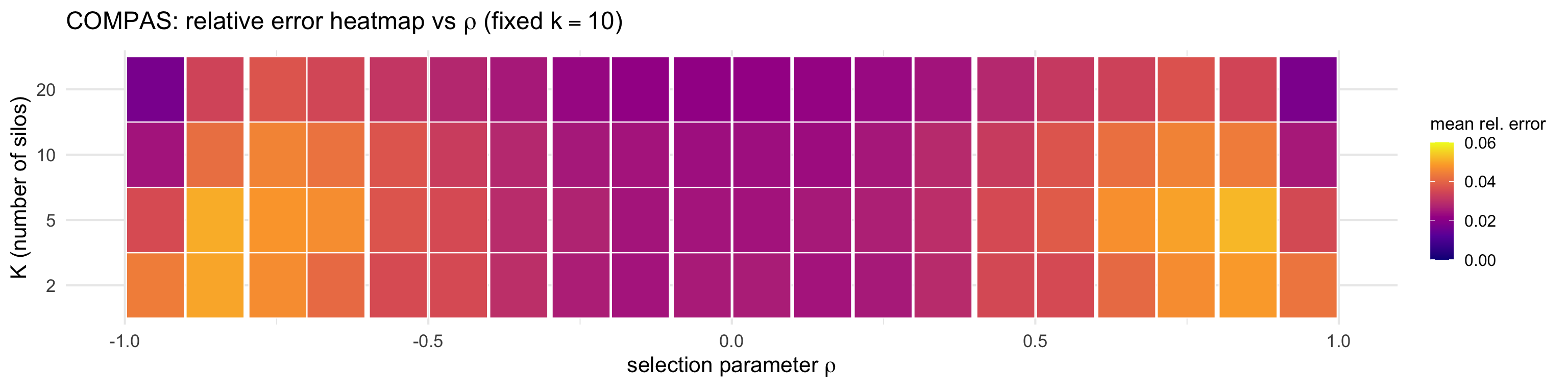}

\includegraphics[width=.99\textwidth]{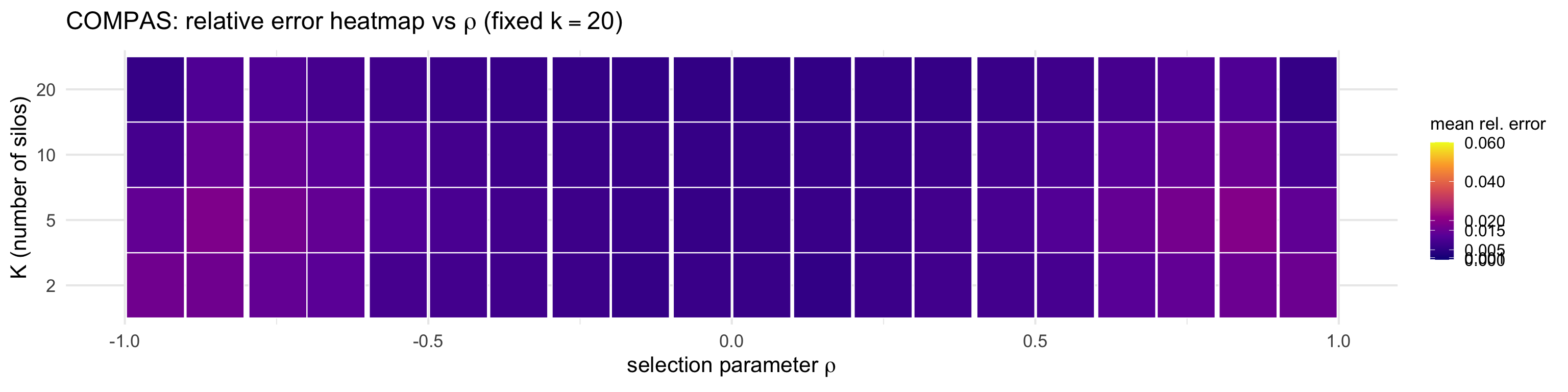}
\caption{\textbf{COMPAS: error versus selection intensity.} ($k=10$ {\em top} and $k=20$ {\em bottom})
Relative error $\frac{|\widehat U_2(k)-U_2|}{U_2}$ versus the realized score--silo dependence,
measured by $|\mathrm{cor}_Z(Z,A)|$ (Spearman correlation).
Stronger selection (higher dependence) tends to increase estimation error at fixed sketch size,
consistent with more heterogeneous within-silo score distributions.}
\label{fig:compas-error-vs-corr}
\end{figure}

\paragraph{(vii) Centralized discretization vs.\ federated estimates.}
To disentangle the effects, we compare the centralized discretized target $U_2(k)$ (computed directly from the pooled sample on a
midpoint grid) to the federated estimates $\widehat{U}_2(k)$.
In these experiments, the two curves are close for moderate $k$, suggesting that (once communication is sufficient) the remaining gap
is primarily explained by discretization rather than federation noise. 

\begin{figure}[!htbp]
\centering
\includegraphics[width=.99\textwidth]{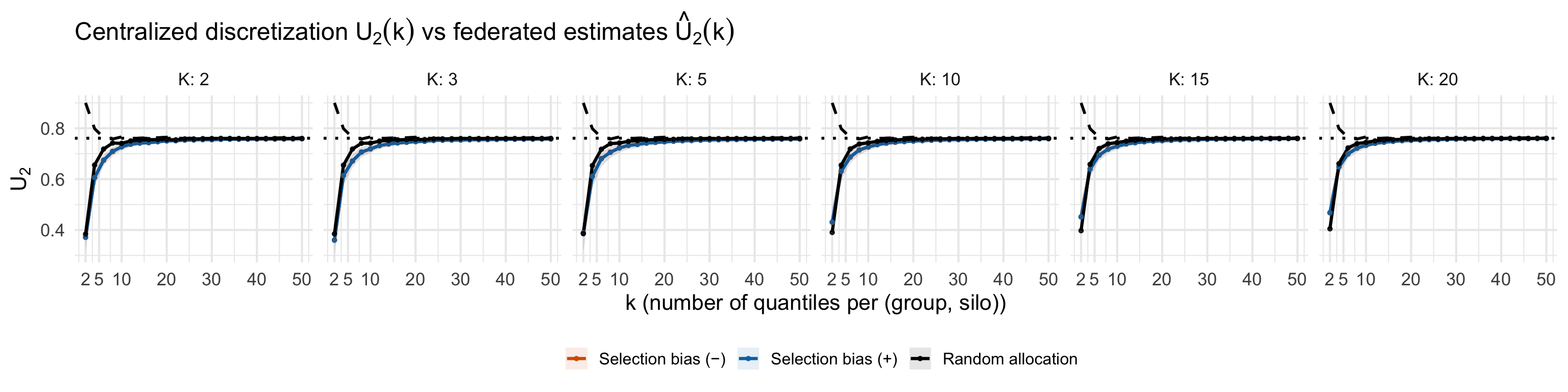}
\caption{\textbf{COMPAS: centralized discretization vs.\ federated estimation.}
Comparison of the centralized discretized target $U_2(k)$ (computed from pooled data on a $k$-point grid)
and the federated estimator $\widehat U_2(k)$ (mean over scenarios, with variability bands).
The horizontal reference is $U_2$ computed on a fine grid.
The closeness of the curves for moderate $k$ indicates that discretization and sketch resolution are the main drivers,
beyond federation noise.}
\label{fig:compas-central-vs-fed}
\end{figure}

\paragraph{(viii) Communication budget vs.\ error.}
With two groups, each silo transmits $2(k+1)$ scalars (namely $k$ quantile values plus one count per group),
so the total communication budget is $2d(k+1)$ scalars.
Figure~\ref{fig:compas-budget} plots relative error versus this budget, yielding a direct accuracy--communication trade-off:
for fixed \(d\), increasing \(k\) quickly improves accuracy; for fixed \(k\), larger \(d\) may increase noise through smaller local
samples, highlighting again the granularity--resolution trade-off. 

\begin{figure}[!htbp]
\centering
\includegraphics[width=.99\textwidth]{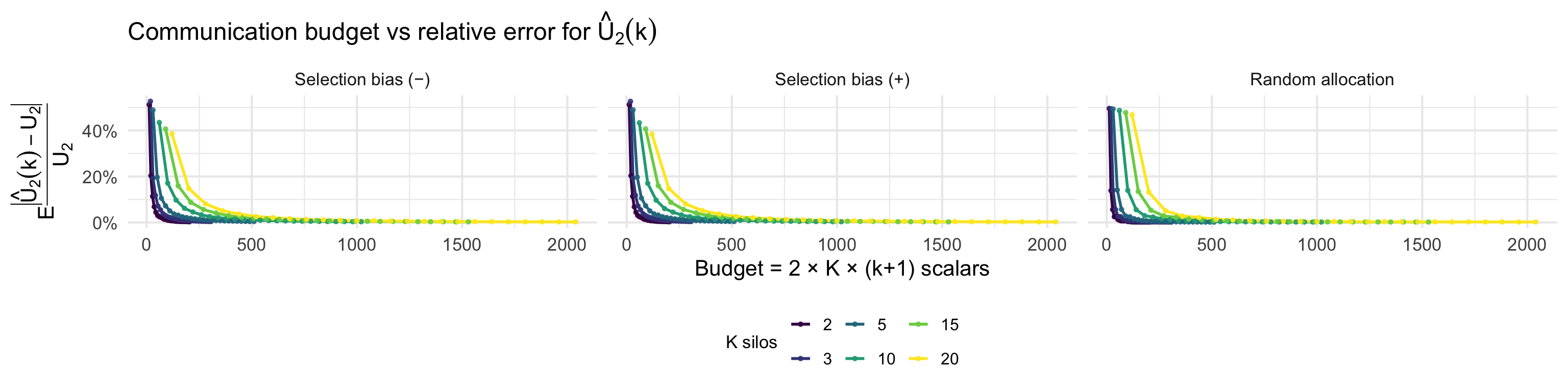}
\caption{\textbf{COMPAS: accuracy--communication trade-off.}
Relative error $\mathbb{E}\!\left[\frac{|\widehat U_2(k)-U_2|}{U_2}\right]$ versus a simple communication proxy
$\text{budget}=2d(k+1)$ scalars (two groups, $d$ silos, and $(k+1)$ quantiles per group and silo).
For fixed $d$, increasing $k$ quickly improves accuracy; for fixed $k$, larger $d$ may increase noise through smaller local samples,
highlighting a practical trade-off between federation granularity and sketch resolution.}
\label{fig:compas-budget}
\end{figure}

\section{Additional experiments on Adult}
\label{app:adult}

We use the Adult (Census Income) dataset from the UCI Machine Learning Repository \cite{BeckerKohavi1996Adult},
a widely used benchmark in the fairness literature \cite{Feldman2015DisparateImpact,Zafar2017FairnessConstraints,Agarwal2018Reductions}.

\paragraph{Data.}
We focus on a binary sensitive attribute derived from \texttt{race} by restricting individuals labeled as ``\texttt{Black}'' or
``\texttt{White}'', denoted (for consistency with the main text) as \emph{AA} and \emph{C}, respectively.
The prediction task is income classification (\texttt{income} $>$ 50K), and we work with a continuous score
$Z\in(0,1)$ defined as the predicted probability $\widehat p(X)$ from a logistic regression model trained without the sensitive
attribute (and, in our default specification, without \texttt{sex} as well).
This yields a smooth score distribution suitable for quantile sketches.

\paragraph{Federated silo interpretation.}
We reuse the federated protocol and the three allocation regimes from Appendix~\ref{app:compas} (random assignment; positive and
negative score--silo dependence via a Gaussian-copula construction with fixed group-wise margins).
Each silo releases only group counts and $k$ quantiles per group, and the server reconstructs global quantile functions to estimate
the distributional fairness targets (here reported through $|\Delta\bar y|$, $U_2$, $H_2$, and $W_2$; see Appendix~\ref{app:compas}
for definitions).

\begin{table}[t]
\caption{\textbf{Adult: random vs.\ selection-biased silo allocations.}
Global and per-silo summary statistics for AA vs.\ C scores under three regimes: random allocation (independence),
negative selection bias, and positive selection bias (Gaussian-copula assignment with fixed group-wise margins).
Rows report the mean-score gap $|\Delta\bar{z}|$ and distributional distances ($U_2$, $H_2$, $W_2$).}
\label{tab:adult:compare:three}
\centering
\begin{tabular}{lrrrrrr}
\toprule
metric & Global & Silo 1 & Silo 2 & Silo 3 & Silo 4 & Silo 5\\
\midrule
\multicolumn{7}{l}{\textit{Random assignment (independence)}}\\
$|\Delta\bar{z}|$ & 0.1338 & 0.1245 & 0.1458 & 0.1410 & 0.1366 & 0.1203\\
$U_2$            & 0.0029 & 0.0025 & 0.0034 & 0.0035 & 0.0032 & 0.0022\\
$H_2$            & 0.0021 & 0.0019 & 0.0024 & 0.0023 & 0.0021 & 0.0017\\
$W_2$            & 0.1821 & 0.1705 & 0.1989 & 0.1945 & 0.1869 & 0.1611\\
\addlinespace[0.7ex]

\multicolumn{7}{l}{\textit{Negative selection bias}}\\
$|\Delta\bar{z}|$ & 0.1338 & 0.6052 & 0.2813 & 0.0769 & 0.1402 & 0.4825\\
$U_2$            & 0.0029 & 0.0501 & 0.0111 & 0.0008 & 0.0022 & 0.0159\\
$H_2$            & 0.0021 & 0.0502 & 0.0123 & 0.0010 & 0.0024 & 0.0150\\
$W_2$            & 0.1821 & 0.6600 & 0.3348 & 0.0946 & 0.1734 & 0.5555\\
\addlinespace[0.7ex]

\multicolumn{7}{l}{\textit{Positive selection bias}}\\
$|\Delta\bar{z}|$ & 0.1338 & 0.0110 & 0.0468 & 0.0751 & 0.1252 & 0.0927\\
$U_2$            & 0.0029 & 0.0001 & 0.0006 & 0.0008 & 0.0013 & 0.0005\\
$H_2$            & 0.0021 & 0.0001 & 0.0006 & 0.0010 & 0.0015 & 0.0005\\
$W_2$            & 0.1821 & 0.0316 & 0.0735 & 0.0940 & 0.1372 & 0.1018\\
\bottomrule

\end{tabular}
\end{table}

\subsection{Figures and discussion}
\label{app:adult:figs}

\paragraph{(i) Convergence in \(k\).}
We study how the sketch size \(k\) affects the accuracy of \(\widehat U_2(k)\).
Figure~\ref{fig:adult:mae-k} reports the mean absolute error
\(\mathrm{MAE}(k)=\mathbb{E}\big|\widehat{U}_2(k)-U_2\big|\) as a function of \(k\), for several values of \(d\) and for each
allocation regime. Across regimes, \(\widehat U_2(k)\) stabilizes rapidly, with diminishing returns beyond a few dozen quantiles.

\begin{figure}[!htbp]
\centering
\includegraphics[width=.99\textwidth]{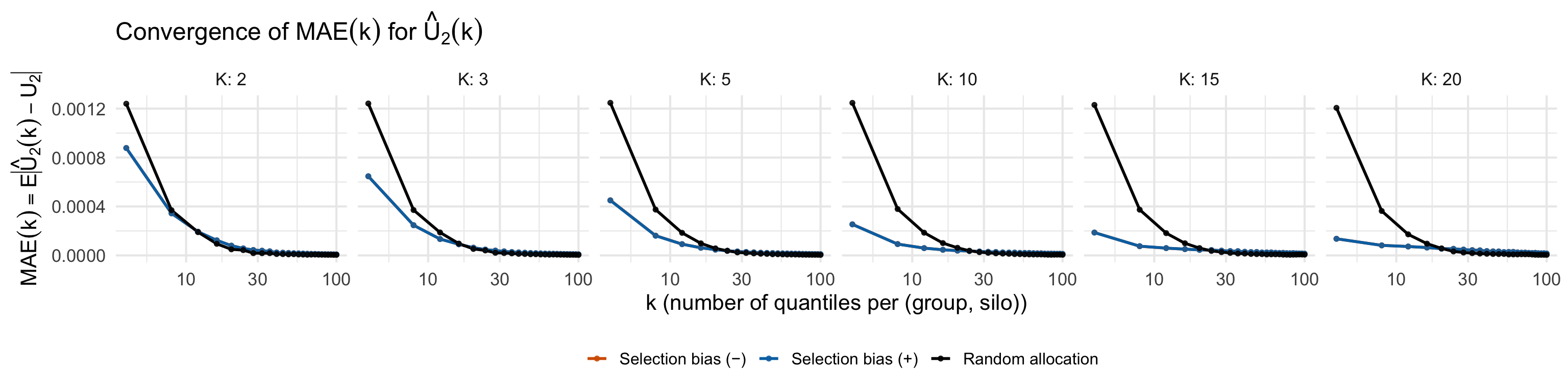}
\caption{\textbf{Adult: convergence in $k$.}
Mean absolute error $\mathrm{MAE}(k)=\mathbb{E}\,|\widehat U_2(k)-U_2|$ as a function of the number of quantiles $k$
(sent per group and per silo) on a log scale, for several numbers of silos $d$ and for different allocation regimes.}
\label{fig:adult:mae-k}
\end{figure}

\paragraph{(ii) Distribution over scenarios.}
Figure~\ref{fig:adult:box-k} shows the distribution of \(\widehat U_2(k)\) across Monte Carlo allocations for a fixed \(d\).
The estimator concentrates near the centralized benchmark \(U_2\) for moderate \(k\), illustrating robustness to scenario variability.

\begin{figure}[!htbp]
\centering
\includegraphics[width=.99\textwidth]{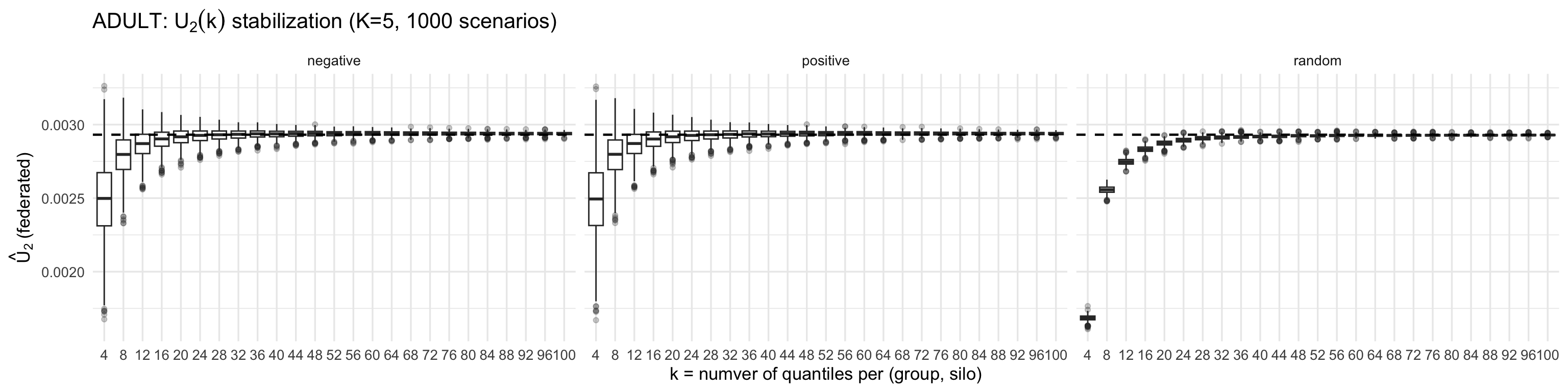}
\caption{\textbf{Adult: stabilization of $\widehat U_2(k)$.}
Boxplots of $\widehat U_2(k)$ over repeated allocations for a fixed $d$ and varying $k$.
The dashed horizontal line is the centralized benchmark $U_2$.}
\label{fig:adult:box-k}
\end{figure}

\paragraph{(iii) Reliability profiles \(p_{\mathrm{ok}}(k)\).}
We also report
\[
p_{\mathrm{ok}}(k)=\mathbb{P}\!\left(\frac{|\widehat{U}_2(k)-U_2|}{U_2}\le \tau\right),\qquad (\tau=1\%).
\]
Figure~\ref{fig:adult:pok-k} shows that increasing \(k\) quickly pushes most runs into the target-accuracy region.
Figure~\ref{fig:adult:k95-heatmap} summarizes this via \(k_{95}\), the smallest \(k\) such that \(p_{\mathrm{ok}}(k)\ge 95\%\).

\begin{figure}[!htbp]
\centering
\includegraphics[width=.99\textwidth]{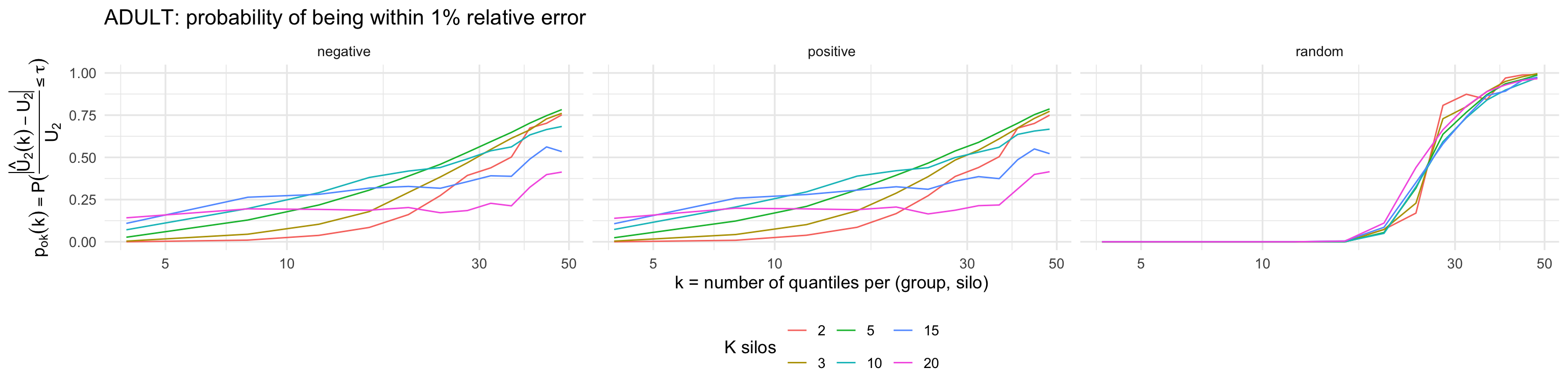}
\caption{\textbf{Adult: reliability profile.}
$p_{\mathrm{ok}}(k)=\mathbb{P}\!\left(\frac{|\widehat U_2(k)-U_2|}{U_2}\le \tau\right)$ (here $\tau=1\%$)
as a function of $k$ (log scale), for several $d$ and allocation regimes.}
\label{fig:adult:pok-k}
\end{figure}

\begin{figure}[!htbp]
\centering
\includegraphics[width=.99\textwidth]{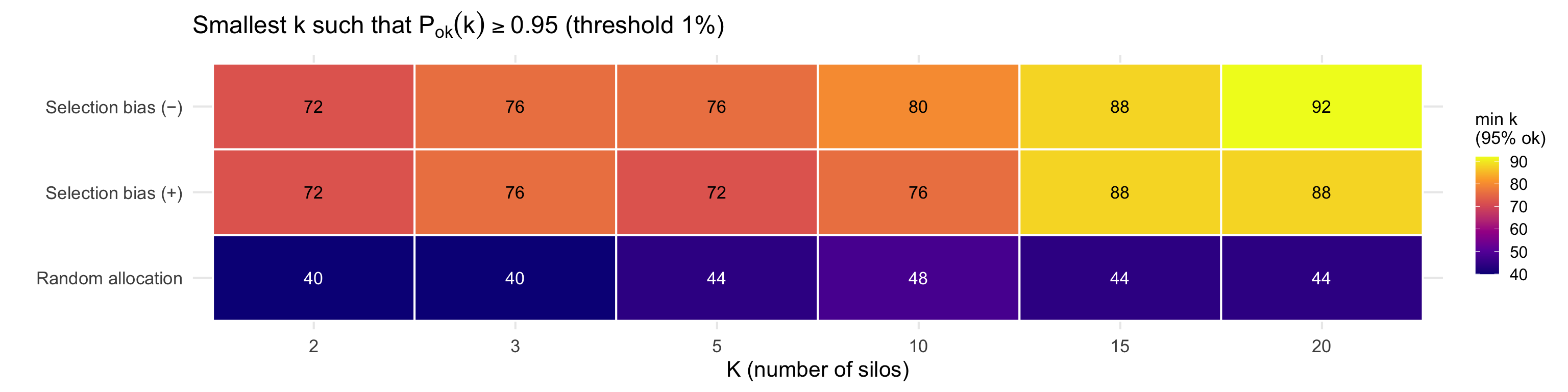}
\caption{\textbf{Adult: minimal sketch size for a target reliability.}
Heatmap of $k_{95}$, the smallest $k$ such that $p_{\mathrm{ok}}(k)\ge 95\%$ (with $\tau=1\%$), as a function of $d$
and the allocation regime.}
\label{fig:adult:k95-heatmap}
\end{figure}

\paragraph{(iv) Error vs.\ selection intensity.}
We relate accuracy to realized score--silo dependence measured by $|\mathrm{cor}_Z(Z,A)|$ (Spearman).
Figure~\ref{fig:adult:error-vs-corr} shows that stronger selection tends to increase error at fixed \(k\), consistent with more
heterogeneous within-silo score distributions.

\begin{figure}[!htbp]
\centering
\includegraphics[width=.99\textwidth]{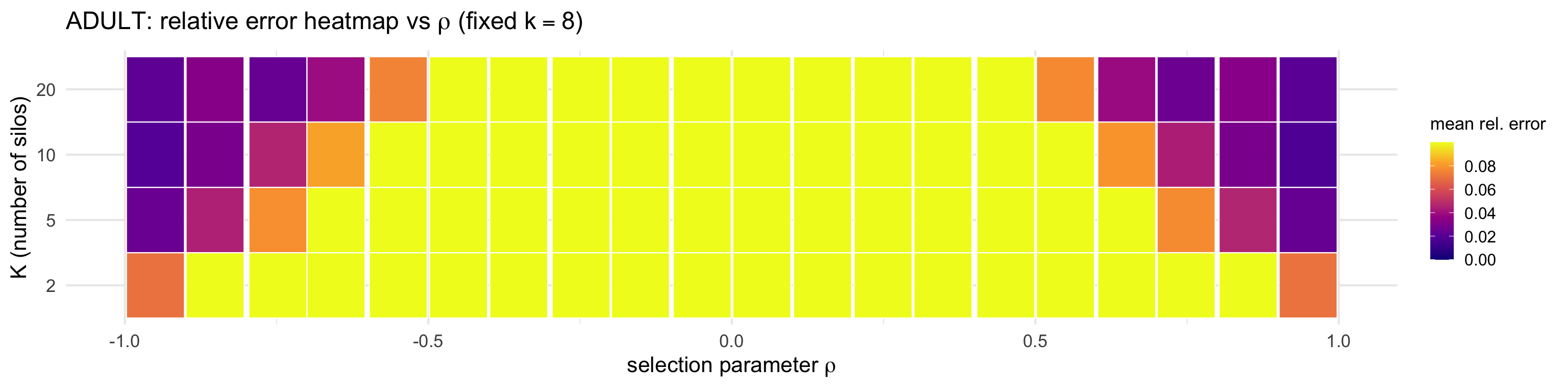}\\[-0.2em]
\includegraphics[width=.99\textwidth]{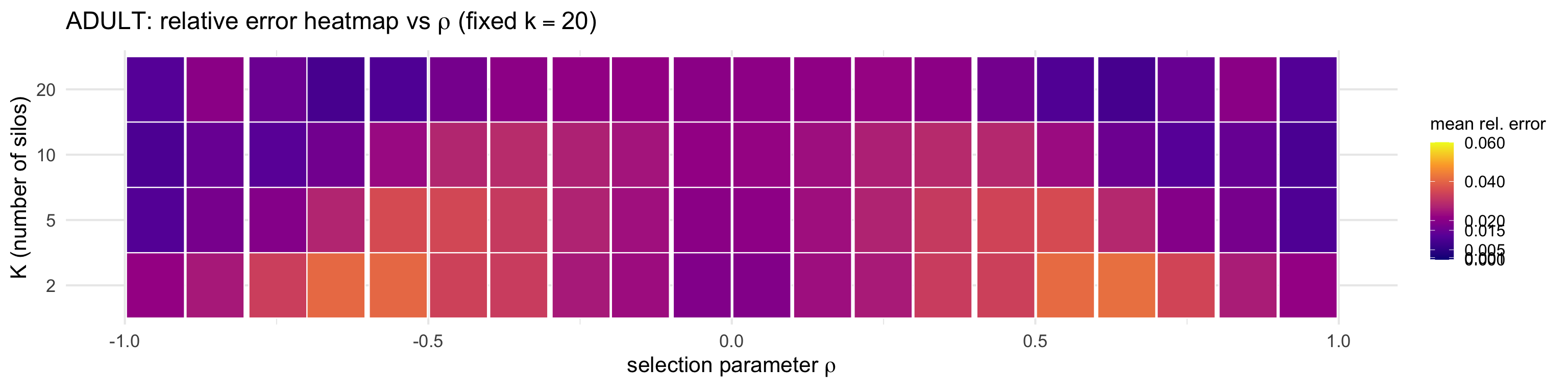}
\caption{\textbf{Adult: error versus selection intensity.} ($k=10$ top, $k=20$ bottom)
Relative error $\frac{|\widehat U_2(k)-U_2|}{U_2}$ versus $|\mathrm{cor}_Z(Z,A)|$ (Spearman).}
\label{fig:adult:error-vs-corr}
\end{figure}

\paragraph{(v) Centralized discretization vs.\ federated estimates.}
To separate discretization from federation noise, Figure~\ref{fig:adult:central-vs-fed} compares the centralized discretized
target \(U_2(k)\) (pooled data on a \(k\)-point grid) with the federated estimate \(\widehat{U}_2(k)\).
For moderate \(k\), the curves are close, indicating that the remaining gap to the fine-grid reference \(U_2\) is primarily driven
by discretization.

\begin{figure}[!htbp]
\centering
\includegraphics[width=.99\textwidth]{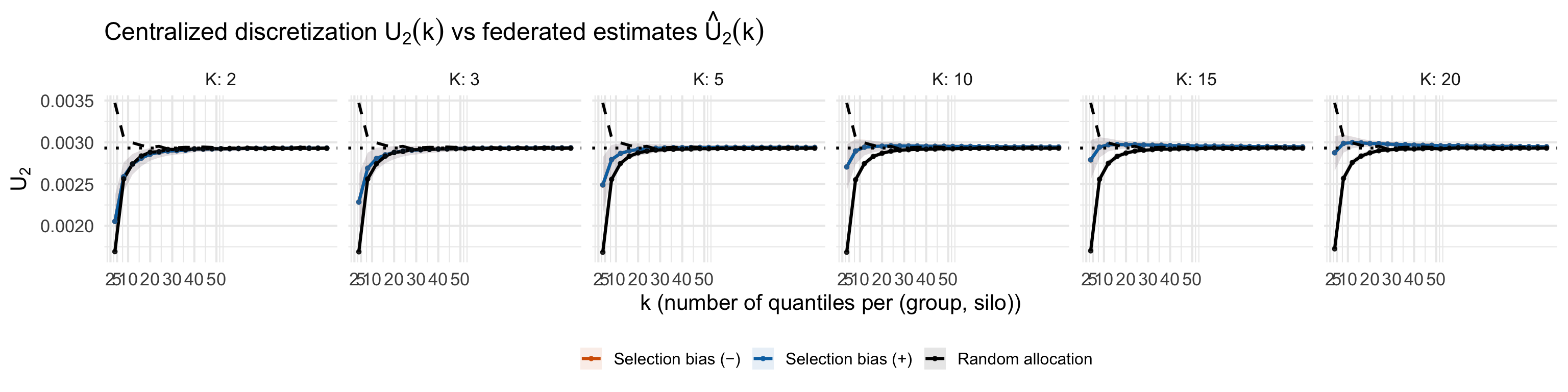}
\caption{\textbf{Adult: centralized discretization vs.\ federated estimation.}
Centralized $U_2(k)$ (pooled data on a $k$-grid) versus federated $\widehat U_2(k)$ (mean over scenarios with variability bands).
The horizontal reference is $U_2$ computed on a fine grid.}
\label{fig:adult:central-vs-fed}
\end{figure}

\paragraph{(vi) Communication budget vs.\ error.}
With two groups, each silo transmits $2(k+1)$ scalars, so the total budget is $2d(k+1)$.
Figure~\ref{fig:adult:budget} reports the relative error versus this proxy, yielding a direct accuracy--communication trade-off.

\begin{figure}[!htbp]
\centering
\includegraphics[width=.99\textwidth]{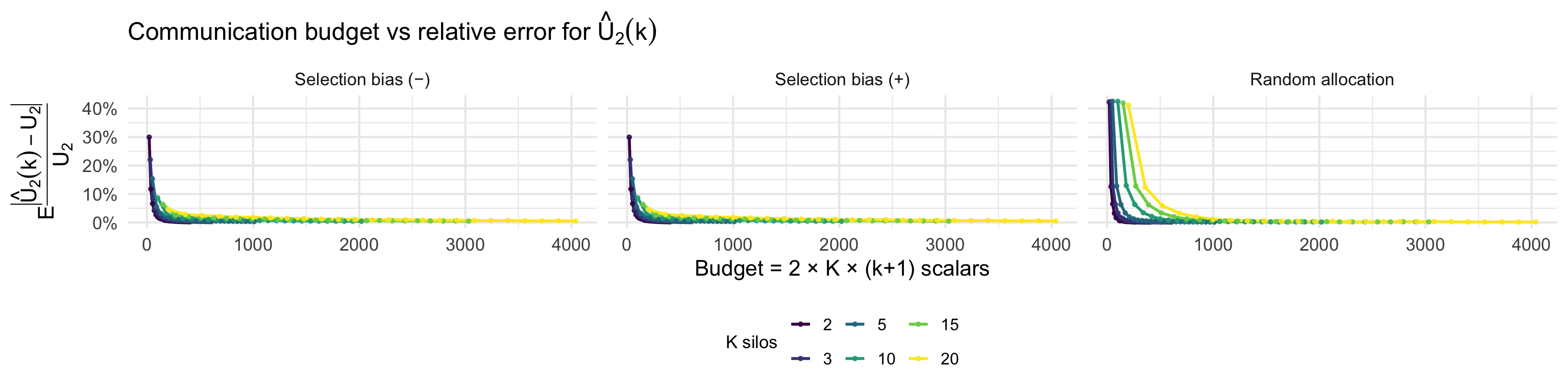}
\caption{\textbf{Adult: accuracy--communication trade-off.}
$\mathbb{E}\!\left[\frac{|\widehat U_2(k)-U_2|}{U_2}\right]$ versus $\text{budget}=2d(k+1)$.}
\label{fig:adult:budget}
\end{figure}

\section{Additional experiments on Law}
\label{app:law}

We use the Law School dataset derived from the LSAC National Longitudinal Bar Passage Study \cite{Wightman1998NLBPS},
which has been frequently used as a benchmark for fairness-aware learning and evaluation
\cite{LeQuy2022SurveyDatasetsFairness,Narasimhan2020PairwiseFairness}.

\paragraph{Data.}
We consider a binary sensitive attribute by restricting to individuals labeled ``\texttt{black}'' or ``\texttt{white}'', denoted
(as in the main text) as \emph{AA} and \emph{C}, respectively.
We take as outcome the bar passage indicator and define a continuous score $Z\in(0,1)$ as the predicted probability of passing
the bar from a logistic regression model trained on non-sensitive covariates (e.g., LSAT, UGPA, and other available features),
excluding the sensitive attribute.
This yields smooth group-wise score distributions suitable for quantile-based sketches.

\paragraph{Federated silo interpretation.}
We reuse the federated protocol and the three allocation regimes from Appendix~\ref{app:compas}:
random assignment (independence), and two selection-biased regimes obtained by introducing score--silo dependence through a
Gaussian-copula construction with fixed group-wise margins.
Each silo releases only group counts and $k$ quantiles per group, and the server reconstructs global quantile functions to estimate
the distributional fairness targets (reported below through $|\Delta\bar y|$, $U_2$, $H_2$, and $W_2$; see Appendix~\ref{app:compas}
for definitions).

\begin{table}[t]
\caption{\textbf{Law: random vs.\ selection-biased silo allocations.}
Global and per-silo summary statistics for AA (\texttt{black}) vs.\ C (\texttt{white}) scores under three regimes:
random assignment (independence), negative selection bias, and positive selection bias (Gaussian-copula assignment with fixed
group-wise margins).
Rows report the mean-score gap $|\Delta\bar{z}|$ and distributional distances ($U_2$, $H_2$, $W_2$).}
\label{tab:law:compare:three}
\centering
\begin{tabular}{lrrrrrr}
\toprule
metric & Global & Silo 1 & Silo 2 & Silo 3 & Silo 4 & Silo 5\\
\midrule
\multicolumn{7}{l}{\textit{Random assignment (independence)}}\\
$|\Delta\bar{z}|$ & 0.2175 & 0.2442 & 0.2071 & 0.2223 & 0.1948 & 0.2202\\
$U_2$            & 0.0036 & 0.0043 & 0.0036 & 0.0036 & 0.0030 & 0.0037\\
$H_2$            & 0.0064 & 0.0085 & 0.0065 & 0.0066 & 0.0063 & 0.0065\\
$W_2$            & 0.2453 & 0.2712 & 0.2402 & 0.2478 & 0.2204 & 0.2477\\
\addlinespace[0.7ex]

\multicolumn{7}{l}{\textit{Negative selection bias}}\\
$|\Delta\bar{z}|$ & 0.2175 & 0.3480 & 0.1695 & 0.0654 & 0.0045 & 0.1358\\
$U_2$            & 0.0036 & 0.0202 & 0.0023 & 0.0002 & 0.0000 & 0.0002\\
$H_2$            & 0.0064 & 0.0373 & 0.0088 & 0.0013 & 0.0000 & 0.0006\\
$W_2$            & 0.2453 & 0.3821 & 0.1917 & 0.0734 & 0.0261 & 0.1606\\
\addlinespace[0.7ex]

\multicolumn{7}{l}{\textit{Positive selection bias}}\\
$|\Delta\bar{z}|$ & 0.2175 & 0.1954 & 0.0998 & 0.0736 & 0.0371 & 0.0225\\
$U_2$            & 0.0036 & 0.0071 & 0.0009 & 0.0002 & 0.0000 & 0.0000\\
$H_2$            & 0.0064 & 0.0137 & 0.0028 & 0.0012 & 0.0003 & 0.0001\\
$W_2$            & 0.2453 & 0.2093 & 0.1154 & 0.0869 & 0.0424 & 0.0313\\
\bottomrule
\end{tabular}
\end{table}

\subsection{Figures and discussion}
\label{app:law:figs}

\paragraph{(i) Convergence in \(k\).}
We study how the sketch size \(k\) affects the accuracy of \(\widehat U_2(k)\).
Figure~\ref{fig:law:mae-k} reports the mean absolute error
\(\mathrm{MAE}(k)=\mathbb{E}\big|\widehat{U}_2(k)-U_2\big|\) as a function of \(k\), for several values of \(d\) and for each
allocation regime. Across regimes, \(\widehat U_2(k)\) stabilizes rapidly, with diminishing returns beyond a few dozen quantiles.

\begin{figure}[!htbp]
\centering
\includegraphics[width=.99\textwidth]{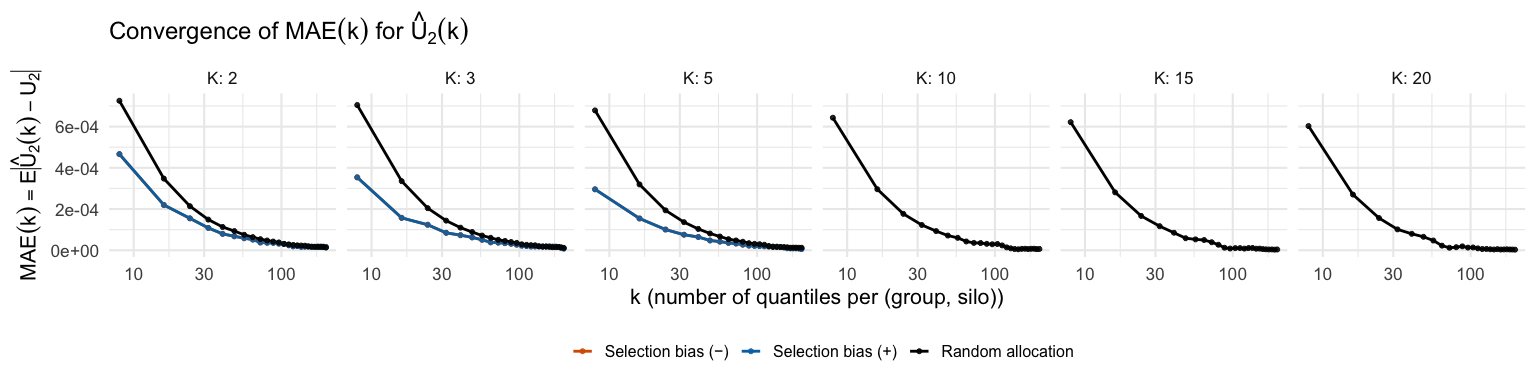}
\caption{\textbf{Law: convergence in $k$.}
Mean absolute error $\mathrm{MAE}(k)=\mathbb{E}\,|\widehat U_2(k)-U_2|$ as a function of the number of quantiles $k$
(sent per group and per silo) on a log scale, for several numbers of silos $d$ and for different allocation regimes.}
\label{fig:law:mae-k}
\end{figure}

\paragraph{(ii) Distribution over scenarios.}
Figure~\ref{fig:law:box-k} shows the distribution of \(\widehat U_2(k)\) across Monte Carlo allocations for a fixed \(d\).
The estimator concentrates near the centralized benchmark \(U_2\) for moderate \(k\), illustrating robustness to scenario variability.

\begin{figure}[!htbp]
\centering
\includegraphics[width=.99\textwidth]{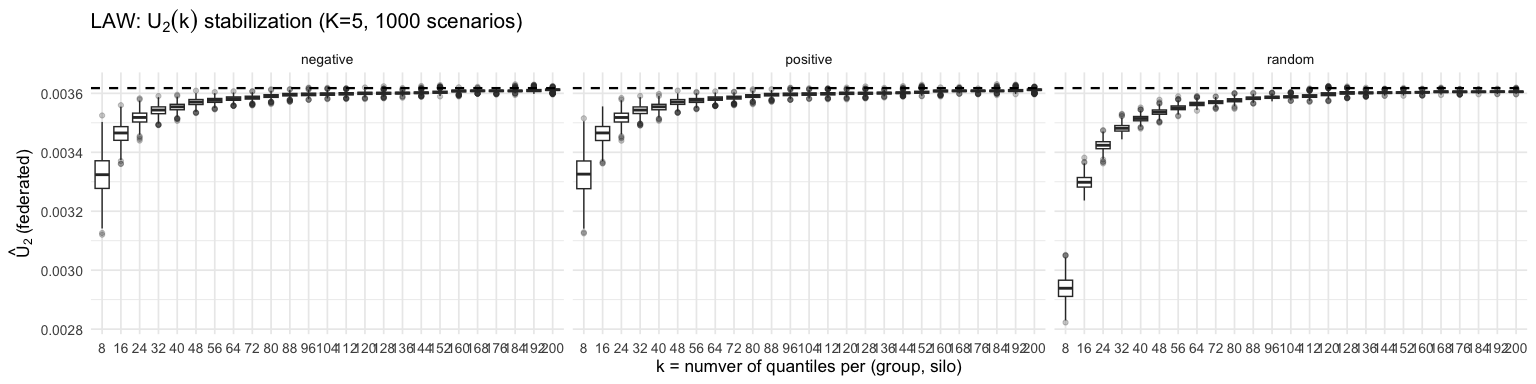}
\caption{\textbf{Law: stabilization of $\widehat U_2(k)$.}
Boxplots of $\widehat U_2(k)$ over repeated allocations for a fixed $d$ and varying $k$.
The dashed horizontal line is the centralized benchmark $U_2$.}
\label{fig:law:box-k}
\end{figure}

\paragraph{(iii) Reliability profiles \(p_{\mathrm{ok}}(k)\).}
We also report
\[
p_{\mathrm{ok}}(k)=\mathbb{P}\!\left(\frac{|\widehat{U}_2(k)-U_2|}{U_2}\le \tau\right),\qquad (\tau=1\%).
\]
Figure~\ref{fig:law:pok-k} shows that increasing \(k\) quickly pushes most runs into the target-accuracy region.
Figure~\ref{fig:law:k95-heatmap} summarizes this via \(k_{95}\), the smallest \(k\) such that \(p_{\mathrm{ok}}(k)\ge 95\%\).

\begin{figure}[!htbp]
\centering
\includegraphics[width=.99\textwidth]{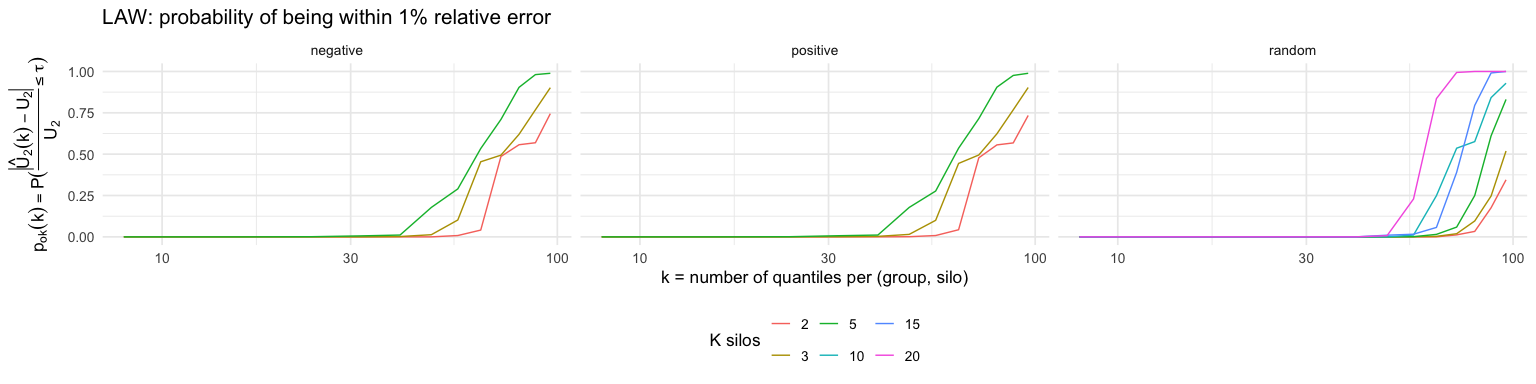}
\caption{\textbf{Law: reliability profile.}
$p_{\mathrm{ok}}(k)=\mathbb{P}\!\left(\frac{|\widehat U_2(k)-U_2|}{U_2}\le \tau\right)$ (here $\tau=1\%$)
as a function of $k$ (log scale), for several $d$ and allocation regimes.}
\label{fig:law:pok-k}
\end{figure}

\begin{figure}[!htbp]
\centering
\includegraphics[width=.99\textwidth]{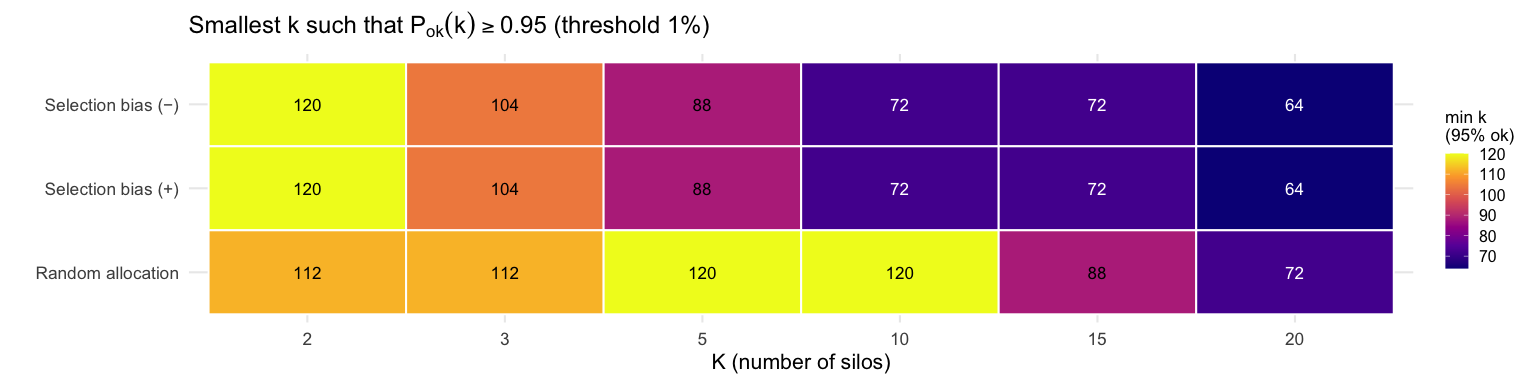}
\caption{\textbf{Law: minimal sketch size for a target reliability.}
Heatmap of $k_{95}$, the smallest $k$ such that $p_{\mathrm{ok}}(k)\ge 95\%$ (with $\tau=1\%$), as a function of $d$
and the allocation regime.}
\label{fig:law:k95-heatmap}
\end{figure}

\paragraph{(iv) Impact of the number of silos \(d\).}
At fixed \(k\), increasing \(d\) makes local \((\text{group},\text{silo})\) samples smaller on average, hence local quantiles
noisier; this effect is visible through the upward shift of errors / slower reliability gains as \(d\) grows in
Figures~\ref{fig:law:mae-k} and \ref{fig:law:pok-k}, and it is summarized by the larger \(k_{95}\) values in
Figure~\ref{fig:law:k95-heatmap}. This highlights the trade-off between federation granularity (large \(d\)) and sketch resolution
(larger \(k\)).

\end{document}